\documentclass{article} 
\usepackage{iclr2021_conference,times}

\usepackage{amsmath,amsfonts,bm}

\def\eqref#1{equation~\ref{#1}}

\def\1{\bm{1}}

\def\vtheta{{\bm{\theta}}}

\DeclareMathAlphabet{\mathsfit}{\encodingdefault}{\sfdefault}{m}{sl}
\SetMathAlphabet{\mathsfit}{bold}{\encodingdefault}{\sfdefault}{bx}{n}

\DeclareMathOperator*{\argmax}{arg\,max}

\usepackage{algorithm}
\usepackage{algpseudocode}
\usepackage{rotating}  
\usepackage{lscape}
\usepackage{hyperref}
\usepackage{url}
\usepackage{wrapfig}
\usepackage{caption}
\usepackage{subcaption}
\usepackage{tabularx}
\usepackage{booktabs}

\DeclareMathOperator*{\ev}{\mathbb{E}}

\newcommand{\de}{\,\mathrm{d}}
\newcommand{\vrho}{\mathbr{\rho}}

\newcommand{\hyscoreprime}[1][\vtheta]{\nabla_{\vrho'}\log\nu_{\vrho'}}

\definecolor{citrine}{rgb}{0.89, 0.82, 0.04}
\definecolor{blued}{RGB}{70,197,221}

\usepackage{amsthm}  
\usepackage{thmtools}
\usepackage{thm-restate}

\definecolor{blue}{rgb}{0.122,0.467,0.706}
\definecolor{green}{rgb}{0.173,0.627,0.173}
\definecolor{orange}{rgb}{1.0,0.498,0.055}
\definecolor{red}{rgb}{0.893,0.153,0.157}

\title{Parameter-Based Value Functions}

\author{Francesco Faccio, Louis Kirsch \& J\"{u}rgen Schmidhuber\\
The Swiss AI Lab IDSIA, USI, SUPSI\\
\texttt{\{francesco,louis,juergen\}@idsia.ch} \\
}

\iclrfinalcopy 
\begin{document}

\maketitle

\begin{abstract}
Traditional off-policy actor-critic Reinforcement Learning (RL) algorithms learn value functions of a single target policy.
However, when value functions are updated to track the learned policy, they forget potentially useful information about old policies.
We introduce a class of value functions called Parameter-Based Value Functions (PBVFs) whose inputs include the policy parameters.
They can generalize across different policies.
PBVFs can evaluate the performance of any policy given a state, a state-action pair, or a distribution over the RL agent's initial states.
First we show how PBVFs yield novel off-policy policy gradient theorems.
Then we derive off-policy actor-critic algorithms based on PBVFs trained by Monte Carlo or Temporal Difference methods.
We show how learned PBVFs can zero-shot learn new policies that outperform any policy seen during training.
Finally our algorithms are evaluated on a selection of discrete and continuous control tasks using shallow policies and deep neural networks.
Their performance is comparable to state-of-the-art methods.

\end{abstract}

\section{Introduction}
\label{sec:intro}
Value functions are central to Reinforcement Learning (RL).
For a given policy, they estimate the value of being in a specific state (or of choosing a particular action in a given state).
Many RL breakthroughs were achieved through improved estimates of such values, which can be used to find optimal policies~\citep{tesauro1995temporal, mnih2015human}.
However, learning value functions of arbitrary policies without observing their behavior in the environment is not trivial.
Such off-policy learning requires to correct the mismatch between the distribution of updates induced by the behavioral policy and the one we want to learn. 
Common techniques include Importance Sampling (IS)~\citep{hesterberg1988advances} and deterministic policy gradient methods (DPG)~\citep{Silver2014}, which adopt the actor-critic architecture~\citep{sutton1984temporal, kondaactorcritic, Peters:2008:NA:1352927.1352986}.

Unfortunately, these approaches have limitations.
IS suffers from large variance~\citep{cortes2010learning, metelli2018policy,wang2016sample} while traditional off-policy actor-critic methods introduce off-policy objectives whose gradients are difficult to follow since they involve the gradient of the action-value function with respect to the policy parameters $\nabla_{\theta} Q^{\pi_{\theta}}(s,a)$~\citep{Degris2012, Silver2014}. This term is usually ignored, resulting in biased gradients for the off-policy objective.
Furthermore, off-policy actor-critic algorithms learn value functions of a single target policy. When value functions are updated to track the learned policy, the information about old policies is lost.

We address the problem of generalization across many value functions in the off-policy setting by introducing a class of \emph{parameter-based value functions} (PBVFs) defined for any policy.
PBVFs are value functions whose inputs include the policy parameters, the PSSVF $V(\theta)$, PSVF $V(s, \theta)$, and PAVF $Q(s, a, \theta)$.
PBVFs can be learned using Monte Carlo (MC)~\citep{MET49} or Temporal Difference (TD)~\citep{sutton1988learning} methods.
The PAVF $Q(s, a, \theta)$ leads to a novel stochastic and deterministic off-policy policy gradient theorem and, unlike previous approaches, can directly compute $\nabla_{\theta} Q^{\pi_{\theta}}(s,a)$.
Based on these results, we develop off-policy actor-critic methods and compare our algorithms to two strong baselines, ARS and DDPG~\citep{mania2018simple, lillicrap2015continuous}, outperforming them in some environments.

We make theoretical, algorithmic, and experimental contributions:
Section~\ref{sec:mdp} introduces the standard MDP setting;
Section~\ref{sec:pvfs} formally presents PBVFs and derive algorithms for $V(\theta)$, $V(s,\theta)$ and $Q(s,a,\theta)$;
Section~\ref{sec:exp} describes the experimental evaluation using shallow and deep policies;
Sections~\ref{sec:related} and~\ref{sec:future} discuss related and future work.
Proofs and derivations can be found in Appendix~\ref{apx:proofs}.

\section{Background}
\label{sec:mdp}
We consider a Markov Decision Process (MDP)~\citep{stratonovich1960,puterman2014markov} $\mathcal{M}=(\mathcal{S},\mathcal{A},P,R,\gamma,\mu_0)$ where at each step an agent observes a state $s \in \mathcal{S}$, chooses action $a \in \mathcal{A}$, transitions into state $s'$ with probability $P(s' | s, a)$ and receives a reward $R(s,a)$. The agent starts from an initial state, chosen with probability $\mu_0(s)$. It is represented by a parametrized stochastic policy $\pi_{\theta}: \mathcal{S} \rightarrow \Delta(\mathcal{A})$, which provides the probability of performing action $a$ in state $s$. $\Theta$ is the space of policy parameters. The policy is deterministic if for each state $s$ there exists an action $a$ such that  $\pi_{\theta}(a|s)=1$. The return $R_t$ is defined as the cumulative discounted reward from time step t: $R_t =  \sum_{k=0}^{T-t-1}\gamma^k R(s_{t+k+1}, a_{t+k+1})$, where T denotes the time horizon and $\gamma$ a real-valued discount factor. The performance of the agent is measured by the cumulative discounted expected reward (expected return), defined as     $J(\pi_{\theta})={\ev}_{\pi_{\theta}}[R_0].$ Given a policy $\pi_{\theta}$, the state-value function $V^{\pi_{\theta}}(s) = {\ev}_{\pi_{\theta}}[R_t|s_t=s]$ is defined as the expected return for being in a state $s$ and following policy $\pi_{\theta}$. By integrating over the state space $\mathcal{S}$, we can express the maximization of the expected cumulative reward in terms of the state-value function $J(\pi_{\theta})=\int_{\mathcal{S}} \mu_0(s) V^{\pi_{\theta}}(s) \de s$.
The action-value function $Q^{\pi_{\theta}}(s,a)$, which is defined as the expected return for performing action $a$ in state $s$, and following the policy $\pi_{\theta}$, is $Q^{\pi_{\theta}}(s,a) = {\ev}_{\pi_{\theta}}[R_t|s_t=s,a_t=a]$, and it is related to the state-value function by $V^{\pi_{\theta}}(s) = \int_{\mathcal{A}} \pi_{\theta}(a|s) Q^{\pi_{\theta}}(s,a)  \de a$. We define as $d^{\pi_{\theta}}(s')$ the discounted weighting of states encountered starting at $s_0 \sim \mu_0(s)$ and following the policy $\pi_{\theta}$: $d^{\pi_{\theta}}(s') = \int_{\mathcal{S}}\sum_{t=1}^{\infty} \gamma^{t-1} \mu_0(s) P(s \rightarrow s', t, \pi_{\theta}) \de s $, where $P(s \rightarrow s', t, \pi_{\theta})$ is the probability of transitioning to $s'$ after t time steps, starting from s and following policy $\pi_{\theta}$. ~\citet{Sutton1999} showed that, for stochastic policies, the gradient of $J(\pi_{\theta})$ does not involve the derivative of $d^{\pi_{\theta}}(s)$ and can be expressed in a simple form:
\begin{equation}
\nabla_{\theta} J(\pi_{\theta}) = \int_{\mathcal{S}} d^{\pi_{\theta}}(s) \int_{\mathcal{A}} \nabla_{\theta} \pi_{\theta}(a|s) Q^{\pi_{\theta}}(s,a)  \de a \de s.
\end{equation}
Similarly, for deterministic policies~\citet{Silver2014} obtained the following:
\begin{equation}
\nabla_{\theta} J(\pi_{\theta}) = \int_{\mathcal{S}} d^{\pi_{\theta}}(s) \nabla_{\theta} \pi_{\theta}(s) \nabla_{a} Q^{\pi_{\theta}}(s,a)|_{a=\pi_{\theta}(s)} \de s.
\end{equation}

\paragraph{Off-policy RL} In off-policy policy optimization, we seek to find the parameters of the policy maximizing a performance index $ J_{b}(\pi_{\theta})$ using data collected from a behavioral policy $\pi_b$. Here the objective function $J_{b}(\pi_{\theta})$ is typically modified to be the value function of the target policy, integrated over $d_{\infty}^{\pi_b}(s) = \lim_{t \rightarrow \infty} P(s_t = s| s_0, \pi_b)$, the limiting distribution of states under $\pi_b$ (assuming it exists)~\citep{Degris2012,imani2018off, wang2016sample}. Throughout the paper we assume that the support of $d_{\infty}^{\pi_b}$ includes the support of $\mu_0$ so that the optimal solution for $J_b$ is also optimal for $J$. Formally, we want to find:
\begin{equation}
    J_b(\pi_{\theta^*}) =\max_{\theta} \int_{\mathcal{S}} d_{\infty}^{\pi_b}(s) V^{\pi_{\theta}}(s) \de s = \max_{\theta}\int_{\mathcal{S}}  d_{\infty}^{\pi_b}(s) \int_{\mathcal{A}}  \pi_{\theta}(a|s)  Q^{\pi_{\theta}}(s,a) \de a \de s.
\end{equation}
Unfortunately, in the off-policy setting, the states are obtained from $d_{\infty}^{\pi_b}$ and not from $d_{\infty}^{\pi_{\theta}}$, hence the gradients suffer from a distribution shift~\citep{liu2019off, nachum2019algaedice}. Moreover, since we have no access to $d_{\infty}^{\pi_{\theta}}$, a term in the policy gradient theorem corresponding to the gradient of the action value function with respect to the policy parameters needs to be estimated. This term is usually ignored in traditional off-policy policy gradient theorems\footnote{With tabular policies, dropping this term still results in a convergent algorithm~\citep{Degris2012}.}. In particular, when the policy is stochastic, ~\citet{Degris2012} showed that:
\begin{align}
	    \nabla_{\theta} J_b(\pi_{\theta}) & = \int_{\mathcal{S}} d_{\infty}^{\pi_{b}}(s) \int_{\mathcal{A}} \pi_{b}(a|s) \frac{\pi_{\theta}(a|s)}{\pi_{b}(a|s)} \left(Q^{\pi_{\theta}}(s,a) \nabla_{\theta} \log\pi_{\theta}(a|s) +  \nabla_{\theta}Q^{\pi_{\theta}}(s,a)\right) \de a \de s \\
	    & \approx \int_{\mathcal{S}} d_{\infty}^{\pi_{b}}(s) \int_{\mathcal{A}} \pi_{b}(a|s) \frac{\pi_{\theta}(a|s)}{\pi_{b}(a|s)} \left(Q^{\pi_{\theta}}(s,a) \nabla_{\theta} \log\pi_{\theta}(a|s) \right) \de a \de s.
	\end{align}
Analogously, ~\citet{Silver2014} provided the following approximation for deterministic policies~\footnote{In the original formulation of~\citet{Silver2014} $d_{\infty}^{\pi_{b}}(s)$ is replaced by $d^{\pi_{b}}(s)$.}:
\begin{align}
	    \nabla_{\theta} J_b(\pi_{\theta}) & = \int_{\mathcal{S}} d_{\infty}^{\pi_{b}}(s) \left( \nabla_{\theta} \pi_{\theta}(s) \nabla_{a} Q^{\pi_{\theta}}(s,a)|_{a=\pi_{\theta}(s)} + \nabla_{\theta} Q^{\pi_{\theta}}(s,a)|_{a=\pi_{\theta}(s)}\right) \de s \\
	    & \approx \int_{\mathcal{S}} d_{\infty}^{\pi_{b}}(s) \left( \nabla_{\theta} \pi_{\theta}(s) \nabla_{a} Q^{\pi_{\theta}}(s,a)|_{a=\pi_{\theta}(s)} \right) \de s.
	\end{align}
Although the term $\nabla_{\theta}Q^{\pi_{\theta}}(s,a)$ is dropped, there might be advantages in using the approximate gradient of $J_b$ in order to find the maximum of the original RL objective $J$. Indeed, if we were on-policy, the approximated off-policy policy gradients by ~\citet{Degris2012,Silver2014} would revert to the on-policy policy gradients, while an exact gradient for $J_b$ would necessarily introduce a bias. However, when we are off-policy, it is not clear whether this would be better than using the exact gradient of $J_b$ in order to maximize $J$. In this work, we assume that $J_b$ can be considered a good objective for off-policy RL and we derive an exact gradient for it.


\section{Parameter-based Value Functions}
\label{sec:pvfs}
In this section, we introduce our parameter-based value functions, the PSSVF $V(\theta)$, PSVF $V(s, \theta)$, and PAVF $Q(s, a, \theta)$ and their corresponding learning algorithms. First, we augment the state and action-value functions, allowing them to receive as an input also the weights of a parametric policy.
The parameter-based state-value function (PSVF) $V(s, \theta) = {\ev}[R_t|s_t=s, \theta]$ is defined as the expected return for being in state $s$ and following policy parameterized by $\theta$. Similarly, the parameter-based action-value function (PAVF) $Q(s,a, \theta) = {\ev}[R_t|s_t=s, a_t=a, \theta]$ is defined as the expected return for being in state s, taking action $a$ and following policy parameterized by $\theta$.
Using PBVFs, the RL objective becomes: $ J(\pi_{\theta})=\int_{\mathcal{S}} \mu_0(s) V^{\pi}(s, \theta) \de s$.
Maximizing this objective leads to on-policy policy gradient theorems that are analogous to the traditional ones ~\citep{Sutton1999, Silver2014}: 
\begin{restatable}[]{thr}{onspavf}
    \label{thr:onspavf}

	Let $\pi_{\theta}$ be stochastic. For any Markov Decision Process, the following holds:
	\begin{equation}
	    \nabla_{\theta} J(\pi_{\theta}) = {\ev}_{s \sim d^{\pi_{\theta}}(s), a \sim \pi_{\theta}(.|s)}\left[\left(Q(s,a,\theta) \nabla_{\theta} \log\pi_{\theta}(a|s)\right)\right].
	\end{equation}
\end{restatable}

\begin{restatable}[]{thr}{ondpavf}
    \label{thr:ondpavf}

	Let $\pi_{\theta}$ be deterministic. Under standard regularity assumptions~\citep{Silver2014}, for any Markov Decision Process, the following holds:
	\begin{equation}
	    \nabla_{\theta} J(\pi_{\theta}) = {\ev}_{s \sim d^{\pi_{\theta}}(s)} \left[   \nabla_{a} Q(s,a,\theta)|_{a=\pi_{\theta}(s)} \nabla_{\theta} \pi_{\theta}(s) \right].
	\end{equation}
\end{restatable}

Parameter-based value functions allow us also to learn a function of the policy parameters that directly approximates $J(\pi_{\theta})$. In particular, the parameter-based start-state-value function (PSSVF) is defined as:
\begin{equation}
    V(\theta) := {\ev}_{s \sim \mu_0(s)}[V(s, \theta)] = \int_{\mathcal{S}} \mu_0(s) V(s, \theta) \de s = J(\pi_{\theta}).
\end{equation}

\paragraph{Off-policy RL} In the off-policy setting, the objective to be maximized becomes:
\begin{equation}
    J_b(\pi_{\theta^{*}})=\max_{\theta}\int_{\mathcal{S}} d_{\infty}^{\pi_b}(s) V(s, \theta) \de s = \max_{\theta}\int_{\mathcal{S}} \int_{\mathcal{A}} d_{\infty}^{\pi_b}(s)  \pi_{\theta}(a|s)  Q(s,a,\theta) \de a \de s.
\end{equation}

By taking the gradient of the performance $J_b$ with respect to the policy parameters $\theta$ we obtain novel policy gradient theorems.
Since $\theta$ is continuous, we need to use function approximators ${V}_{\textbf{w}}(\theta) \approx V(\theta)$, ${V}_{\textbf{w}}(s, \theta) \approx V(s, \theta)$ and ${Q}_{\textbf{w}}(s,a,\theta) \approx Q(s,a,\theta)$. Compatible function approximations can be derived to ensure that the approximated value function is following the true gradient. Like in previous approaches, this would result in linearity conditions. However, here we consider nonlinear function approximation and we leave the convergence analysis of linear PBVFs as future work. In episodic settings, we do not have access to $d_{\infty}^{\pi_b}$, so in the algorithm derivations and in the experiments we approximate it by sampling trajectories generated by the behavioral policy.
In all cases, the policy improvement step can be very expensive, due to the computation of the $\argmax$ over a continuous space $\Theta$.
Actor-critic methods can be derived to solve this optimization problem, where the critic (PBVFs) can be learned using TD or MC methods, while the actor is updated following the gradient with respect to the critic.
Although our algorithms on PSSVF and PSVF can be used with both stochastic and deterministic policies, removing the stochasticity of the action-selection process might facilitate learning the value function. All our algorithms make use of a replay buffer.

\subsection{Parameter-based Start-State-Value Function $V(\theta)$}
We first derive the PSSVF $V(\theta)$. Given the original performance index $J$, and taking the gradient with respect to $\theta$, we obtain:
\begin{equation}
    \nabla_{\theta}J(\pi_{\theta})=\int_{\mathcal{S}} \mu_0(s)\nabla_{\theta} V(s, \theta) \de s = {\ev}_{s \sim \mu_0(s)}[\nabla_{\theta} V(s, \theta)] = \nabla_{\theta}V(\theta).
\end{equation}
In Algorithm~\ref{alg:pvf}, the critic $V_{\textbf{w}}(\theta)$ is learned using MC to estimate the value of any policy $\theta$. The actor is then updated following the direction of improvement suggested by the critic. Since the main application of PSSVF is in episodic tasks\footnote{Alternatives include regenerative method for MC estimation~\citep{regen}.}, we optimize for the undiscounted objective. 
\begin{algorithm}[H]
  \caption{Actor-critic with Monte Carlo prediction for $V(\theta)$}
  \label{alg:pvf}
   \hspace*{\algorithmicindent} \textbf{Input}: Differentiable critic $V_{\textbf{w}}: \Theta \rightarrow \mathcal{R}$ with parameters $\textbf{w}$; deterministic or stochastic actor $\pi_{\theta}$ with parameters $\theta$; empty replay buffer $D$ \\
   \hspace*{\algorithmicindent} \textbf{Output} : Learned $V_{\textbf{w}} \approx V(\theta) \forall \theta$, learned $\pi_{\theta} \approx \pi_{\theta^*}$
    \begin{algorithmic}
    \State Initialize critic and actor weights $\textbf{w}, \theta$
	\Repeat:
		\State Generate an episode $s_{0}, a_{0}, r_{1}, s_{1}, a_{1}, r_{2}, \dots, s_{T-1}, a_{T-1}, r_{T}$ with policy $\pi_{\theta}$
		\State Compute return $r = \sum_{k=1}^T r_k$
		\State Store $(\theta, r)$ in the replay buffer $D$
		\For {many steps}:
		    \State Sample a batch $B = \{(r, \theta)\}$ from $D$
		    \State Update critic by stochastic gradient descent: $\nabla_{\textbf{w}} \ev_{(r, \theta) \in B} [r - V_{\textbf{w}}(\theta)]^2$
	    \EndFor
		\For {many steps}:
		    \State Update actor by gradient ascent: $\nabla_{\theta}  V_{\textbf{w}}(\theta)$ 
	    \EndFor
    \Until{convergence}
    \end{algorithmic}
\end{algorithm}

\subsection{Parameter-based State-Value Function $V(s,\theta)$}
Learning the value function using MC approaches can be difficult due to the high variance of the estimate. Furthermore, episode-based algorithms like Algorithm~\ref{alg:pvf} are unable to credit good actions in bad episodes. Gradient methods based on TD updates provide a biased estimate of $V(s,\theta)$ with much lower variance and can credit actions at each time step. Taking the gradient of $J_b(\pi_{\theta})$ in the PSVF formulation\footnote{Compared to standard methods based on the state-value function, we can directly optimize the policy following the performance gradient of the PSVF, obtaining a policy improvement step in a model-free way.}, we obtain:
\begin{equation}
    \nabla_{\theta}J_b(\pi_{\theta})=\int_{\mathcal{S}} d_{\infty}^{\pi_b}(s) \nabla_{\theta} V(s, \theta) \de s = {\ev}_{s \sim d_{\infty}^{\pi_b}(s)}[\nabla_{\theta} V(s, \theta)].
\end{equation}

Algorithm~\ref{alg:psvf} (Appendix) uses the actor-critic architecture, where the critic is learned via TD\footnote{Note that the differentiability of the policy $\pi_{\theta}$ is never required in PSSVF and PSVF.}.
\subsection{Parameter-based Action-Value Function $Q(s,a,\theta)$}

The introduction of the PAVF $Q(s, a, \theta)$ allows us to derive new policy gradients theorems when using a stochastic or deterministic policy.

\paragraph{Stochastic policy gradients}
We want to use data collected from some stochastic behavioral policy $\pi_{b}$ in order to learn the action-value of a target policy $\pi_{\theta}$. Traditional off-policy actor-critic algorithms only approximate the gradient of $J_b$, since they do not estimate the gradient of the action-value function with respect to the policy parameters $\nabla_{\theta} Q^{\pi_{\theta}}(s,a)$~\citep{Degris2012, Silver2014}. With PBVFs, we can directly compute this contribution to the gradient. This yields an exact policy gradient theorem for $J_b$:

\begin{restatable}[]{thr}{spavf}
    \label{thr:spavf}

	For any Markov Decision Process, the following holds:
	\begin{equation}
	    \nabla_{\theta} J_b(\pi_{\theta}) = {\ev}_{s \sim d_{\infty}^{\pi_{b}}(s), a \sim \pi_{b}(.|s)}\left[\frac{\pi_{\theta}(a|s)}{\pi_{b}(a|s)} \left(Q(s,a,\theta) \nabla_{\theta} \log\pi_{\theta}(a|s) +  \nabla_{\theta}Q(s,a,\theta)\right)\right].
	\end{equation}
\end{restatable}
Algorithm~\ref{alg:spavf} (Appendix) uses an actor-critic architecture and can be seen as an extension of Off-PAC~\citep{Degris2012} to PAVF. 

\paragraph{Deterministic policy gradients}
Estimating $Q(s,a,\theta)$ is in general a difficult problem due to the stochasticity of the policy. Deterministic policies of the form $\pi: \mathcal{S} \rightarrow \mathcal{A}$ can help improving the efficiency in learning value functions, since the expectation over the action space is no longer required. Using PBVFs, we can write the performance of a policy $\pi_{\theta}$ as:
\begin{equation}
    J_b(\pi_{\theta})=\int_{\mathcal{S}} d_{\infty}^{\pi_{b}}(s) V(s, \theta) \de s = \int_{\mathcal{S}} d_{\infty}^{\pi_{b}}(s)   Q(s,\pi_{\theta}(s),\theta) \de s.
\end{equation}
Taking the gradient with respect to $\theta$ we obtain a deterministic policy gradient theorem:
\begin{restatable}[]{thr}{dpavf}
	Under standard regularity assumptions~\citep{Silver2014}, for any Markov Decision Process, the following holds:
	\begin{equation}
	    \nabla_{\theta} J_b(\pi_{\theta}) = {\ev}_{s \sim d_{\infty}^{\pi_{b}}(s)} \left[   \nabla_{a} Q(s,a,\theta)|_{a=\pi_{\theta}(s)} \nabla_{\theta} \pi_{\theta}(s) + \nabla_{\theta} Q(s,a,\theta)|_{a=\pi_{\theta}(s)}   \right].
	\end{equation}
\end{restatable}

Algorithm~\ref{alg:dpavf} (Appendix) uses an actor-critic architecture and can be seen as an extension of DPG~\citep{Silver2014} to PAVF. Despite the novel formulation of algorithm~\ref{alg:spavf}, we decided to avoid the stochasticity of the policy and to implement and analyze only the deterministic PAVF.

\section[Experiments]{Experiments\footnote{Code is available at: \url{https://github.com/FF93/Parameter-based-Value-Functions}}}
\label{sec:exp}
Applying algorithms~\ref{alg:pvf},~\ref{alg:psvf} and ~\ref{alg:dpavf} directly can lead to convergence to local optima, due to the lack of exploration.
In practice, like in standard deterministic actor-critic algorithms, we use a noisy version of the current learned policy in order to act in the environment and collect data to encourage exploration.
More precisely, at each episode we use $\pi_{\Tilde{\theta}}$ with $\Tilde{\theta} = \theta + \epsilon, \epsilon \sim \mathcal{N}(0, \sigma^2 I)$ instead of $\pi_{\theta}$ and then store $\Tilde{\theta}$ in the replay buffer.
In our experiments, we report both for our methods as well as the baselines the performance of the policy without parameter noise.

\subsection{Visualizing PBVFs using LQRs}
We start with an illustrative example that allows us to visualize how PBVFs are learning to estimate the expected return over the parameter space.
For this purpose, we use an instance of the 1D Linear Quadratic Regulator (LQR) problem and a linear deterministic policy with bias.
In figure \ref{fig:lqr_pvf}, we plot the episodic $J(\theta)$, the cumulative return that an agent would obtain by acting in the environment using policy $\pi_{\theta}$ for a single episode, and the cumulative return predicted by the PSSVF $V(\theta)$ for two different times during learning.
At the beginning of the learning process, the PSSVF is able to provide just a local estimation of the performance of the agent, since only few data have been observed.
However, after 1000 episodes, it is able to provide a more accurate global estimate over the parameter space.
Appendix \ref{experiments_detail:lqr} contains a similar visualization for PSVF and PAVF, environment details and hyperparameters used.
\begin{figure}[h]
\begin{center}
\centering
Optimization after 60 episodes
\includegraphics[width=1.0\linewidth]{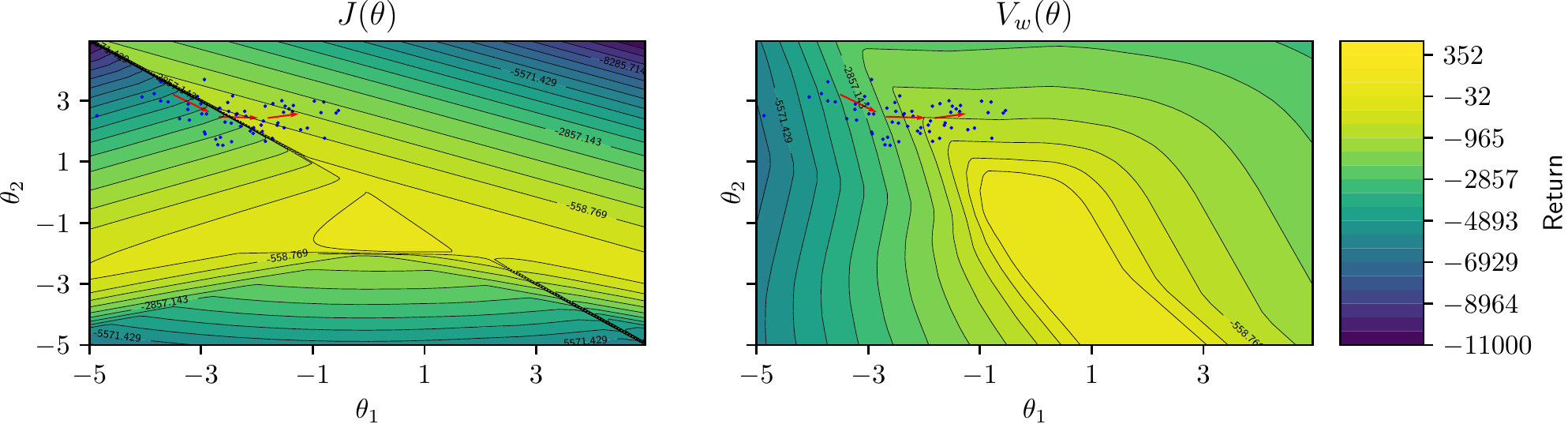} \\
\centering
Optimization after 1000 episodes
\includegraphics[width=1.0\linewidth]{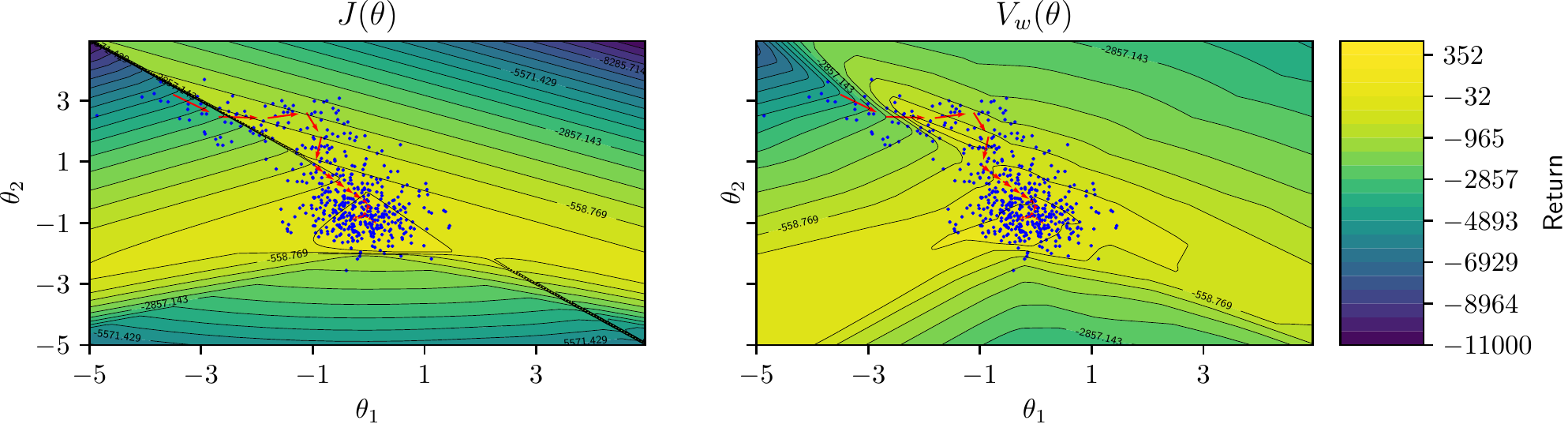}
\end{center}
\caption{True episodic return $J(\theta)$ and PSSVF estimation $V(\theta)$ as a function of the policy parameters at two different stages in training. The red arrows represent an optimization trajectory in  parameter space. The blue dots represent the perturbed policies used to train $V(\theta)$.}
\label{fig:lqr_pvf}
\end{figure}

\subsection{Main results}
Given the similarities between our PAVF and DPG, Deep Deterministic Policy Gradients (DDPG) is a natural choice for the baseline.
Additionally, the PSSVF $V(\theta)$ resembles evolutionary methods as the critic can be interpreted as a global fitness function.
Therefore, we decided to include in the comparison Augmented Random Search (ARS) which is known for its state-of-the-art performance using only linear policies in continuous control tasks.
For the policy, we use a 2-layer MLP (64,64) with tanh activations and a linear policy followed by a tanh nonlinearity. Figure~\ref{fig:learning_curves} shows results for deterministic policies with both architectures.
In all the tasks the PSSVF is able to achieve at least the same performance compared to ARS, often outperforming it.
In the Inverted Pendulum environment, PSVF and PAVF with deep policy are very slow to converge, but they excel in the Swimmer task and MountainCarContinuous.
In Reacher, all PBVFs fail to learn the task, while DDPG converges quickly to the optimal policy.
We conjecture that for this task it is difficult to perform a search in  parameter space.
On the other hand, in MountainCarContinuous, the reward is more sparse and DDPG only rarely observes positive reward when exploring in action space.
In Appendix~\ref{apx:exp} we include additional results for PSSVF and PSVF with stochastic policies and hyperparameters. We analyze the sensitivity of the algorithms on the choice of hyperparameters in Appendix~\ref{experiments_detail:sensitivity}.

\begin{figure}[h]

\begin{subfigure}[c]{1.0\textwidth}
  \centering
Shallow policies
  \includegraphics[width=0.9\linewidth]{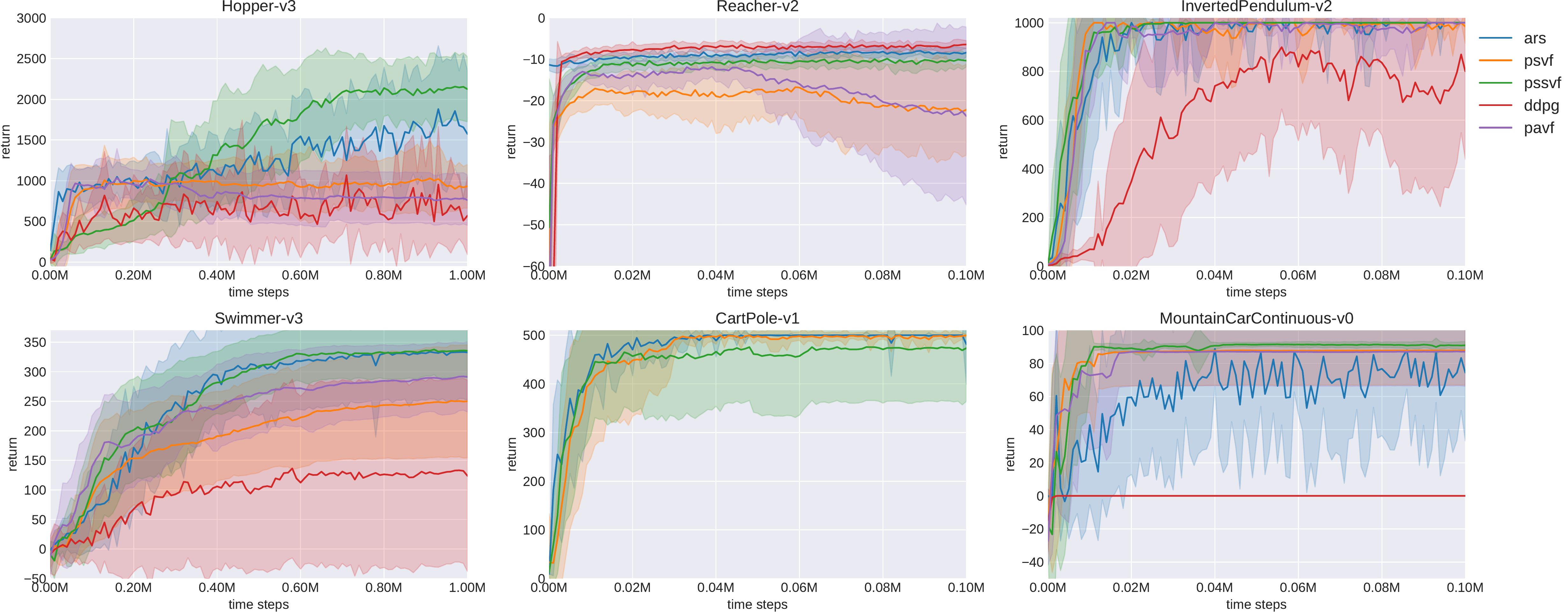}
  \vspace{0.3cm}
    \centering
\end{subfigure}%
\hfill
\begin{subfigure}[c]{1.0\textwidth}
	\vspace{0.2cm}
  \centering
  Deep policies

  \includegraphics[width=0.9\linewidth]{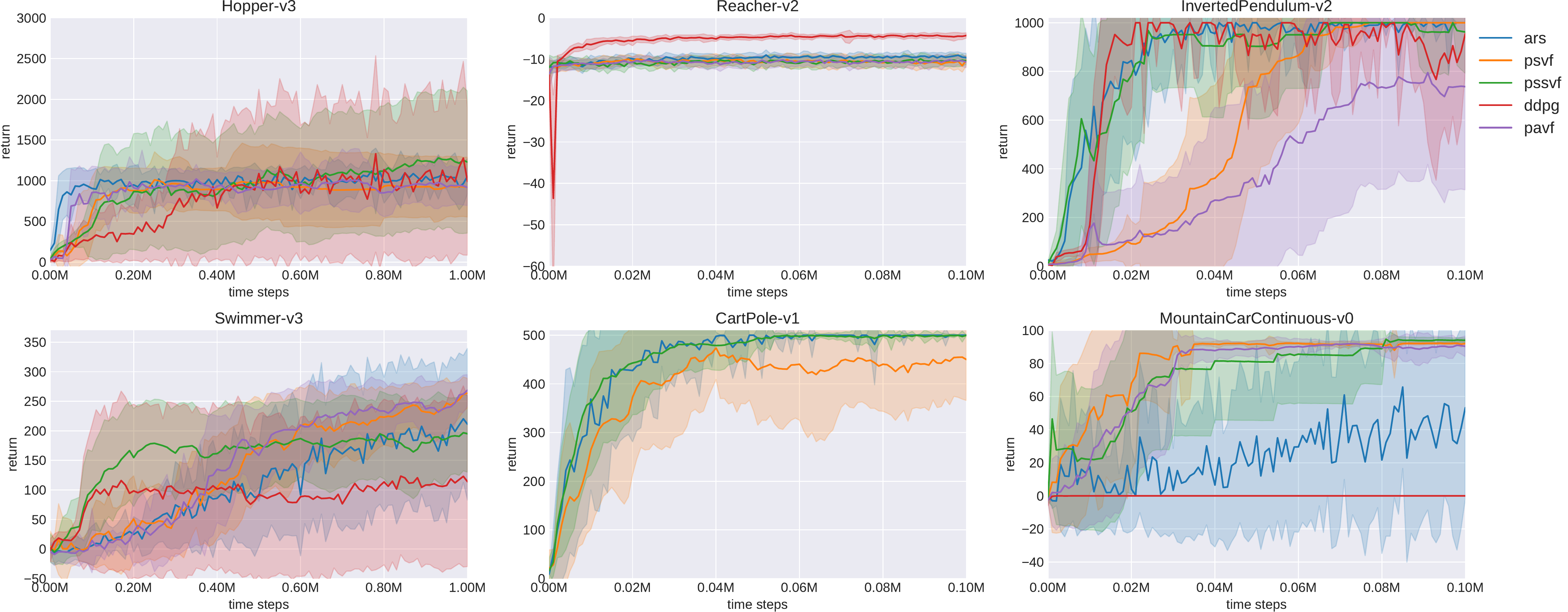}
    \centering
\end{subfigure}%
\caption{Average return of shallow and deep deterministic policies as a function of the number of time steps used for learning (across 20 runs, one standard deviation), for different environments and algorithms. We use the best hyperparameters found when maximizing the average return.}
\label{fig:learning_curves}
\end{figure}
\subsection{Zero-shot learning}
In order to test whether PBVFs are generalizing across the policy space, we perform the following experiment with shallow deterministic policies:
while learning using algorithm~\ref{alg:pvf}, we stop training and randomly initialize 5 policies.
Then, without interacting with the environment, we train these policies offline, in a zero-shot manner, following only the direction of improvement suggested by $\nabla_{\theta} V_{w}(\theta)$, whose weights $w$ remain frozen.
We observe that shallow policies can be effectively trained from scratch.
Results for PSSVFs in Swimmer-v3 are displayed in figure~\ref{fig:scratch}.
In particular, we compare the performance of the policy learned, the best perturbed policy for exploration seen during training and five policies learned from scratch at three different stages in training.
We note that after the PSSVF has been trained for 100,000 time steps interactions with the environment (first snapshot), these policies are already able to outperform both the current policy and any policy seen while training the PSSVF.
They achieve an average return of 297, while the best observed return was 225.
We include additional results for PSVF and PAVF in different environments, using shallow and deep policies in Appendix~\ref{experiments_detail:zero_shot}. When using deep policies, we obtain similar results only for the simplest environments.
For this task, we use the same hyperparameters as in figure~\ref{fig:learning_curves}.

\begin{figure}[h]
\begin{center}
\centering
\includegraphics[width=0.6\linewidth]{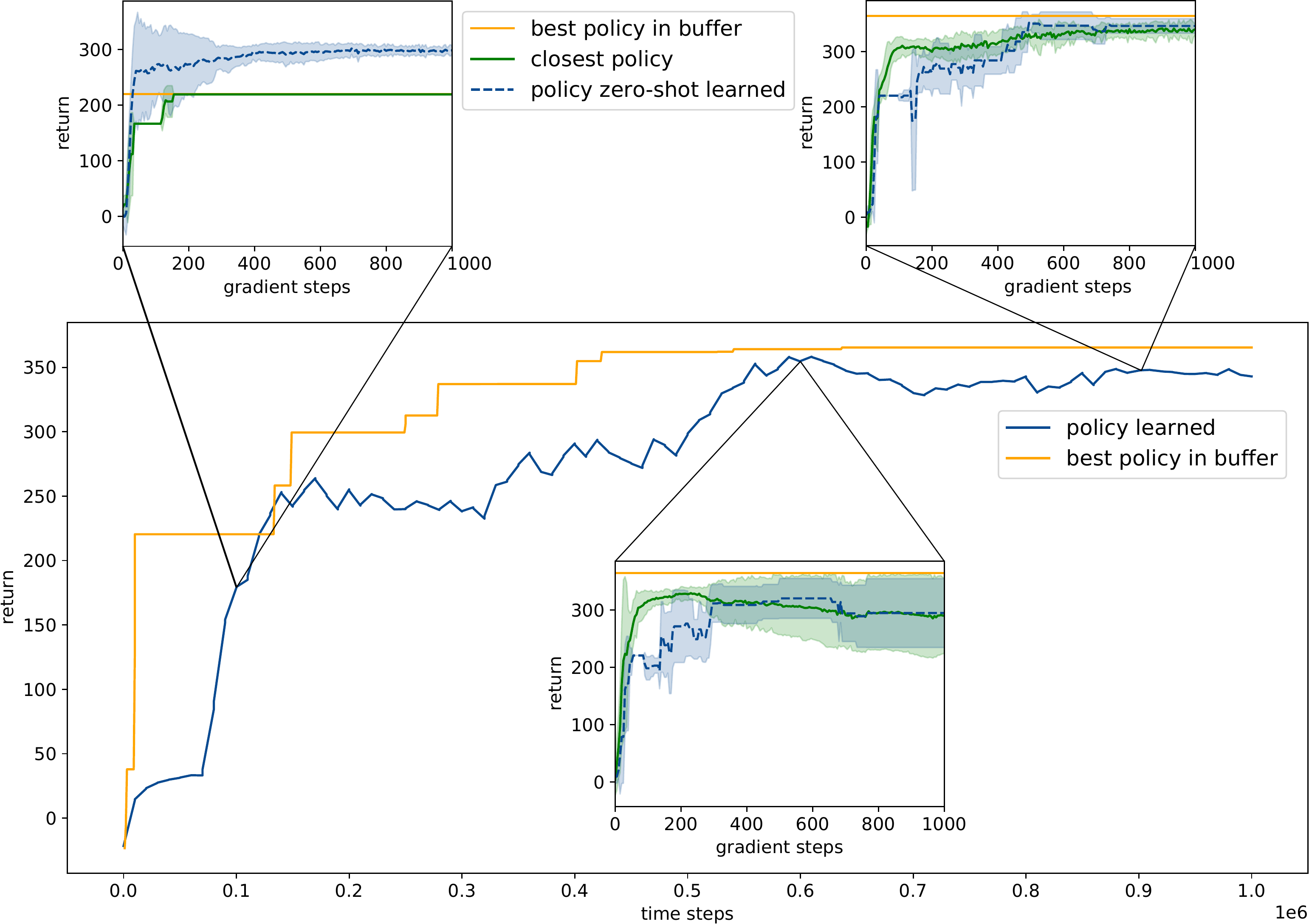} \\

\end{center}
\caption{Policies learned from scratch during training. The plot in the center represents the return of the agent learning while interacting with the environment using Algorithm~\ref{alg:pvf}. We compare the {\color{orange} best noisy policy $\pi_{\Tilde{\theta}}$} used for exploration to the {\color{blue}policy $\pi_{\theta}$} learned through the critic. The learning curves in the small plots represent the return obtained by {\color{blue}policies trained from scratch} following the fixed critic $V_{\textbf{w}}(\theta)$ after different time steps of training. The return of the {\color{green}closest policy} (L2 distance) in the replay buffer with respect to the policy learned from scratch is depicted in green.}
\label{fig:scratch}

\end{figure}

\subsection{Offline learning with fragmented behaviors}
\label{sec:offline}
In our last experiment, we investigate how PSVFs are able to learn in a completely offline setting.
The goal is to learn a good policy in Swimmer-v3 given a fixed dataset containing 100,000 transitions, without additional environment interactions.
Furthermore, the policy generating the data is perturbed every 200 time steps, for a total of 5 policies per episode.
Observing only incomplete trajectories for each policy parameter makes TD bootstrapping harder: In order to learn, the PSVF needs to generalize across both the state and the parameter space.
Given the fixed dataset, we first train the PSVF, minimizing the TD error.
Then, at different stages during learning, we train 5 new shallow deterministic policies.
Figure~\ref{fig:offline} describes this process.
We note that at the beginning of training, when the PSVF $V(s,\theta)$ has a larger TD error, these policies have poor performance.
However, after 7000 gradient updates, they are able to achieve a reward of 237, before eventually degrading to 167.
They outperform the best policy in the dataset used to train the PSVF, whose return is only of 58.

\begin{figure}[h]
\begin{center}
\centering
\includegraphics[width=0.6\linewidth]{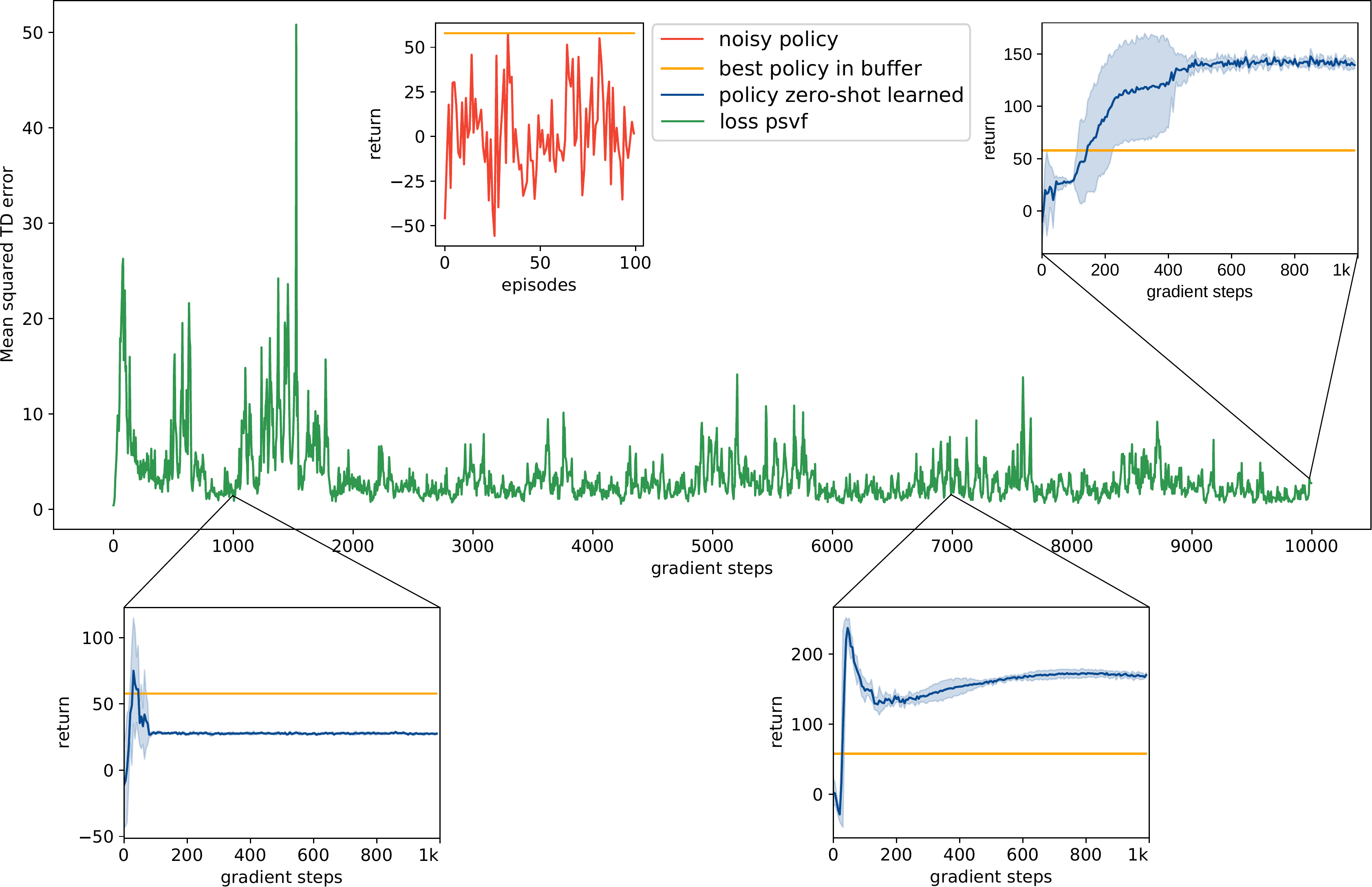}
\end{center}
\caption{Offline learning of PSVF. We plot the {\color{green} mean squared TD error of a PSVF} trained using data coming from a set of {\color{red} noisy policies}. In the small plots, we compare the return obtained by {\color{blue}policies trained from scratch} following the fixed critic $V_{\textbf{w}}(s, \theta)$ after different time steps of value function training and the return of the best {\color{orange} noisy policy} used to train V.}
\label{fig:offline}
\end{figure}

\section{Related work}
\label{sec:related}
There are two main classes of similar algorithms performing search in policy parameter space.
Evolutionary algorithms~\citep{wierstra2014natural,salimans2017evolution,mania2018simple} iteratively estimate a fitness function evaluating the performance of a population of policies and then perform gradient ascent in parameter space, often estimating the gradient using finite difference approximation.
By replacing the performance of a population through a likelihood estimation, evolutionary algorithms become a form of Parameter Exploring Policy Gradients~\citep{sehnkepgpecontrol,sehnke_parameterexploring_2010}.
Our methods are similar to evolution since our value function can be seen as a fitness.
Unlike evolution, however, our approach allows for obtaining the fitness gradient directly and is more suitable for reusing past data.
While direct $V(\theta)$ optimization is strongly related to evolution, our more informed algorithms optimize $V(s,\theta)$ and $Q(s,a,\theta)$.
That is, ours both perform a search in policy parameter space AND train the value function and the policy online, without having to wait for the ends of trials or episodes.\\
\newline
The second related class of methods involves surrogate functions~\citep{box1951, Booker1998, NIPS1995_1124}. They often use local optimizers for generalizing across fitness functions. In particular, Bayesian Optimization (BO)~\citep{snoek2012practical, snoekscalable} uses a surrogate function to evaluate the performance of a model over a set of hyperparameters and follows the uncertainty on the surrogate to query the new data to sample. Unlike BO, we do not build a probabilistic model and we use the gradient of the value function instead of a sample from the posterior to decide which policy parameters to use next in the policy improvement step. \\
\newline
The possibility of augmenting the value functions with auxiliary parameters was already considered in work on General Value Functions~\citep{Sutton:2011:HSR:2031678.2031726}, where the return is defined with respect to an arbitrary reward function. Universal Value Function Approximators~\citep{Schaul:2015:UVF:3045118.3045258} extended this approach to learn a single value function $V^{\pi_{\theta}}(s,g)$, representing the value, given possible agent goals $g$. In particular, they learn different embeddings for states and goals, exploiting their common structure, and they show generalization to new unseen goals. Similarly, our PSVF $V(s,\theta)$ is able to generalize to unseen policies, observing data for only a few $(s,\theta)$ pairs. General and Universal Value Functions have not been applied to learn a single value function for every possible policy.\\
\newline
Policy Evaluation Networks (PENs)~\citep{harb2020policy} are closely related to our work and share the same motivation. PENs focus on the simplest PSSVF $V(\theta)$ trained without an actor-critic architecture. Like in some of our experiments, the authors show how following the direction of improvement suggested by $V(\theta)$ leads to an increase in policy performance. They also suggest to explore in future work a more complex setting where a PSVF $V(s, \theta)$ is learned using an actor-critic architecture. Our work directly introduces the PSVF $V(s,\theta)$ and PAVF $Q(s,a,\theta)$ and presents novel policy gradient theorems for PAVFs when stochastic or deterministic policies are used. There are many differences between our approach to learning $V(\theta)$ and theirs. For example, we do not use a fingerprint mechanism~\citep{harb2020policy} for embedding the weights of complex policies. Instead, we simply parse all the policy weights as inputs to the value function, even in the nonlinear case. Fingerprinting may be important for representing nonlinear policies without losing information about their structure and for saving memory required to store the weights.~\citet{harb2020policy} focus on the offline setting. They first use randomly initialized policies to perform rollouts and collect reward from the environment. Then, once $V(\theta)$ is trained using the data collected, many gradient ascent steps through V yield new, unseen, randomly initialized policies in a zero-shot manner, exhibiting improved performance. They train their value function using small nonlinear policies of one hidden layer and 30 neurons on Swimmer-v3. They evaluate 2000 deterministic policies on 500 episodes each (1 million policy evaluations), achieving a final expected return of $\approx 180$ on new policies trained from scratch through V. On the other hand, in our zero-shot learning experiment using a linear PSSVF, after only 100 policy evaluations, we obtain a return of 297. In our main experiments, we showed that a fingerprint mechanism is not necessary for the tasks we analyzed: even when using a much bigger 2-layers MLP policy, we are able to outperform the results in PEN. Although ~\citet{harb2020policy} use Swimmer-v3 ``to scale up their experiments'', our results suggest that Swimmer-v3 does not conclusively demonstrate possible benefits of their policy embedding.\\ 
\newline
Gradient Temporal Difference~\citep{gtd, fastgd, convergentnonlinear, offpolicycontrol, maei2011gradient} and Emphatic Temporal Difference methods~\citep{sutton2016emphatic} were developed to address convergence under on-policy and off-policy~\citep{Precup2001OffpolicyTD} learning with function approximation. The first attempt to obtain a stable off-policy actor-critic algorithm under linear function approximation was called Off-PAC~\citep{Degris2012}, where the critic is updated using GTD($\lambda$)~\citep{maei2011gradient} to estimate the state-value function. This algorithm converges when using tabular policies. However, in general, the actor does not follow the true gradient direction for $J_b$. A paper on DPG~\citep{Silver2014} extended the Off-PAC policy gradient theorem~\citep{Degris2012} to deterministic policies. This was coupled with a deep neural network to solve continuous control tasks through Deep Deterministic Policy Gradients~\citep{lillicrap2015continuous}.
\citet{imani2018off} used emphatic weights to derive an exact off-policy policy gradient theorem for $J_b$. Differently from Off-PAC, they do not ignore the gradient of the action-value function with respect to the policy, which is incorporated in the emphatic weighting: a vector that needs to be estimated. Our off-policy policy gradients provide an alternative approach that does not need emphatic weights.\\
\newline
The widely used off-policy objective function $J_b$ suffers the distribution shift problem.~\citet{liu2019off} provided an off-policy policy gradient theorem which is unbiased for the true RL objective $J(\pi_{\theta})$, introducing a term $d_{\infty}^{\pi_{\theta}}/                          d_{\infty}^{\pi_b}$ that corrects the mismatch between the states distributions. Despite their sound off-policy formulation, estimating the state weighting ratio remains challenging.
All our algorithms are based on the off-policy actor-critic architecture. The two algorithms based on $Q(s,a,\theta)$ can be viewed as analogous to Off-PAC and DPG where the critic is defined for all policies and the actor is updated following the true gradient with respect to the critic.

\section{Limitations and future work}
We introduced PBVFs, a novel class of value functions which receive as input the parameters of a policy and can be used for off-policy learning.
We showed that PBVFs are competitive to ARS and DDPG~\citep{mania2018simple, lillicrap2015continuous} while generalizing across policies and allowing for zero-shot training in an offline setting.
Despite their positive results on shallow and deep policies, PBVFs suffer the curse of dimensionality when the number of policy parameters is high.
Embeddings similar to those used in PENs~\citep{harb2020policy} may be useful not only for saving memory and computational time, but also for facilitating search in parameter space.
We intend to evaluate the benefits of such embeddings and other dimensionality reduction techniques.
We derived off-policy policy gradient theorems, showing how PBVFs follow the true gradient of the performance $J_b$.
With these results, we plan to analyze the convergence of our algorithms using stochastic approximation techniques~\citep{borkar2009stochastic} and test them on environments where traditional methods are known to diverge~\citep{baird1995residual}.
Finally, we want to investigate how PBVFs applied to supervised learning tasks or POMDPs, can avoid BPTT by mapping the weights of an RNN to its loss.
\label{sec:future}

%

\subsubsection*{Acknowledgments}
We thank Paulo Rauber, Imanol Schlag, Miroslav Strupl, Róbert Csordás, Aleksandar Stanić, Anand Gopalakrishnan, Sjoerd Van Steenkiste and Julius Kunze for their feedback. This work was supported by the ERC Advanced Grant (no: 742870). We also thank NVIDIA Corporation for donating a DGX-1 as part of the Pioneers of AI Research Award and to IBM for donating a Minsky machine.

\bibliography{iclr2021_conference}

\begin{thebibliography}{48}
\providecommand{\natexlab}[1]{#1}
\providecommand{\url}[1]{\texttt{#1}}
\expandafter\ifx\csname urlstyle\endcsname\relax
  \providecommand{\doi}[1]{doi: #1}\else
  \providecommand{\doi}{doi: \begingroup \urlstyle{rm}\Url}\fi

\bibitem[Achiam(2018)]{SpinningUp2018}
Joshua Achiam.
\newblock {Spinning Up in Deep Reinforcement Learning}.
\newblock 2018.

\bibitem[Baird(1995)]{baird1995residual}
Leemon Baird.
\newblock Residual algorithms: Reinforcement learning with function
  approximation.
\newblock In \emph{Machine Learning Proceedings 1995}, pp.\  30--37. Elsevier,
  1995.

\bibitem[Booker et~al.(1998)Booker, Dennis, Frank, Serafini, and
  Torczon]{Booker1998}
Andrew~J. Booker, J.~E. Dennis, Paul~D. Frank, David~B. Serafini, and Virginia
  Torczon.
\newblock \emph{Optimization Using Surrogate Objectives on a Helicopter Test
  Example}, pp.\  49--58.
\newblock Birkh{\"a}user Boston, Boston, MA, 1998.
\newblock ISBN 978-1-4612-1780-0.
\newblock \doi{10.1007/978-1-4612-1780-0_3}.

\bibitem[Borkar(2009)]{borkar2009stochastic}
Vivek~S Borkar.
\newblock \emph{Stochastic approximation: a dynamical systems viewpoint},
  volume~48.
\newblock Springer, 2009.

\bibitem[Box \& Wilson(1951)Box and Wilson]{box1951}
G.~E.~P. Box and K.~B. Wilson.
\newblock On the experimental attainment of optimum conditions.
\newblock \emph{Journal of the Royal Statistical Society. Series B
  (Methodological)}, 13\penalty0 (1):\penalty0 1--45, 1951.
\newblock ISSN 00359246.

\bibitem[Cortes et~al.(2010)Cortes, Mansour, and Mohri]{cortes2010learning}
Corinna Cortes, Yishay Mansour, and Mehryar Mohri.
\newblock Learning bounds for importance weighting.
\newblock In \emph{Advances in neural information processing systems}, pp.\
  442--450, 2010.

\bibitem[Degris et~al.(2012)Degris, White, and Sutton]{Degris2012}
Thomas Degris, Martha White, and Richard~S. Sutton.
\newblock Off-policy actor-critic.
\newblock In \emph{Proceedings of the 29th International Coference on
  International Conference on Machine Learning}, ICML'12, pp.\  179--186, USA,
  2012. Omnipress.
\newblock ISBN 978-1-4503-1285-1.

\bibitem[Harb et~al.(2020)Harb, Schaul, Precup, and Bacon]{harb2020policy}
Jean Harb, Tom Schaul, Doina Precup, and Pierre-Luc Bacon.
\newblock Policy evaluation networks.
\newblock \emph{arXiv preprint arXiv:2002.11833}, 2020.

\bibitem[Hesterberg(1988)]{hesterberg1988advances}
Timothy~Classen Hesterberg.
\newblock \emph{Advances in importance sampling}.
\newblock PhD thesis, Stanford University, 1988.

\bibitem[Imani et~al.(2018)Imani, Graves, and White]{imani2018off}
Ehsan Imani, Eric Graves, and Martha White.
\newblock An off-policy policy gradient theorem using emphatic weightings.
\newblock In \emph{Advances in Neural Information Processing Systems}, pp.\
  96--106, 2018.

\bibitem[Jaderberg et~al.(2017)Jaderberg, Czarnecki, Osindero, Vinyals, Graves,
  Silver, and Kavukcuoglu]{jaderberg2017decoupled}
Max Jaderberg, Wojciech~Marian Czarnecki, Simon Osindero, Oriol Vinyals, Alex
  Graves, David Silver, and Koray Kavukcuoglu.
\newblock Decoupled neural interfaces using synthetic gradients.
\newblock In \emph{Proceedings of the 34th International Conference on Machine
  Learning-Volume 70}, pp.\  1627--1635. JMLR. org, 2017.

\bibitem[Konda \& Tsitsiklis(2001)Konda and Tsitsiklis]{kondaactorcritic}
Vijay Konda and John Tsitsiklis.
\newblock Actor-critic algorithms.
\newblock \emph{Society for Industrial and Applied Mathematics}, 42, 04 2001.

\bibitem[Lillicrap et~al.(2015)Lillicrap, Hunt, Pritzel, Heess, Erez, Tassa,
  Silver, and Wierstra]{lillicrap2015continuous}
Timothy~P Lillicrap, Jonathan~J Hunt, Alexander Pritzel, Nicolas Heess, Tom
  Erez, Yuval Tassa, David Silver, and Daan Wierstra.
\newblock Continuous control with deep reinforcement learning.
\newblock \emph{arXiv preprint arXiv:1509.02971}, 2015.

\bibitem[Liu et~al.(2019)Liu, Swaminathan, Agarwal, and Brunskill]{liu2019off}
Yao Liu, Adith Swaminathan, Alekh Agarwal, and Emma Brunskill.
\newblock Off-policy policy gradient with state distribution correction.
\newblock \emph{arXiv preprint arXiv:1904.08473}, 2019.

\bibitem[Maei et~al.(2009)Maei, Szepesv\'{a}ri, Bhatnagar, Precup, Silver, and
  Sutton]{convergentnonlinear}
Hamid~R. Maei, Csaba Szepesv\'{a}ri, Shalabh Bhatnagar, Doina Precup, David
  Silver, and Richard~S. Sutton.
\newblock Convergent temporal-difference learning with arbitrary smooth
  function approximation.
\newblock In \emph{Proceedings of the 22nd International Conference on Neural
  Information Processing Systems}, NIPS’09, pp.\  1204–1212, Red Hook, NY,
  USA, 2009. Curran Associates Inc.
\newblock ISBN 9781615679119.

\bibitem[Maei(2011)]{maei2011gradient}
Hamid~Reza Maei.
\newblock \emph{Gradient temporal-difference learning algorithms}.
\newblock PhD thesis, University of Alberta, 2011.

\bibitem[Maei et~al.(2010)Maei, Szepesv\'{a}ri, Bhatnagar, and
  Sutton]{offpolicycontrol}
Hamid~Reza Maei, Csaba Szepesv\'{a}ri, Shalabh Bhatnagar, and Richard~S.
  Sutton.
\newblock Toward off-policy learning control with function approximation.
\newblock In \emph{Proceedings of the 27th International Conference on
  International Conference on Machine Learning}, ICML’10, pp.\  719–726,
  Madison, WI, USA, 2010. Omnipress.
\newblock ISBN 9781605589077.

\bibitem[Mania et~al.(2018)Mania, Guy, and Recht]{mania2018simple}
Horia Mania, Aurelia Guy, and Benjamin Recht.
\newblock Simple random search of static linear policies is competitive for
  reinforcement learning.
\newblock In \emph{Advances in Neural Information Processing Systems}, pp.\
  1800--1809, 2018.

\bibitem[Metelli et~al.(2018)Metelli, Papini, Faccio, and
  Restelli]{metelli2018policy}
Alberto~Maria Metelli, Matteo Papini, Francesco Faccio, and Marcello Restelli.
\newblock Policy optimization via importance sampling.
\newblock In \emph{Advances in Neural Information Processing Systems}, pp.\
  5442--5454, 2018.

\bibitem[Metropolis \& Ulam(1949)Metropolis and Ulam]{MET49}
N.~Metropolis and S.~Ulam.
\newblock The monte carlo method.
\newblock \emph{J.~Am.~Stat.~Assoc.}, 44:\penalty0 335, 1949.

\bibitem[Mnih et~al.(2015)Mnih, Kavukcuoglu, Silver, Rusu, Veness, Bellemare,
  Graves, Riedmiller, Fidjeland, Ostrovski, et~al.]{mnih2015human}
Volodymyr Mnih, Koray Kavukcuoglu, David Silver, Andrei~A Rusu, Joel Veness,
  Marc~G Bellemare, Alex Graves, Martin Riedmiller, Andreas~K Fidjeland, Georg
  Ostrovski, et~al.
\newblock Human-level control through deep reinforcement learning.
\newblock \emph{Nature}, 518\penalty0 (7540):\penalty0 529--533, 2015.

\bibitem[Moore \& Schneider(1996)Moore and Schneider]{NIPS1995_1124}
Andrew~W. Moore and Jeff~G. Schneider.
\newblock Memory-based stochastic optimization.
\newblock In D.~S. Touretzky, M.~C. Mozer, and M.~E. Hasselmo (eds.),
  \emph{Advances in Neural Information Processing Systems 8}, pp.\  1066--1072.
  MIT Press, 1996.

\bibitem[Nachum et~al.(2019)Nachum, Dai, Kostrikov, Chow, Li, and
  Schuurmans]{nachum2019algaedice}
Ofir Nachum, Bo~Dai, Ilya Kostrikov, Yinlam Chow, Lihong Li, and Dale
  Schuurmans.
\newblock Algaedice: Policy gradient from arbitrary experience.
\newblock \emph{arXiv preprint arXiv:1912.02074}, 2019.

\bibitem[Peters \& Schaal(2008)Peters and
  Schaal]{Peters:2008:NA:1352927.1352986}
Jan Peters and Stefan Schaal.
\newblock Natural actor-critic.
\newblock \emph{Neurocomput.}, 71\penalty0 (7-9):\penalty0 1180--1190, March
  2008.
\newblock ISSN 0925-2312.
\newblock \doi{10.1016/j.neucom.2007.11.026}.

\bibitem[Precup et~al.(2001)Precup, Sutton, and
  Dasgupta]{Precup2001OffpolicyTD}
Doina Precup, Richard~S. Sutton, and Sanjoy Dasgupta.
\newblock Off-policy temporal difference learning with function approximation.
\newblock In \emph{ICML}, 2001.

\bibitem[Puterman(2014)]{puterman2014markov}
Martin~L Puterman.
\newblock \emph{Markov decision processes: discrete stochastic dynamic
  programming}.
\newblock John Wiley \& Sons, 2014.

\bibitem[Rubinstein \& Kroese(2016)Rubinstein and Kroese]{regen}
Reuven~Y. Rubinstein and Dirk~P. Kroese.
\newblock \emph{Simulation and the Monte Carlo Method}.
\newblock Wiley Publishing, 3rd edition, 2016.
\newblock ISBN 1118632168.

\bibitem[Salimans et~al.(2017)Salimans, Ho, Chen, Sidor, and
  Sutskever]{salimans2017evolution}
Tim Salimans, Jonathan Ho, Xi~Chen, Szymon Sidor, and Ilya Sutskever.
\newblock Evolution strategies as a scalable alternative to reinforcement
  learning.
\newblock \emph{arXiv preprint arXiv:1703.03864}, 2017.

\bibitem[Schaul et~al.(2015)Schaul, Horgan, Gregor, and
  Silver]{Schaul:2015:UVF:3045118.3045258}
Tom Schaul, Dan Horgan, Karol Gregor, and David Silver.
\newblock Universal value function approximators.
\newblock In \emph{Proceedings of the 32Nd International Conference on
  International Conference on Machine Learning - Volume 37}, ICML'15, pp.\
  1312--1320. JMLR.org, 2015.

\bibitem[Schmidhuber(1990)]{schmidhuber1990networks}
J{\"u}rgen Schmidhuber.
\newblock Networks adjusting networks.
\newblock In \emph{Proceedings of" Distributed Adaptive Neural Information
  Processing"}, pp.\  197--208, 1990.

\bibitem[Sehnke et~al.(2008)Sehnke, Osendorfer, R{\"u}ckstie{\ss}, Graves,
  Peters, and Schmidhuber]{sehnkepgpecontrol}
Frank Sehnke, Christian Osendorfer, Thomas R{\"u}ckstie{\ss}, Alex Graves, Jan
  Peters, and J{\"u}rgen Schmidhuber.
\newblock Policy gradients with parameter-based exploration for control.
\newblock In V{\'e}ra K{\r{u}}rkov{\'a}, Roman Neruda, and Jan Koutn{\'i}k
  (eds.), \emph{Artificial Neural Networks - ICANN 2008}, pp.\  387--396,
  Berlin, Heidelberg, 2008. Springer Berlin Heidelberg.
\newblock ISBN 978-3-540-87536-9.

\bibitem[Sehnke et~al.(2010)Sehnke, Osendorfer, Rückstieß, Graves, Peters,
  and Schmidhuber]{sehnke_parameterexploring_2010}
Frank Sehnke, Christian Osendorfer, Thomas Rückstieß, Alex Graves, Jan
  Peters, and Jürgen Schmidhuber.
\newblock Parameter-exploring policy gradients.
\newblock \emph{Neural Networks}, 23\penalty0 (4):\penalty0 551--559, May 2010.
\newblock ISSN 08936080.
\newblock \doi{10.1016/j.neunet.2009.12.004}.

\bibitem[Silver et~al.(2014)Silver, Lever, Heess, Degris, Wierstra, and
  Riedmiller]{Silver2014}
David Silver, Guy Lever, Nicolas Heess, Thomas Degris, Daan Wierstra, and
  Martin Riedmiller.
\newblock Deterministic policy gradient algorithms.
\newblock In \emph{Proceedings of the 31st International Conference on
  International Conference on Machine Learning - Volume 32}, ICML'14, pp.\
  I--387--I--395. JMLR.org, 2014.

\bibitem[Snoek et~al.(2012)Snoek, Larochelle, and Adams]{snoek2012practical}
Jasper Snoek, Hugo Larochelle, and Ryan~P Adams.
\newblock Practical bayesian optimization of machine learning algorithms.
\newblock In \emph{Advances in neural information processing systems}, pp.\
  2951--2959, 2012.

\bibitem[Snoek et~al.(2015)Snoek, Rippel, Swersky, Kiros, Satish, Sundaram,
  Patwary, Prabhat, and Adams]{snoekscalable}
Jasper Snoek, Oren Rippel, Kevin Swersky, Ryan Kiros, Nadathur Satish,
  Narayanan Sundaram, Md. Mostofa~Ali Patwary, Prabhat Prabhat, and Ryan~P.
  Adams.
\newblock Scalable bayesian optimization using deep neural networks.
\newblock In \emph{Proceedings of the 32nd International Conference on
  International Conference on Machine Learning - Volume 37}, ICML’15, pp.\
  2171–2180. JMLR.org, 2015.

\bibitem[Stratonovich(1960)]{stratonovich1960}
RL~Stratonovich.
\newblock Conditional {Markov} processes.
\newblock \emph{Theory of Probability And Its Applications}, 5\penalty0
  (2):\penalty0 156--178, 1960.

\bibitem[Sutton(1984)]{sutton1984temporal}
Richard~S Sutton.
\newblock \emph{Temporal Credit Assignment in Reinforcement Learning}.
\newblock PhD thesis, University of Massachusetts Amherst, 1984.

\bibitem[Sutton(1988)]{sutton1988learning}
Richard~S Sutton.
\newblock Learning to predict by the methods of temporal differences.
\newblock \emph{Machine learning}, 3\penalty0 (1):\penalty0 9--44, 1988.

\bibitem[Sutton et~al.(1999)Sutton, McAllester, Singh, and Mansour]{Sutton1999}
Richard~S. Sutton, David McAllester, Satinder Singh, and Yishay Mansour.
\newblock Policy gradient methods for reinforcement learning with function
  approximation.
\newblock In \emph{Proceedings of the 12th International Conference on Neural
  Information Processing Systems}, NIPS'99, pp.\  1057--1063, Cambridge, MA,
  USA, 1999. MIT Press.

\bibitem[Sutton et~al.(2009{\natexlab{a}})Sutton, Maei, and
  Szepesv{\'a}ri]{gtd}
Richard~S Sutton, Hamid~R Maei, and Csaba Szepesv{\'a}ri.
\newblock A convergent $ o (n) $ temporal-difference algorithm for off-policy
  learning with linear function approximation.
\newblock In \emph{Advances in neural information processing systems}, pp.\
  1609--1616, 2009{\natexlab{a}}.

\bibitem[Sutton et~al.(2009{\natexlab{b}})Sutton, Maei, Precup, Bhatnagar,
  Silver, Szepesv\'{a}ri, and Wiewiora]{fastgd}
Richard~S. Sutton, Hamid~Reza Maei, Doina Precup, Shalabh Bhatnagar, David
  Silver, Csaba Szepesv\'{a}ri, and Eric Wiewiora.
\newblock Fast gradient-descent methods for temporal-difference learning with
  linear function approximation.
\newblock In \emph{Proceedings of the 26th Annual International Conference on
  Machine Learning}, ICML ’09, pp.\  993–1000, New York, NY, USA,
  2009{\natexlab{b}}. Association for Computing Machinery.
\newblock ISBN 9781605585161.
\newblock \doi{10.1145/1553374.1553501}.

\bibitem[Sutton et~al.(2011)Sutton, Modayil, Delp, Degris, Pilarski, White, and
  Precup]{Sutton:2011:HSR:2031678.2031726}
Richard~S. Sutton, Joseph Modayil, Michael Delp, Thomas Degris, Patrick~M.
  Pilarski, Adam White, and Doina Precup.
\newblock Horde: A scalable real-time architecture for learning knowledge from
  unsupervised sensorimotor interaction.
\newblock In \emph{The 10th International Conference on Autonomous Agents and
  Multiagent Systems - Volume 2}, AAMAS '11, pp.\  761--768, Richland, SC,
  2011. International Foundation for Autonomous Agents and Multiagent Systems.
\newblock ISBN 0-9826571-6-1, 978-0-9826571-6-4.

\bibitem[Sutton et~al.(2016)Sutton, Mahmood, and White]{sutton2016emphatic}
Richard~S Sutton, A~Rupam Mahmood, and Martha White.
\newblock An emphatic approach to the problem of off-policy temporal-difference
  learning.
\newblock \emph{The Journal of Machine Learning Research}, 17\penalty0
  (1):\penalty0 2603--2631, 2016.

\bibitem[Tesauro(1995)]{tesauro1995temporal}
Gerald Tesauro.
\newblock Temporal difference learning and td-gammon.
\newblock \emph{Commun. ACM}, 38\penalty0 (3):\penalty0 58–68, March 1995.
\newblock ISSN 0001-0782.
\newblock \doi{10.1145/203330.203343}.

\bibitem[Unterthiner et~al.(2020)Unterthiner, Keysers, Gelly, Bousquet, and
  Tolstikhin]{unterthiner2020predicting}
Thomas Unterthiner, Daniel Keysers, Sylvain Gelly, Olivier Bousquet, and Ilya
  Tolstikhin.
\newblock Predicting neural network accuracy from weights.
\newblock \emph{arXiv preprint arXiv:2002.11448}, 2020.

\bibitem[Wang et~al.(2016)Wang, Bapst, Heess, Mnih, Munos, Kavukcuoglu, and
  de~Freitas]{wang2016sample}
Ziyu Wang, Victor Bapst, Nicolas Heess, Volodymyr Mnih, Remi Munos, Koray
  Kavukcuoglu, and Nando de~Freitas.
\newblock Sample efficient actor-critic with experience replay.
\newblock \emph{arXiv preprint arXiv:1611.01224}, 2016.

\bibitem[Werbos(1990)]{werbos1990backpropagation}
Paul~J Werbos.
\newblock Backpropagation through time: what it does and how to do it.
\newblock \emph{Proceedings of the IEEE}, 78\penalty0 (10):\penalty0
  1550--1560, 1990.

\bibitem[Wierstra et~al.(2014)Wierstra, Schaul, Glasmachers, Sun, Peters, and
  Schmidhuber]{wierstra2014natural}
Daan Wierstra, Tom Schaul, Tobias Glasmachers, Yi~Sun, Jan Peters, and
  J{\"u}rgen Schmidhuber.
\newblock Natural evolution strategies.
\newblock \emph{The Journal of Machine Learning Research}, 15\penalty0
  (1):\penalty0 949--980, 2014.

\end{thebibliography}
\bibliographystyle{iclr2021_conference}

\clearpage
\appendix
\section{Appendix}
\subsection*{Index of the appendix}
In the following, we briefly recap the contents of the appendix.
\begin{itemize}
	\item Appendix~\ref{apx:relatedWorks} contains additional related works
	\item Appendix~\ref{apx:proofs} reports all proofs and derivations.
	\item Appendix~\ref{apx:impl} illustrates implementation details and pseudocode.
	\item Appendix~\ref{apx:exp} provides the hyperparameters used in the experiments and further results.
\end{itemize}

\subsection{Additional related works}
\label{apx:relatedWorks}
Recent work~\citep{unterthiner2020predicting} shows how to map the weights of a trained Convolutional Neural Network to its accuracy. Experiments show how these predictions allow for performance rankings of neural networks on new unseen tasks. These maps are either learned by taking the flattened weights as input or using simple statistics. However, these predictions do not guide the training process of CNNs. \\
\newline
In 1990, adaptive critics trained by TD were used to predict the gradients of an RNN from its activations~\citep{schmidhuber1990networks}, avoiding backpropagation through time (BPTT)~\citep{werbos1990backpropagation}. This idea was later used to update the weights of a neural network asynchronously~\citep{jaderberg2017decoupled}. In our work, the critic is predicting errors instead of gradients. If applied to POMDPs, or supervised learning tasks involving long time lags between relevant events, the PSSVF could avoid BPTT by viewing the parameters of an RNN as a static object and mapping them to their loss (negative reward).\\
\newline
 Additional differences between our work and Policy Evaluation Networks (PENs)~\citep{harb2020policy} concern the optimization problem: we do not predict a bucket index for discretized reward, but perform a regression task. Therefore our loss is simply the mean squared error between the prediction of $V(\theta)$ and the reward obtained by $\pi_{\theta}$, while their loss~\citep{harb2020policy} is the KL divergence between the predicted and target distributions. Both approaches optimize the undiscounted objective when learning $V(\theta)$.\\
\newline

\subsection{Proofs and derivations}
\label{apx:proofs}
\onspavf*
\begin{proof}
The proof follows the standard approach by ~\citet{Sutton1999} and we report it for completeness. We start by deriving an expression for $\nabla_{\theta}V(s,\theta)$:


    
    \begin{align*}
        \nabla_{\theta}V(s,\theta) & = \nabla_{\theta} \int_{\mathcal{A}} \pi_{\theta}(a|s) Q(s,a,\theta) \de a = \int_{\mathcal{A}} \nabla_{\theta} \pi_{\theta}(a|s) Q(s,a,\theta) + \pi_{\theta}(a|s) \nabla_{\theta} Q(s,a,\theta)\de a\\
        & =  \int_{\mathcal{A}} \nabla_{\theta} \pi_{\theta}(a|s) Q(s,a,\theta) + \pi_{\theta}(a|s) \nabla_{\theta} \left( R(s,a) + \gamma \int_{\mathcal{S}} P(s'|s,a) V(s', \theta) \de s'  \right)\de a\\
        & = \int_{\mathcal{A}} \nabla_{\theta} \pi_{\theta}(a|s) Q(s,a,\theta) + \pi_{\theta}(a|s) \gamma \int_{\mathcal{S}} P(s'|s,a) \nabla_{\theta} V(s', \theta) \de s'  \de a\\
        & = \int_{\mathcal{A}} \nabla_{\theta} \pi_{\theta}(a|s) Q(s,a,\theta) +  \pi_{\theta}(a|s) \gamma \int_{\mathcal{S}} P(s'|s,a)  \times \\
        & \times \int_{\mathcal{A}} \nabla_{\theta} \pi_{\theta}(a'|s') Q(s',a',\theta) + \pi_{\theta}(a'|s') \gamma  \int_{\mathcal{S}} P(s^{''}|s',a') \nabla_{\theta} V(s^{''}, \theta) \de s^{''}  \de a' \de s' \de a\\
        = & \int_{\mathcal{S}} \sum_{t=0}^{\infty} \gamma^t P(s \rightarrow s', t, \pi_{\theta}) \int_{\mathcal{A}} \nabla_{\theta} \pi_{\theta}(a|s') Q(s',a,\theta) \de a \de s'.
    \end{align*}
    
Taking the expectation with respect to $s_0 \sim \mu_0(s)$ we have:
    \begin{align*}
        \nabla_{\theta} J(\theta) & = \nabla_{\theta} \int_{\mathcal{S}} \mu_0(s) V(s,\theta) \de s = \int_{\mathcal{S}} \mu_0(s) \nabla_{\theta} V(s,\theta) \de s\\
        & = \int_{\mathcal{S}} \mu_0(s) \int_{\mathcal{S}} \sum_{t=0}^{\infty} \gamma^t P(s \rightarrow s', t, \pi_{\theta}) \int_{\mathcal{A}} \nabla_{\theta} \pi_{\theta}(a|s) Q(s,a,\theta) \de s' \de a \de s \\
        & = \int_{\mathcal{S}} d^{\pi_{\theta}}(s) \int_{\mathcal{A}} \nabla_{\theta} \pi_{\theta}(a|s) Q(s,a,\theta) \de a \de s \\
        & = {\ev}_{s \sim d^{\pi_{\theta}}(s), a \sim \pi_{\theta}(.|s)}\left[\left(Q(s,a,\theta) \nabla_{\theta} \log\pi_{\theta}(a|s)\right)\right].
    \end{align*}


\end{proof}

\ondpavf*
\begin{proof}
    The proof follows the standard approach by ~\citet{Silver2014} and we report it for completeness. We start by deriving an expression for $\nabla_{\theta}V(s,\theta)$:
    \begin{align*}
         \nabla_{\theta}V(s,\theta) & = \nabla_{\theta}Q(s, \pi_{\theta}(s),\theta) = \nabla_{\theta} \left( R(s,\pi_{\theta}(s)) + \gamma \int_{\mathcal{S}} P(s'|s,\pi_{\theta}(s)) V(s', \theta) \de s'  \right)\\
         & = \nabla_{\theta} \pi_{\theta}(s) \nabla_{a} R(s,a)|_{a=\pi_{\theta}(s)} +  \\
         &  +\gamma \int_{\mathcal{S}} P(s'|s,\pi_{\theta}(s)) \nabla_{\theta} V(s', \theta) + \nabla_{\theta} \pi_{\theta}(s) \nabla_{a} P(s'|s,a)|_{a=\pi_{\theta}(s)} \de s' \\
         & = \nabla_{\theta} \pi_{\theta}(s) \nabla_{a} \left(R(s,a) + \gamma \int_{\mathcal{S}} P(s'|s,a) V(s', \theta) \de s' \right)|_{a=\pi_{\theta}(s)} + \\
         & +\gamma \int_{\mathcal{S}} P(s'|s,\pi_{\theta}(s)) \nabla_{\theta} V(s', \theta) \de s' \\
         & =  \nabla_{\theta} \pi_{\theta}(s) \nabla_{a}Q(s,a,\theta)|_{a=\pi_{\theta}(s)} + \gamma \int_{\mathcal{S}} P(s'|s,\pi_{\theta}(s)) \nabla_{\theta} V(s', \theta) \de s'\\
         & =  \nabla_{\theta} \pi_{\theta}(s) \nabla_{a}Q(s,a,\theta)|_{a=\pi_{\theta}(s)} + \\
         & +\gamma \int_{\mathcal{S}} P(s'|s,\pi_{\theta}(s)) \nabla_{\theta} \pi_{\theta}(s') \nabla_{a}Q(s',a,\theta)|_{a=\pi_{\theta}(s')} \de s' + \\
         & + \gamma \int_{\mathcal{S}} P(s'|s,\pi_{\theta}(s)) \gamma \int_{\mathcal{S}} P(s^{''}|s',\pi_{\theta}(s')) \nabla_{\theta} V(s^{''}, \theta) \de s^{''} \de s'\\
         & = \int_{\mathcal{S}} \sum_{t=0}^{\infty} \gamma^t P(s \rightarrow s', t, \pi_{\theta})\nabla_{\theta} \pi_{\theta}(s') \nabla_{a}Q(s',a,\theta)|_{a=\pi_{\theta}(s')}  \de s'
    \end{align*}

Taking the expectation with respect to $s_0 \sim \mu_0(s)$ we have:
    \begin{align*}
    \nabla_{\theta} J(\theta) & = \nabla_{\theta}\int_{\mathcal{S}} \mu_0(s) V(s,\theta) \de s =\int_{\mathcal{S}} \mu_0(s) \nabla_{\theta} V(s,\theta) \de s\\
    & =  \int_{\mathcal{S}} \mu_0(s)  \int_{\mathcal{S}} \sum_{t=0}^{\infty} \gamma^t P(s \rightarrow s', t, \pi_{\theta})\nabla_{\theta} \pi_{\theta}(s') \nabla_{a}Q(s',a,\theta)|_{a=\pi_{\theta}(s')}  \de s' \de s \\
    & = \int_{\mathcal{S}} d^{\pi_{\theta}}(s) \nabla_{\theta} \pi_{\theta}(s) \nabla_{a}Q(s,a,\theta)|_{a=\pi_{\theta}(s)} \de s \\
    & = {\ev}_{s \sim d^{\pi_{\theta}}(s)} \left[  \nabla_{\theta} \pi_{\theta}(s) \nabla_{a}Q(s,a,\theta)|_{a=\pi_{\theta}(s)}    \right]
    \end{align*}

\end{proof}

\spavf*
\begin{proof}
    \begin{align}
        \nabla_{\theta}J_b(\pi_{\theta}) & = \nabla_{\theta} \int_{\mathcal{S}} d_{\infty}^{\pi_{b}}(s) V(s,\theta)  \de s \\
        & = \nabla_{\theta} \int_{\mathcal{S}} d_{\infty}^{\pi_{b}}(s) \int_{\mathcal{A}}  \pi_{\theta}(a|s)  Q(s,a,\theta)  \de a \de s \\
        & =  \int_{\mathcal{S}} d_{\infty}^{\pi_{b}}(s) \int_{\mathcal{A}} [Q(s,a,\theta) \nabla_{\theta} \pi_{\theta}(a|s) +  \pi_{\theta}(a|s) \nabla_{\theta}Q(s,a,\theta)]  \de a \de s\\
        & = \int_{\mathcal{S}} d_{\infty}^{\pi_{b}}(s) \int_{\mathcal{A}} \frac{\pi_{b}(a|s)}{\pi_{b}(a|s)}\pi_{\theta}(a|s) [Q(s,a,\theta) \nabla_{\theta} \log\pi_{\theta}(a|s) +  \nabla_{\theta}Q(s,a,\theta)]  \de a \de s\\
        & = {\ev}_{s \sim d_{\infty}^{\pi_{b}}(s), a \sim \pi_{b}(.|s)}\left[\frac{\pi_{\theta}(a|s)}{\pi_b(a|s)} \left(Q(s,a,\theta) \nabla_{\theta} \log\pi_{\theta}(a|s) +  \nabla_{\theta}Q(s,a,\theta)\right)\right]
    \end{align}
\end{proof}

\dpavf*
\begin{proof}
    \begin{align}
        \nabla_{\theta}J_b(\pi_{\theta}) & = \int_{\mathcal{S}} d_{\infty}^{\pi_{b}}(s) \nabla_{\theta}   Q(s,\pi_{\theta}(s),\theta)  \de s \\
        & =  \int_{\mathcal{S}} d_{\infty}^{\pi_{b}}(s) \left[ \nabla_{a}   Q(s,a,\theta)|_{a=\pi_{\theta}(s)} \nabla_{\theta} \pi_{\theta}(s) + \nabla_{\theta} Q(s,a,\theta)|_{a=\pi_{\theta}(s)} \right]  \de s \\
        & = {\ev}_{s \sim d_{\infty}^{\pi_{b}}(s)} \left[   \nabla_{a}   Q(s,a,\theta)|_{a=\pi_{\theta}(s)} \nabla_{\theta} \pi_{\theta}(s) + \nabla_{\theta} Q(s,a,\theta)|_{a=\pi_{\theta}(s)}   \right]
    \end{align}
\end{proof}

\subsection{Implementation details}
\label{apx:impl}
\subsubsection{}
In this appendix, we report the implementation details for PSSVF, PSVF, PAVF and the baselines. We specify for each hyperparameter, which algorithms and tasks are sharing them.\\
\newline
Shared hyperparameters:
\begin{itemize}
    \item Deterministic policy architecture (continuous control tasks): We use three different deterministic policies: a linear mapping between states and actions; a single-layer MLP with 32 neurons and tanh activation; a 2-layers MLP (64,64) with tanh activations. All policies contain a bias term and are followed by a tanh nonlinearity in order to bound the action. 
    \item Deterministic policy architecture (discrete control tasks): We use three different deterministic policies: a linear mapping between states and a probability distribution over actions; a single-layer MLP with 32 neurons and tanh activation; a 2-layers MLP (64,64) with tanh activations. The deterministic action $a$ is obtained choosing $a = \argmax \pi_{\theta}(a|s)$. All policies contain a bias term.
    \item Stochastic policy architecture (continuous control tasks): We use three different stochastic policies: a linear mapping; a single-layer MLP with 32 neurons and tanh activation; a 2-layers MLP (64,64) with tanh activations all mapping from states to the mean of a Normal distribution. The variance is state-independent and parametrized as $e^{2\Omega}$ with diagonal $\Omega$. All policies contain a bias term. Actions sampled are given as input to a tanh nonlinearity in order to bound them in the action space. 
    \item Stochastic policy architecture (discrete control tasks): We use three different deterministic policies: a linear mapping between states and a probability distribution over actions; a single-layer MLP with 32 neurons and tanh activation; a 2-layers MLP (64,64) with tanh activations. All policies contain a bias term.
    \item Policy initialization: all weights and biases are initialized using the default Pytorch initialization for PBVFs and DDPG and are set to zero for ARS.
    \item Critic architecture: 2-layers MLP (512,512) with bias and ReLU activation functions for PSVF, PAVF; 2-layers MLP (256,256) with bias and ReLU activation functions for DDPG.
    \item Critic initialization: all weights and biases are initialized using the default Pytorch initialization for PBVFs and DDPG.
    \item Batch size: 128 for DDPG, PSVF, PAVF; 16 for PSSVF.
    \item Actor's frequency of updates: every episode for PSSVF; every batch of episodes for ARS; every 50 time steps for DDPG, PSVF, PAVF.
    \item Critic's frequency of updates: every episode for PSSVF; every 50 time steps for DDPG, PSVF, PAVF.
    \item Replay buffer: the size is 100k; data are sampled uniformly.
    \item Optimizer: Adam for PBVFs and DDPG.
\end{itemize}
Tuned hyperparameters:
\begin{itemize}
    \item Number of directions and elite directions for ARS ([directions, elite directions]): tuned with values in $[[1,1], [4,1], [4,4], [16,1], [16,4], [16,16]]$.
    \item Policy's learning rate: tuned with values in $[1e-2, 1e-3, 1e-4]$.
    \item Critic's learning rate: tuned with values in $[1e-2, 1e-3, 1e-4]$.
    \item Noise for exploration: the perturbation for the action (DDPG) or the parameter is sampled from $\mathcal{N}(0, \sigma I)$ with $\sigma$ tuned with values in $[1, 1e-1]$ for PSSVF, PSVF, PAVF; $[1e-1, 1e-2]$ for DDPG; $[1, 1e-1, 1e-2, 1e-3]$ for ARS. For stochastic PSSVF and PSVF we include also the value $\sigma=0$, although it almost never results optimal.
    
    \end{itemize}
Environment hyperparameters:
\begin{itemize}
    \item Environment interactions: 1M time steps for Swimmer-v3 and Hopper-v3; 100k time steps for all other environments.
    \item Discount factor for TD algorithms: 0.999 for Swimmer; 0.99 for all other environments.
    \item Survival reward in Hopper: True for DDPG, PSVF, PAVF; False for ARS, PSSVF.
\end{itemize}

Algorithm-specific hyperparameters:
\begin{itemize}
    \item Critic's number of updates: 50 for DDPG, 5 for PSVF and PAVF; 10 for PSSVF.
    \item Actor's number of updates: 50 for DDPG, 1 for PSVF and PAVF; 10 for PSSVF.
    \item Observation normalization: False for DDPG; True for all other algorithms.
    \item Starting steps in DDPG (random actions and no training): first $1\%$.
    \item Polyak parameter in DDPG: 0.995.
    
\end{itemize}

\paragraph{PAVF $\nabla_\theta Q(s,a,\theta)$ ablation}
We investigate the effect of the term $\nabla_\theta Q(s,a,\theta)$ in the off-policy policy gradient theorem for deterministic PAVF. We follow the same methodology as in our main experiments to find the optimal hyperparameters when updating using the now biased gradient:
	\begin{equation}
	    \nabla_{\theta} J_b(\pi_{\theta}) \approx {\ev}_{s \sim d_{\infty}^{\pi_{b}}(s)} \left[   \nabla_{a} Q(s,a,\theta)|_{a=\pi_{\theta}(s)} \nabla_{\theta} \pi_{\theta}(s) \right],
	\end{equation}
which corresponds to the gradient that DDPG is following. Figure~\ref{fig:ablation_pavf} reports the results for Hopper and Swimmer using shallow and deep policies. We observe a significant drop in performance in Swimmer when removing part of the gradient. In Hopper the loss of performance is less significant, possibly because both algorithms tend to converge to the same sub-optimal behavior.

\begin{figure}[h]
\begin{center}
\centering
\includegraphics[width=1.0\linewidth]{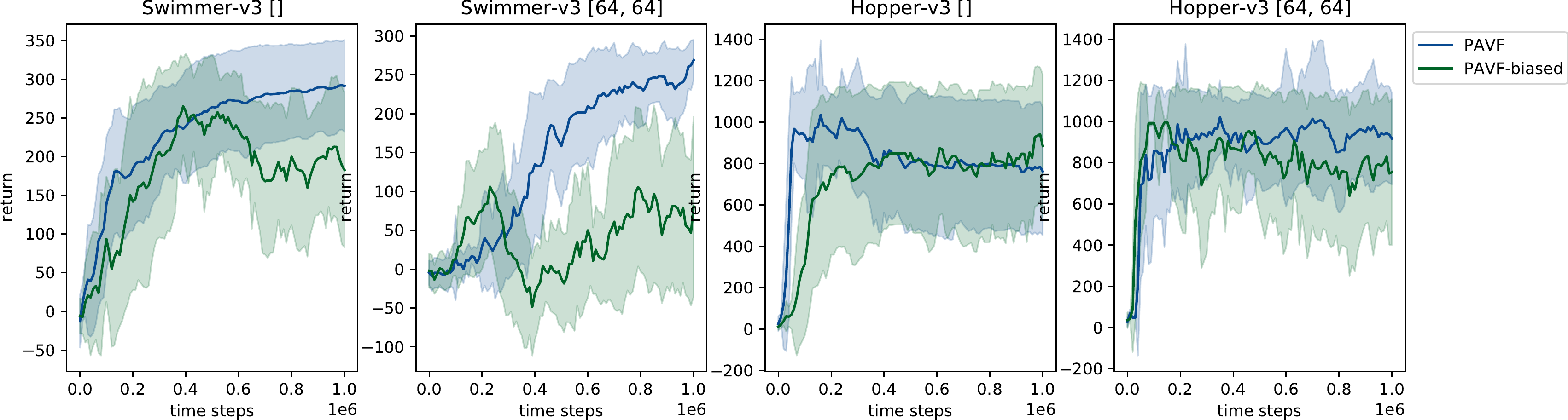}
\end{center}
\caption{Performance of PAVF and biased PAVF (PAVF without the gradient of the action-value function with respect to the policy parameters) using deterministic policies. We use the hyperparameters maximizing average return and report the best hyperparameters found for the biased version in Table~\ref{tab:hyp_bpavf}. Learning curves are averaged over 20 seeds.}

\label{fig:ablation_pavf}
\end{figure}

\begin{table}[!h]
  \caption{Table of best hyperparameters for biased PAVFs}
  \begin{tabular}{lrcc}
    \hline
    \textbf{Learning rate policy}  & Policy: &
     [] & [64,64] \\
    &Metric: &avg  &avg    \\
    \hline
    Swimmer-v3 & &1e-3 & 1e-4 \\
    Hopper-v3 & &1e-4 & 1e-4 \\
    \hline
    \textbf{Learning rate critic} & & &\\
    \hline
    Swimmer-v3 & &1e-4 & 1e-4 \\
    Hopper-v3 & &1e-3 & 1e-3 \\
    \hline
    \textbf{Noise for exploration} & & &\\
    \hline
    Swimmer-v3 & & 1.0 & 1.0\\
    Hopper-v3 & & 0.1 & 0.1 \\
    \hline
  \end{tabular}
    \label{tab:hyp_bpavf}

\end{table}

\paragraph{ARS}
For ARS, we used the official implementation provided by the authors and we modified it in order to use nonlinear policies. More precisely, we used the implementation of ARSv2-t~\citep{mania2018simple}, which uses observation normalization, elite directions and an adaptive learning rate based on the standard deviation of the return collected. To avoid divisions by zero, which may happen if all data sampled have the same return, we perform the standardization only in case the standard deviation is not zero. In the original implementation of ARS~\citep{mania2018simple}, the survival bonus for the reward in the Hopper environment is removed to avoid local minima. Since we wanted our PSSVF to be close to their setting, we also applied this modification. We did not remove the survival bonus from all TD algorithms and we did not investigate how this could affect their performance. We provide a comparison of the performance of PSSVF with and without the bonus in figure~\ref{fig:ablation_pvf} using deterministic policies.
\begin{figure}[h]
\begin{center}
\centering
\includegraphics[width=0.6\linewidth]{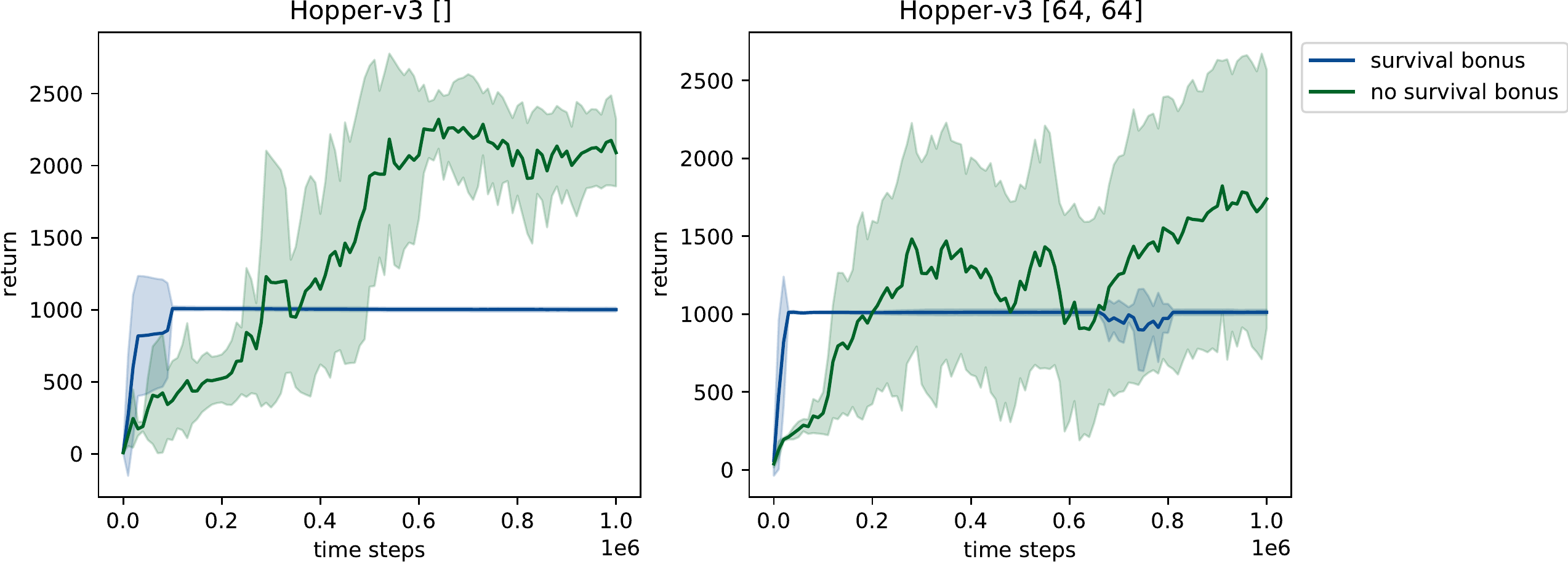}
\end{center}
\caption{Performance of PSSVF with and without the survival bonus for the reward in Hopper-v3 when using the hyperparameters maximizing the average return. Learning curves are averaged over 5 seeds.}

\label{fig:ablation_pvf}
\end{figure}
\paragraph{DDPG}
For DDPG, we used the Spinning Up implementation provided by OpenAI~\citep{SpinningUp2018}, which includes target networks for the actor and the critic and no learning for a fixed set of time steps, called starting steps. We did not include target networks and starting steps in our PBVFs, although they could potentially help stabilizing training. The implementation of DDPG that we used~\citep{SpinningUp2018} does not use observation normalization. In preliminary experiments we observed that it failed to significantly increase or decrease performance, hence we did not use it. Another difference between our TD algorithms and DDPG consists in the number of updates of the actor and the critic. Since DDPG's critic needs to keep track of the current policy, the critic and the actor are updated in a nested form, with the first's update depending on the latter and vice versa. Our PSVF and PAVF do not need to track the policy learned, hence, when it is time to update, we need only to train once the critic for many gradient steps and then train the actor for many gradient steps. This requires less compute. On the other hand, when using nonlinear policies, our PBVFs suffer the curse of dimensionality. For this reason, we profited from using a bigger critic. In preliminary experiments, we observed that DDPG’s performance did not change significantly through a bigger critic. We show differences in performance for our methods when removing observation normalization and when using a smaller critic (MLP(256,256)) in figure~\ref{fig:ablation_psvf}. We observe that the performance is decreasing if observation normalization is removed. However, only for shallow policies in Swimmer and deep policies in Hopper there seems to be a significant benefit. Future work will assess when bigger critics help.
\begin{figure}[h]
\begin{subfigure}[c]{1.0\textwidth}
  \centering
  \includegraphics[width=0.9\linewidth]{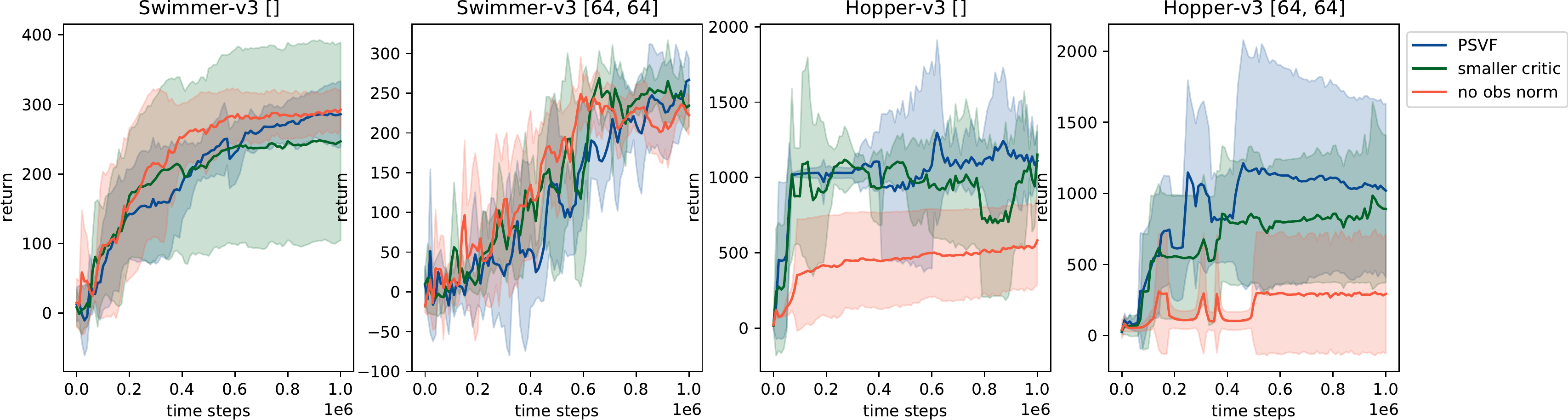}
  \vspace{0.3cm}
    \centering
\end{subfigure}%
\hfill
\begin{subfigure}[c]{1.0\textwidth}
	\vspace{0.2cm}
  \centering
  \includegraphics[width=0.9\linewidth]{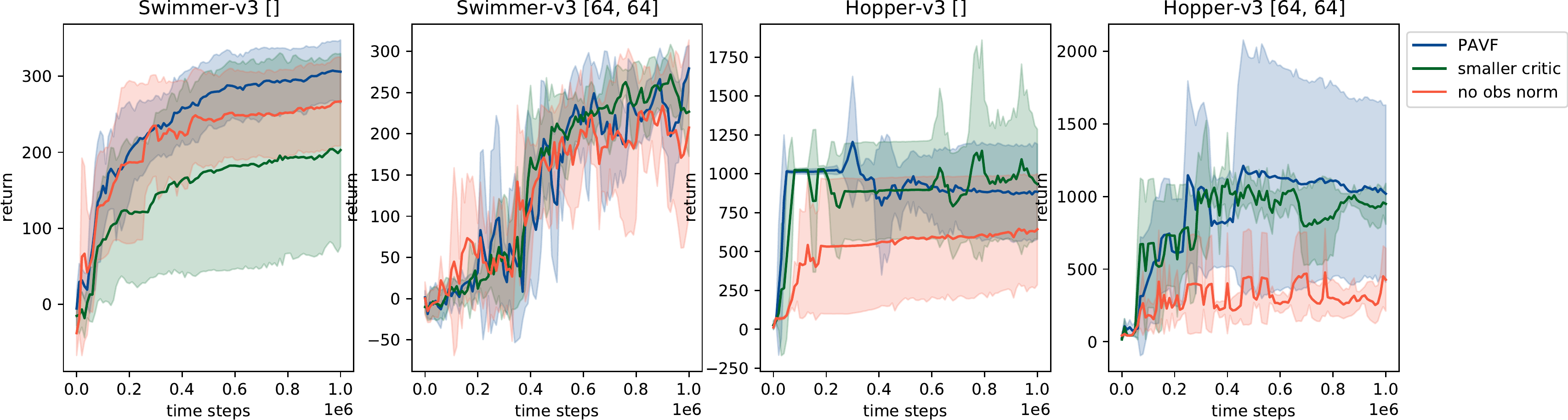}
    \centering
\end{subfigure}%
\caption{Learning curves for PSVF and PAVF for different environments and policies removing observation normalization and using a smaller critic. We use the hyperparameters maximizing the average return. Learning curves are averaged over 5 seeds. For this ablation we use deterministic policies.}
\label{fig:ablation_psvf}
\end{figure}

\paragraph{Discounting in Swimmer} For TD algorithms, we chose a fixed discount factor $\gamma=0.99$ for all environments but Swimmer-v3. This environment is known to be challenging for TD based algorithms because discounting causes the agents to become too short-sighted. We observed that, with the standard discounting, DDPG, PSVF and PAVF were not able to learn the task. However, making the algorithms more far-sighted greatly improved their performance. In figure~\ref{fig:ablation_discount} we report the return obtained by DDPG, PSVF and PAVF for different values of the discount factor in Swimmer when using deterministic policies.
\begin{figure}[h]
\begin{center}
\centering
\includegraphics[width=0.9\linewidth]{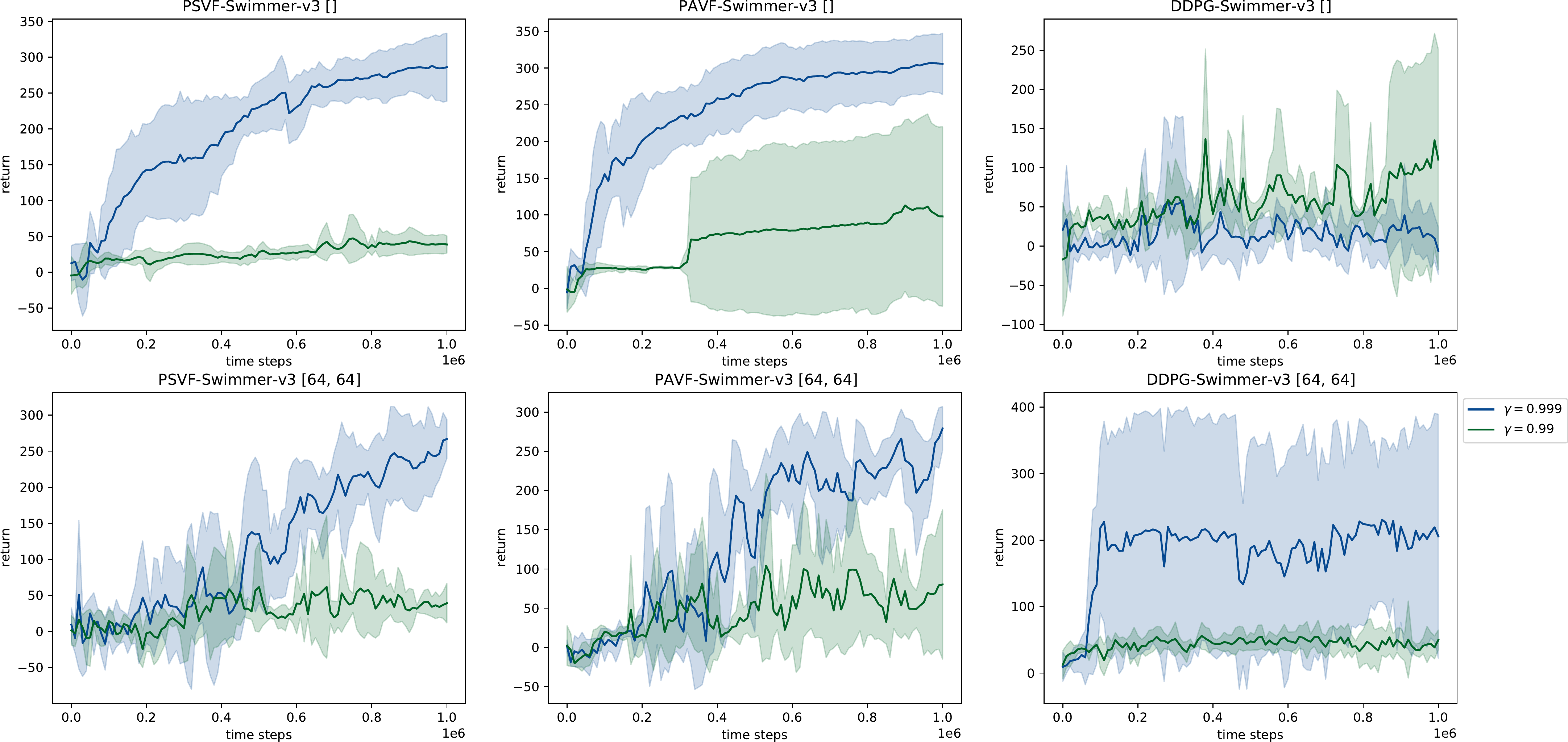}
\end{center}
\caption{Effect of different choices of the discount factor in Swimmer-v3 for PSVF, PAVF and DDPG, with shallow and deep deterministic policies. We use the hyperparameters maximizing the average return. Learning curves are averaged over 5 seeds}
\label{fig:ablation_discount}
\end{figure}

\subsubsection{Pseudocode}

\begin{algorithm}[H]
  \caption{Actor-critic with TD prediction for $V(s, \theta)$}
  \label{alg:psvf}
    \hspace*{\algorithmicindent} \textbf{Input}: Differentiable critic $V_{\textbf{w}}: \mathcal{S} \times \Theta \rightarrow \mathcal{R}$ with parameters $\textbf{w}$; deterministic or stochastic actor $\pi_{\theta}$ with parameters $\theta$; empty replay buffer $D$ \\
    \hspace*{\algorithmicindent} \textbf{Output} : Learned $V_{\textbf{w}} \approx V(s, \theta)$, learned $\pi_{\theta} \approx \pi_{\theta^*}$
    \begin{algorithmic}
    \State Initialize critic and actor weights $\textbf{w}, \theta$
	\Repeat:
	    \State Observe state s, take action $a=\pi_{\theta}(s)$, observe reward $r$ and next state $s'$
	    \State Store $(s, \theta, r, s')$ in the replay buffer $D$
        \If {it's time to update}:	
    		\For {many steps}:
    		    \State Sample a batch $B_1 = \{(s, \Tilde{\theta}, r, s')\}$ from $D$
    		    \State Update critic by stochastic gradient descent:
    		    \State $\nabla_{\textbf{w}} \frac{1}{|B_1|} \ev_{(s,  \Tilde{\theta}, r, s') \in B_1} [V_{\textbf{w}}(s, \Tilde{\theta}) - (r + \gamma V_{\textbf{w}}(s', \Tilde{\theta}))]^2$
    		\EndFor
    		\For {many steps}:
                \State Sample a batch $B_2 = \{(s)\}$ from $D$
                \State Update actor by stochastic gradient ascent:  $\nabla_{\theta} \frac{1}{|B_2|} \ev_{s \in B_2}[ V_{\textbf{w}}(s, \theta)]$
	        \EndFor
	   \EndIf
    \Until{convergence}
    \end{algorithmic}
\end{algorithm}

\begin{algorithm}[H]
  \caption{Stochastic actor-critic with TD prediction for $Q(s,a,\theta)$}
  \label{alg:spavf}
    \hspace*{\algorithmicindent} \textbf{Input}: Differentiable critic $Q_{\textbf{w}}: \mathcal{S} \times \mathcal{A} \times \Theta \rightarrow \mathcal{R}$ with parameters $\textbf{w}$; stochastic differentiable actor $\pi_{\theta}$ with parameters $\theta$; empty replay buffer $D$ \\
    \hspace*{\algorithmicindent} \textbf{Output} : Learned $Q_{\textbf{w}} \approx Q(s, a, \theta)$, learned $\pi_{\theta} \approx \pi_{\theta^*}$
    \begin{algorithmic}
    \State Initialize critic and actor weights $\textbf{w}, \theta$
	\Repeat:
	    \State Observe state s, take action $a=\pi_{\theta}(s)$, observe reward $r$ and next state $s'$
	    \State Store $(s, a, \theta, r, s')$ in the replay buffer $D$
        \If {it's time to update}:	
    		\For {many steps}:
    		    \State Sample a batch $B_1 = \{(s, a, \Tilde{\theta}, r, s')\}$ from $D$
    		    \State Update critic by stochastic gradient descent:
    		    \State $\nabla_{\textbf{w}} \frac{1}{|B_1|} \ev_{(s, a, \Tilde{\theta}, r, s') \in B_1} [Q_{\textbf{w}}(s, a, \Tilde{\theta}) - (r + \gamma Q_{\textbf{w}}(s',a' \sim \pi_{\Tilde{\theta}}(s'),\Tilde{\theta}))]^2$

    		\EndFor
    		\For {many steps}:
    		
    		    \State Sample a batch $B_2 = \{(s, a, \Tilde{\theta})\}$ from $D$

    		    \State Update actor by stochastic gradient ascent:
    		    \State $ \frac{1}{|B_2|} \ev_{(s,a, \Tilde{\theta}) \in B_2} \left[ \frac{\pi_{\theta}(a|s)}{\pi_{\Tilde{\theta}}(a|s)} \left(Q(s,a,\theta) \nabla_{\theta} \log\pi_{\theta}(a|s) +  \nabla_{\theta}Q(s,a,\theta)\right)\right]$
    		\EndFor
	    \EndIf
	\Until{convergence}
    \end{algorithmic}
  \end{algorithm}

\begin{algorithm}[H]
  \caption{Deterministic actor-critic with TD prediction for $Q(s,a,\theta)$}
  \label{alg:dpavf}
   \hspace*{\algorithmicindent} \textbf{Input}: Differentiable critic $Q_{\textbf{w}}: \mathcal{S} \times \mathcal{A} \times \Theta \rightarrow \mathcal{R}$ with parameters $\textbf{w}$; differentiable deterministic actor $\pi_{\theta}$ with parameters $\theta$; empty replay buffer $D$ \\
   \hspace*{\algorithmicindent} \textbf{Output} : Learned $Q_{\textbf{w}} \approx Q(s, a, \theta)$, learned $\pi_{\theta} \approx \pi_{\theta^*}$
    \begin{algorithmic}
    \State Initialize critic and actor weights $\textbf{w}, \theta$
	\Repeat:
	    \State Observe state s, take action $a=\pi_{\theta}(s)$, observe reward $r$ and next state $s'$
	    \State Store $(s, a, \theta, r, s')$ in the replay buffer $D$
        \If {it's time to update}:	
    		\For {many steps}:
    		    \State Sample a batch $B_1 = \{(s, a, \Tilde{\theta}, r, s')\}$ from $D$
    		    \State Update critic by stochastic gradient descent:
    		    \State $\nabla_{\textbf{w}} \frac{1}{|B_1|} \ev_{(s, a, \Tilde{\theta}, r, s') \in B_1} [Q_{\textbf{w}}(s, a, \Tilde{\theta}) - (r + \gamma Q_{\textbf{w}}(s',\pi_{\Tilde{\theta}}(s'),\Tilde{\theta}))]^2$

    		\EndFor
    		\For {many steps}:
    		    \State Sample a batch $B_2 = \{(s)\}$ from $D$

    		    \State Update actor by stochastic gradient ascent:
    		    \State $ \frac{1}{|B_2|} \ev_{s \in B_2}[  \nabla_{\theta} \pi_{\theta}(s) \nabla_a Q_{\textbf{w}}(s,a,\theta)|_{a=\pi_{\theta}(s)} + \nabla_{\theta} Q_{\textbf{w}}(s, a, \theta)|_{a=\pi_{\theta}(s)}]$    		
    		\EndFor    		
    		
	    \EndIf
	    \Until{convergence}
    \end{algorithmic}
  \end{algorithm}


\subsection{Experimental details}
\label{apx:exp}
\subsubsection{LQR}
\label{experiments_detail:lqr}
For our visualization experiment, we employ an instance of the Linear Quadratic Regulator. Here, the agent observes a 1-D state, corresponding to its position and chooses a 1-D action. The transitions are $s' = s + a$ and there is a quadratic negative term for the reward: $R(s,a) = -s^2 - a^2$. The agent starts in state $s_0=1$ and acts in the environment for 50 time steps. The state space is bounded in [-2,2]. The goal of the agent is to reach and remain in the origin. The agent is expected to perform small steps towards the origin when it uses the optimal policy. For this task, we use a deterministic policy without tanh nonlinearity and we do not use observation normalization. Below additional details and plots for different algorithms.

\paragraph{PSSVF} We use a learning rate of $1e-3$ for the policy and $1e-2$ for the PSSVF. Weights are perturbed every episode using $\sigma = 0.5$. The policy is initialized with weight $3.2$ and bias $-3.5$. All the other hyperparameters are set to their default. The true episodic $J(\theta)$ is computed by running 10,000 policies in the environment with parameters in $[-5,5] \times [-5,5]$. $V_w(\theta)$ is computed by measuring the output of the PSSVF on the same set of policies. Each red arrow in figure~\ref{fig:lqr_pvf} represents 200 update steps of the policy.

\paragraph{PSVF ad PAVF} Using the exact same setting, we run PSVF and PAVF in LQR environment and we compare learned $V(s_{0}, \theta)$ and $Q(s_0, \pi_{\theta}(s_0), \theta)$ with the true PSVF and PAVF over the parameter space. Computing the value of the true PSVF and PAVF requires computing the infinite sum of discounted reward obtained by the policy. Here we approximate it by running 10,000 policies in the environment with parameters in $[-5,-5] \times [-5,5]$ for 500 time steps. This, setting $\gamma = 0.99$, provides a good approximation of their true values, since further steps in the environment result in almost zero discounted reward from $s_0$. We use a learning rate of $1e-2$ for the policy and $1e-1$ for the PSVF and PAVF. Weights are perturbed every episode using $\sigma = 0.5$. The policy is updated every 10 time steps using 2 gradient steps; the PSVF and PAVF are updated every 10 time steps using 10 gradient updates. The critic is a 1-layer MLP with 64 neurons and tanh nonlinearity.\\
\newline
In Figures~\ref{fig:lqr_psvf} and~\ref{fig:lqr_pavf} we report $J(\theta)$, the cumulative discounted reward that an agent would obtain by acting in the environment for infinite time steps using policy $\pi_{\theta}$ and the cumulative return predicted by the PSVF and PAVF for two different times during learning. Like in the PSSVF experiment, the critic is able improve its predictions over the parameter space. Since in the plots $V(s,\theta)$ and $Q(s, \pi_{\theta}(s),\theta)$ are evaluated only in $s_0$, the results show that PBVFs are able to effectively bootstrap the values of future states. Each red arrow in Figures~\ref{fig:lqr_psvf} and~\ref{fig:lqr_pavf} represents 50 update steps of the policy.

\begin{figure}[h]
\begin{center}
\centering
Optimization after 15 episodes
\includegraphics[width=1.0\linewidth]{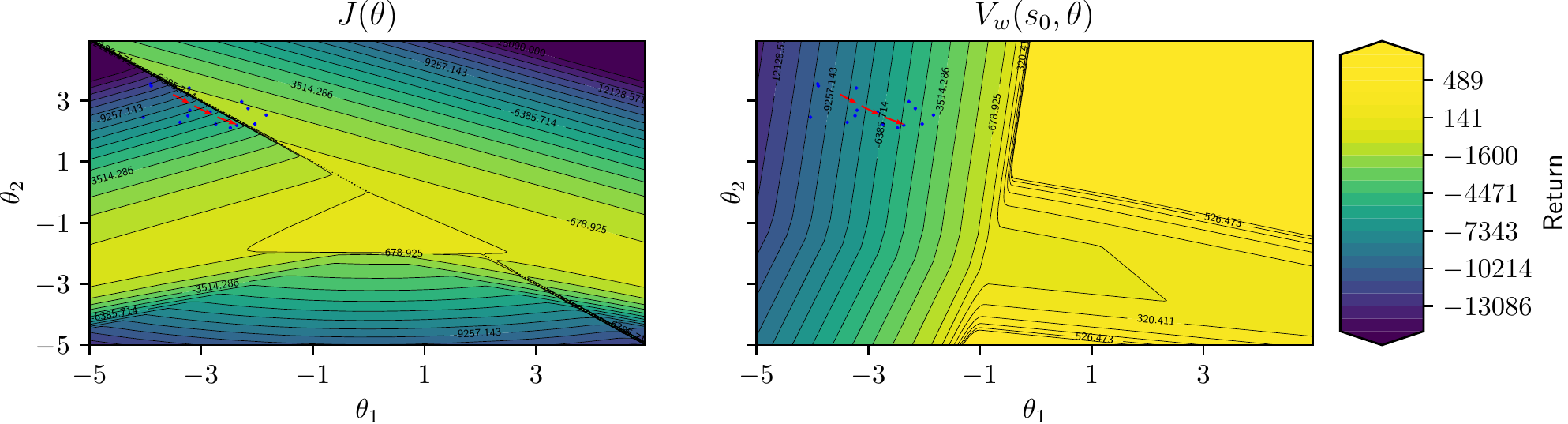} \\
\centering
Optimization after 100 episodes
\includegraphics[width=1.0\linewidth]{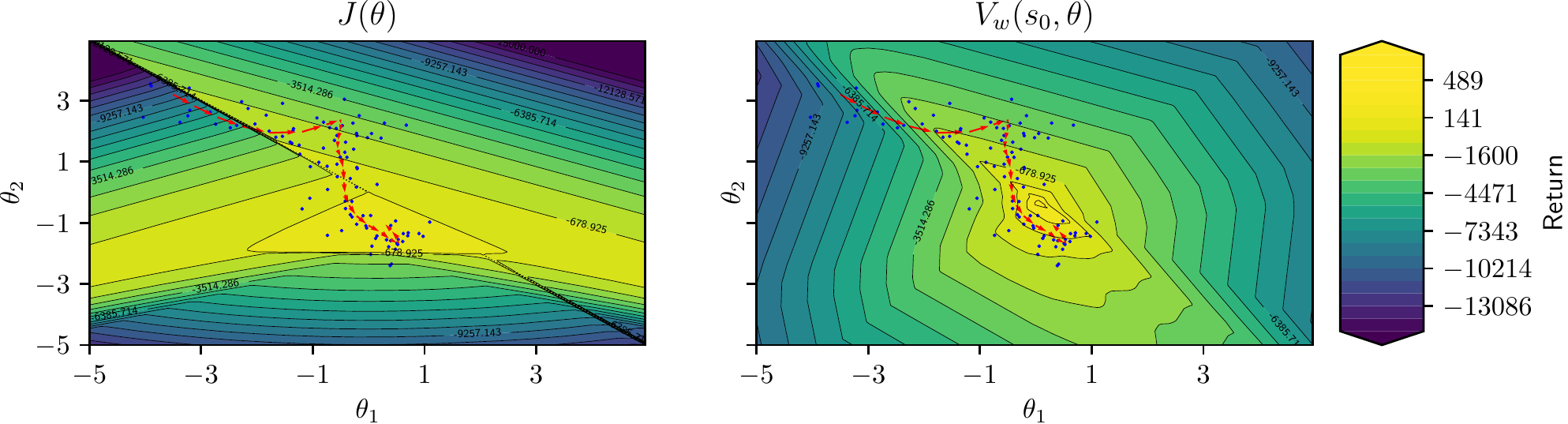}
\end{center}
\caption{True cumulative discounted reward $J(\theta)$ and PSVF estimation $V_w(s_0, \theta)$ as a function of the policy parameters at two different stages in training. The red arrows represent an optimization trajectory in  parameter space. The blue dots represent the perturbed policies used to train $V_w(s_0, \theta)$.}
\label{fig:lqr_psvf}
\end{figure}

\begin{figure}[h]
\begin{center}
\centering
Optimization after 15 episodes
\includegraphics[width=1.0\linewidth]{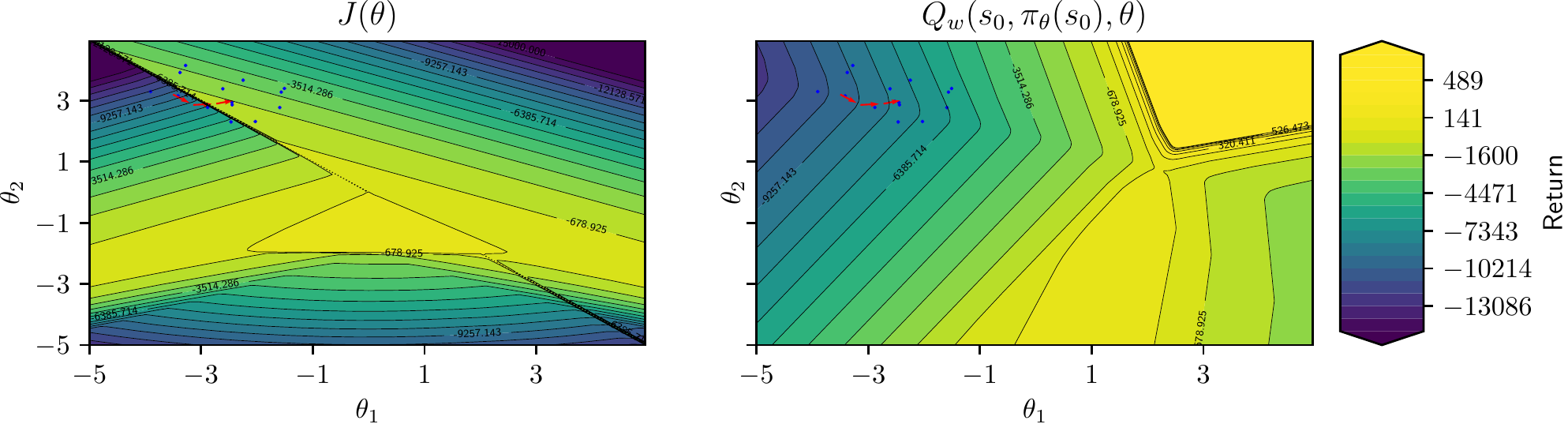} \\
\centering
Optimization after 100 episodes
\includegraphics[width=1.0\linewidth]{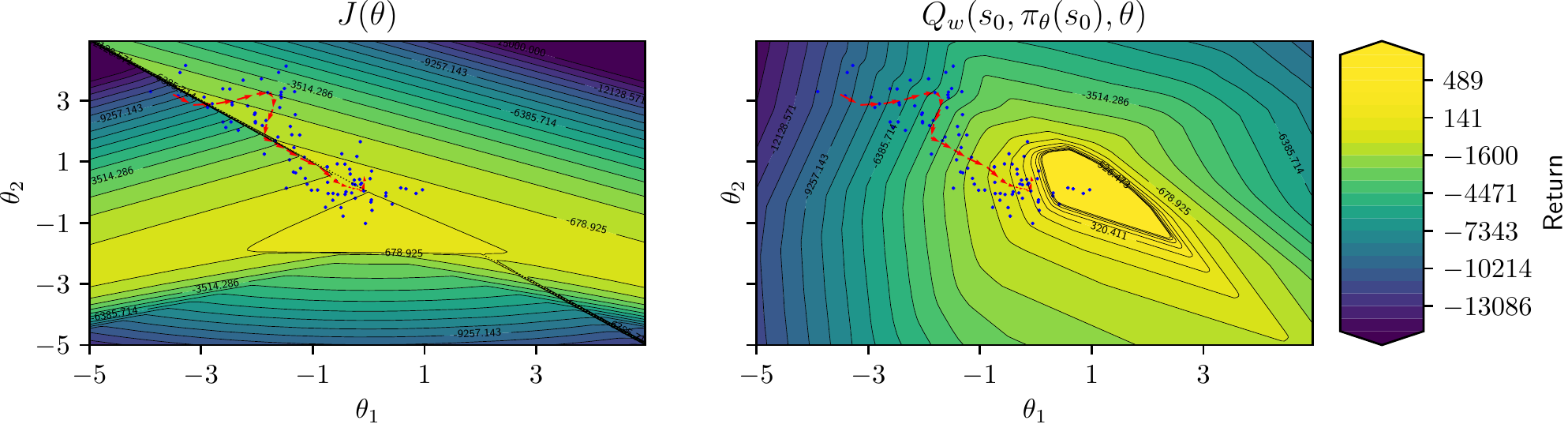}
\end{center}
\caption{True cumulative discounted reward $J(\theta)$ and PAVF estimation $Q_w(s_0, \pi_{\theta}(s_0),\theta)$ as a function of the policy parameters at two different stages in training. The red arrows represent an optimization trajectory in  parameter space. The blue dots represent the perturbed policies used to train $Q_w(s_0, \pi_{\theta}(s_0),\theta)$.}
\label{fig:lqr_pavf}
\end{figure}

\subsubsection{Offline experiments}
\label{experiments_detail:zero_shot}

\paragraph{Zero-shot learning}
We evaluate the performance of the policies learned from scratch evaluating them with 5 test trajectories every 5 gradient steps. In addition to the results in the main paper, we report in Figures~\ref{fig:zero_shot_linear} and~\ref{fig:zero_shot_deep} a comparison of zero-shot performance between PSSVF, PSVF and PAVF in three different environments using deterministic shallow and deep policies (2-layers MLP(64,64)). In this task we use the same hyperparameters found in tables~\ref{tab:hyp_pvf},~\ref{tab:hyp_psvf} and~\ref{tab:hyp_pavf}. One additional hyperparameter needs to be considered: the learning rate of the policies trained from scratch. In Figure~\ref{fig:scratch} of the main paper, we use a tuned learning rate of 0.02 that we found working particularly well for PSSVF in the Swimmer environment. In the additional experiments in Figures~\ref{fig:zero_shot_linear} and~\ref{fig:zero_shot_deep}, we use a learning rate of $0.05$ that we found working well across all policies, environments and algorithms when learning zero-shot.\\
\newline
We observe that, using shallow policies, PBVFs can effectively zero-shot learn policies with performance comparable to the policy learned in the environment without additional tuning for the learning rate. We note the regular presence of a spike in performance followed by a decline due to the policy going to regions of the parameter space never observed. This suggests that there is a trade-off between exploiting the generalization of the critic and remaining in the part of the parameter space where the critic is accurate. Measuring the width of these spikes can be useful for determining the number of offline gradient steps to perform in the general algorithm. When using deep policies the results become much worse and zero-shot learned policies can recover the performance of the main policy being learned only in simple environments and at beginning of training (eg. MountainCarContinuous). We observe that, when the critic is trained (last column), the replay buffer contains policies that are very distant to policies randomly initialized. This might explain why the zero-shot performance is better sometimes at the beginning of training (eg. second column). However, since PBVFs in practice perform mostly local off-policy evaluation around the learned policy, this problem is less prone to arise in our main experiments.

\begin{figure}[h]
\footnotesize
\begin{subfigure}[l]{.33\textwidth}
\vspace{0.0cm}
  \centering
  \includegraphics[scale=0.30]{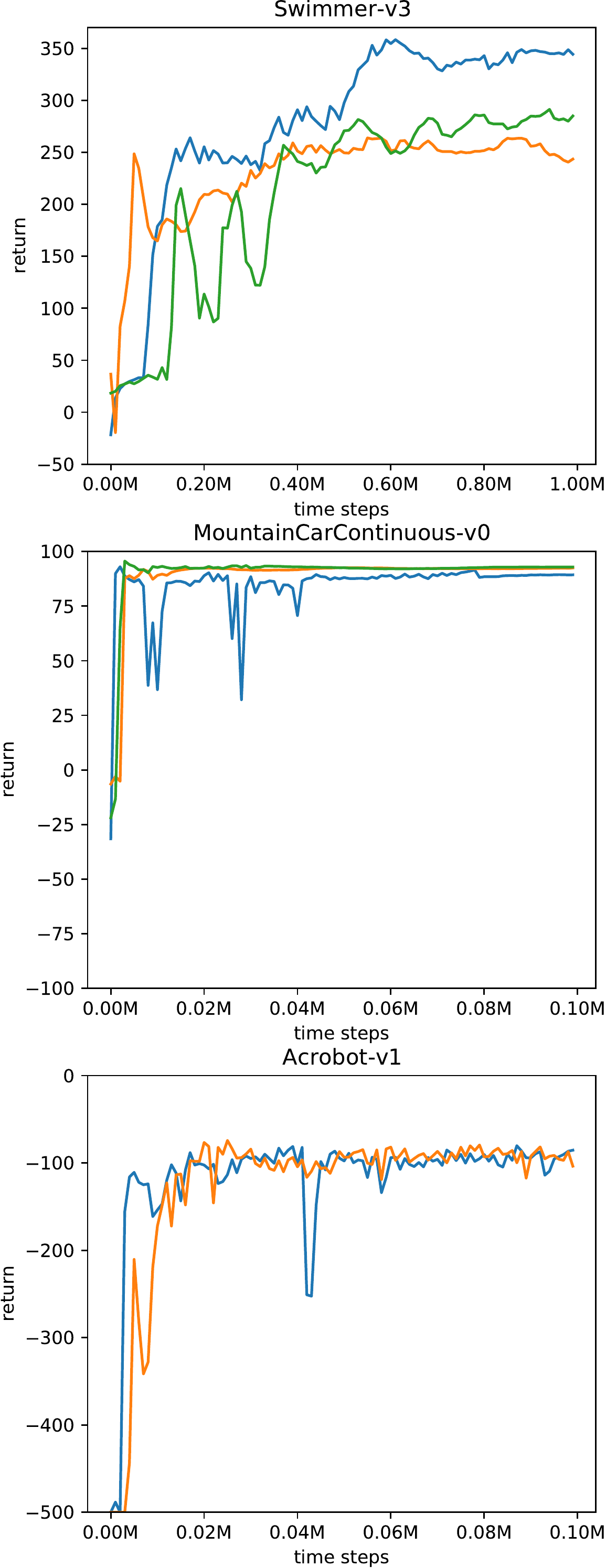}
\end{subfigure}
\begin{subfigure}[l]{.1\textwidth}
	\vspace{0.0cm}
  \includegraphics[scale=0.24]{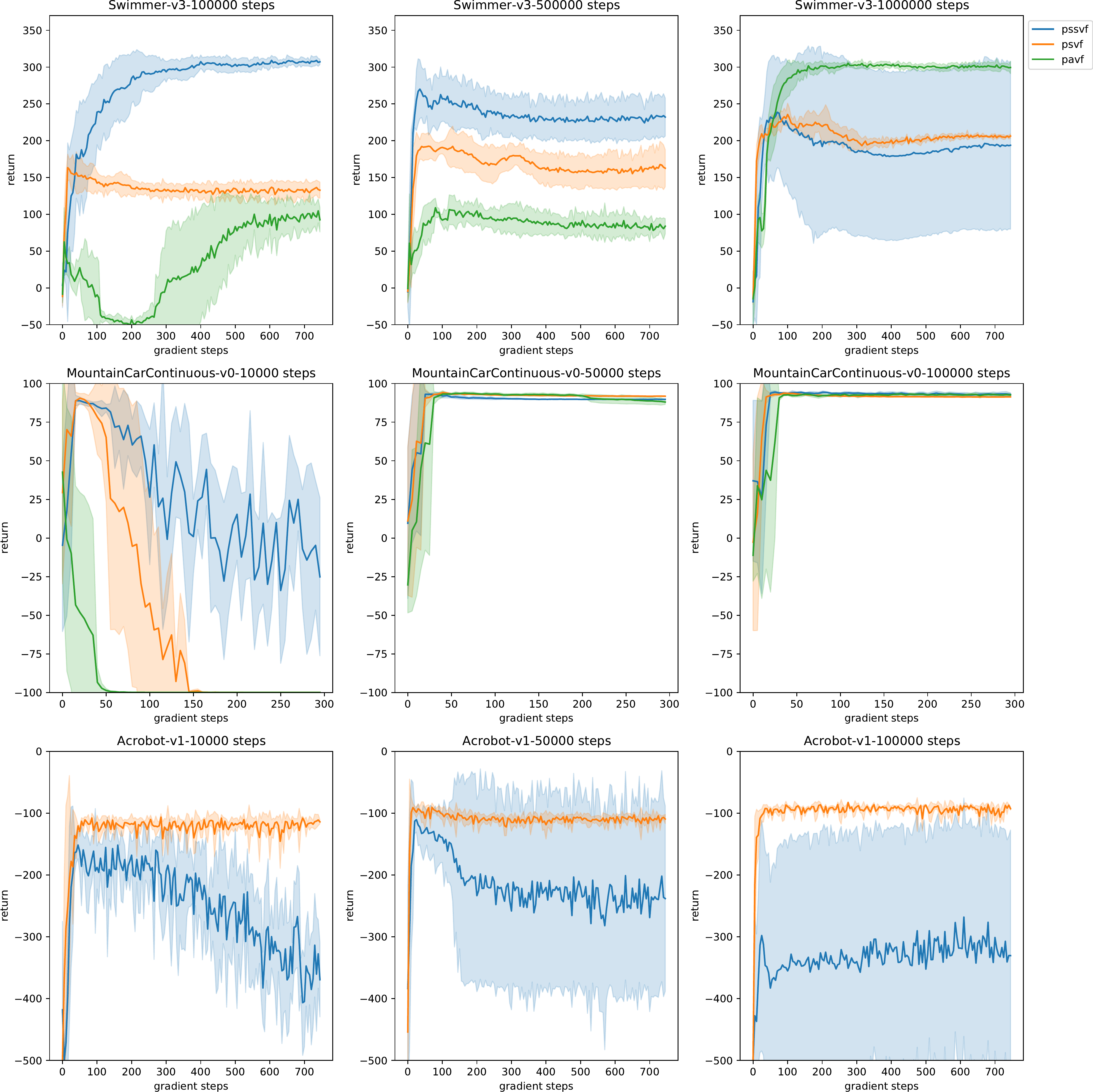}
\end{subfigure}%
\caption{Shallow policies learned from scratch during training. The plots in the left column represent the return of agents learning while interacting with the environment using different algorithms. The learning curves in the other plots represent the return obtained by policies trained from scratch following the fixed critics after different time steps of training. Zero-shot learning curves are averaged over 5 seeds.}
\label{fig:zero_shot_linear}
\end{figure}

\begin{figure}[h]
\footnotesize
\begin{subfigure}[l]{.33\textwidth}
\vspace{0.0cm}
  \centering
  \includegraphics[scale=0.30]{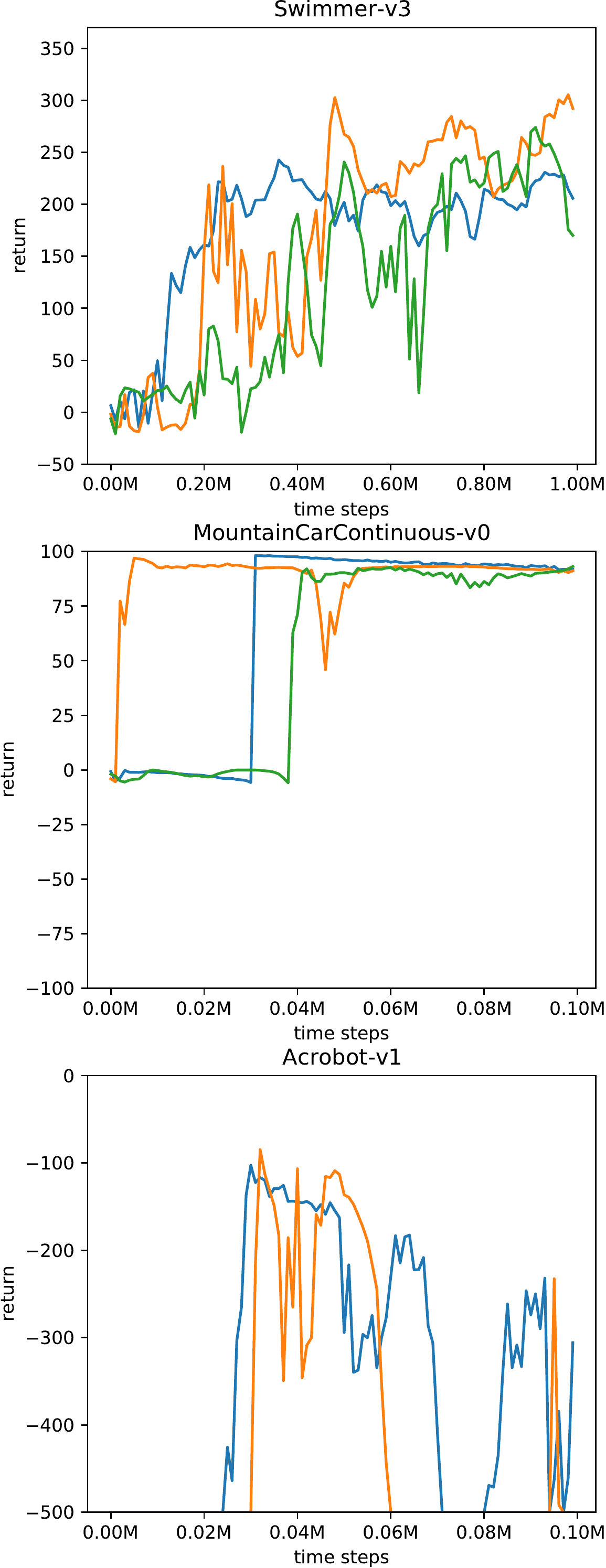}
\end{subfigure}
\begin{subfigure}[l]{.1\textwidth}
	\vspace{0.0cm}
  \includegraphics[scale=0.24]{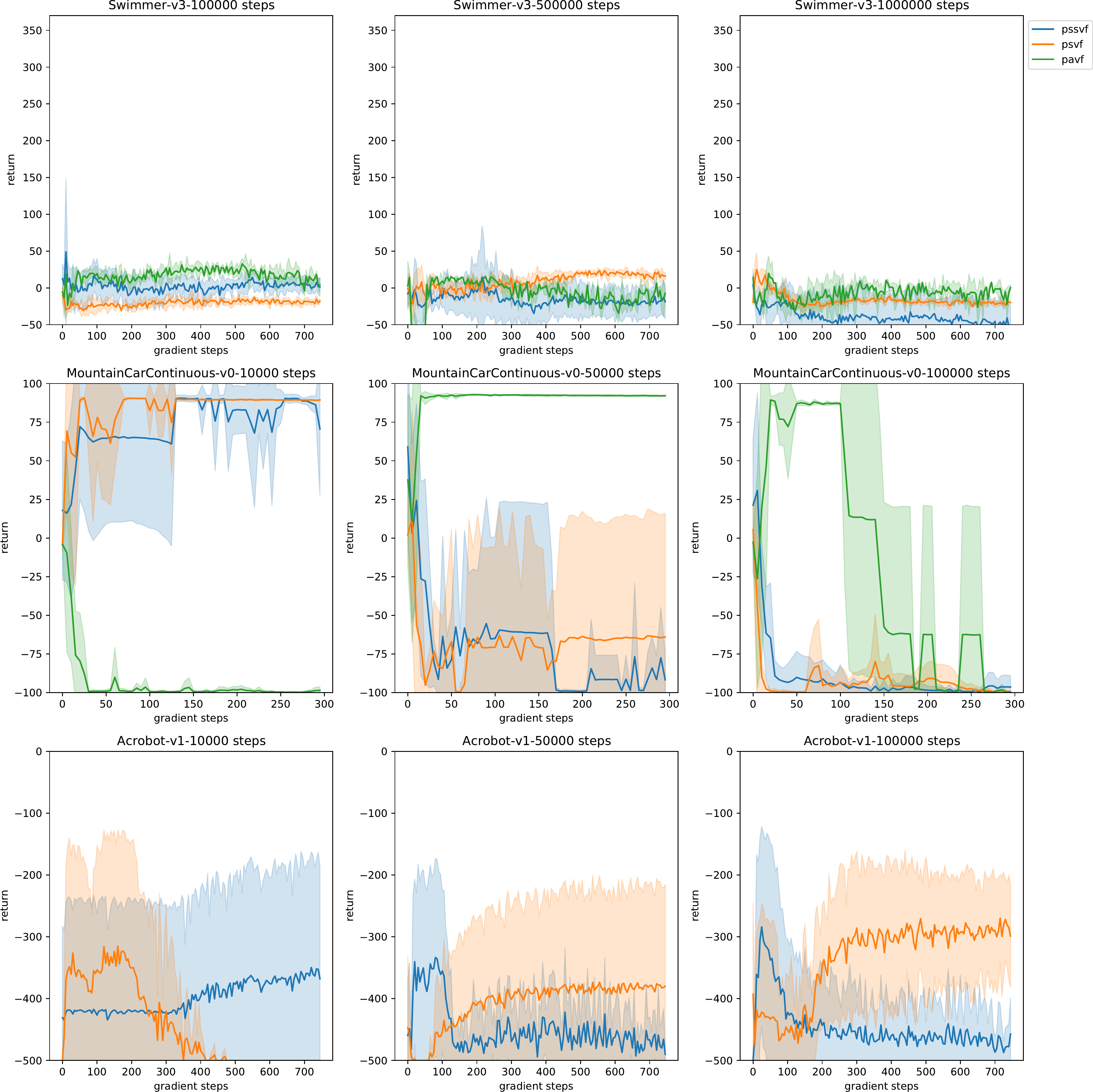}
\end{subfigure}%
\caption{Deep policies learned from scratch during training. The plots in the left column represent the return of agents learning while interacting with the environment using different algorithms. The learning curves in the other plots represent the return obtained by policies trained from scratch following the fixed critics after different time steps of training. Zero-shot learning curves are averaged over 5 seeds.}
\label{fig:zero_shot_deep}
\end{figure}
%
\paragraph{Offline learning with fragmented behaviors}
In this task, data are generated by perturbing a randomly initialized deterministic policy every 200 time steps and using it to act in the environment. We use $\sigma=0.5$ for the perturbations. After the dataset is collected, the PSVF is trained using a learning rate of $1e-3$ with a batch size of $128$. When the policy is learned, we use a learning rate of 0.02. All other hyperparameters are set to default values.

\subsubsection{Full experimental results}
\paragraph{Methodology}
In order to ensure a fair comparison of our methods and the baselines, we adopt the following procedure. For each hyperparameter configuration, for each environment and policy architecture, we run 5 instances of the learning algorithm using different seeds. We measure the learning progress by running 100 evaluations while learning the deterministic policy (without action or parameter noise) using 10 test trajectories. We use two metrics to determine the best hyperparameters: the average return over policy evaluations during the whole training process and the average return over policy evaluations during the last 20\% time steps. For each algorithm, environment and policy architecture, we choose the two hyperparameter configurations maximizing the performance of the two metrics and test them on 20 new seeds, reporting average and final performance in table~\ref{tab:avg_results} and ~\ref{tab:final_results} respectively.

\begin{table}[!h]
  \caption{Average return with standard deviation (across 20 seeds) for hypermarameters optimizing the average return during training using deterministic policies. Square brackets represent the number of neurons per layer of the policy. [] represents a linear policy.}
    \begin{tabular}{lccccc}
    \hline
    \toprule
    \textbf{Policy: []}  & MountainCar & Inverted & Reacher & Swimmer & Hopper\\
    & Continuous-v0 & Pendulum-v2 & -v2 & -v3 & -v3 \\
    \midrule
    ARS   &$ 63\pm 6$ &$ 886\pm 72$ &$ -9.2\pm 0.3$&$ 228\pm 89$&$ 1184\pm 345$\\
    PSSVF &$ 85\pm 4$ &$ 944\pm 33$ &$ -11.7\pm0.9 $&$ 259\pm 47$&$ 1392\pm 287$\\
    DDPG  &$ 0\pm0 $ &$ 612\pm169 $ &$ -8.6\pm 0.9 $&$ 95\pm112 $&$ 629\pm 145$\\
    PSVF  &$ 84\pm 20$ &$ 926\pm 34$ &$ -19.7\pm 6.0$&$ 188\pm 71$&$ 917\pm 249$\\
    PAVF  &$ 82\pm 21$ &$ 913\pm 40$ &$ -17.0\pm 7.7 $&$ 231\pm 56$&$ 814\pm 223$\\
    \toprule
    \textbf{Policy:[32]}   &  &  &  &  & \\
    ARS   &$ 37\pm 11$ &$ 851\pm 46$ &$ -9.6\pm 0.3$&$ 139\pm78 $&$ 1003\pm 66$\\
    PSSVF &$ 60\pm 33$ &$ 701\pm 138$ &$ 10.4\pm 0.5 $&$ 189\pm 35$&$ 707\pm 668$\\
    DDPG  &$ 0\pm0 $ &$ 816\pm 36$ &$ -5.7\pm 0.3 $&$  61\pm 32 $&$ 1384\pm 125$\\
    PSVF  &$ 71\pm 25$ &$ 529\pm281 $ &$ -11.9\pm 1.2$&$ 226\pm33 $&$ 864\pm 272$\\
    PAVF  &$ 71\pm 27$ &$ 563\pm 228$ &$ -10.9\pm1.1 $&$ 222\pm 28$&$ 793\pm 322$\\
    \toprule
    \textbf{Policy: [64,64]}   &  &  &  &  & \\
    \midrule
    ARS   &$ 28\pm 8$ &$ 812\pm 239$ &$ -9.8\pm 0.3$&$ 129\pm68 $&$ 964\pm 47$\\
    PSSVF &$ 72\pm 22$ &$ 850\pm 93$ &$ -10.7\pm 0.2$&$ 158\pm 59$&$ 922\pm 568$\\
    DDPG  &$ 0\pm0 $ &$ 834\pm 36$ &$-5.5\pm 0.4$&$ 92\pm117 $&$ 767\pm 627$\\
    PSVF  &$ 80\pm 9$ &$ 580\pm 107$ &$ -10.7\pm 0.6$&$ 137\pm 38$&$ 843\pm 282$\\
    PAVF  &$ 73\pm 10$ &$ 399\pm 219$ &$ -10.7\pm 0.5$&$ 142\pm26 $&$ 875\pm 136$\\
    \hline
  \end{tabular}
  \begin{tabular}{lcc}
    \hline
    \toprule
    \textbf{Policy: []}  & Acrobot-v1 & CartPole-v1 \\
    \midrule
    ARS   &$ -161\pm23 $ &$ 476\pm 13$\\
    PSSVF &$ -137\pm 14$ &$ 443\pm 105$\\
    PSVF  &$ -148\pm 25$ &$ 459\pm 28$\\
    \toprule
    \textbf{Policy:[32]}    &  &  \\
    \midrule
    ARS   &$ -296\pm38 $ &$ 395\pm 141$\\
    PSSVF &$ -251\pm 80$ &$ 463\pm 18$\\
    PSVF  &$ -270\pm 113$ &$ 413\pm 61$\\
    \toprule
    \textbf{Policy: [64,64]}   &  &  \\
    \midrule
    ARS   &$-335 \pm 35$ &$ 416\pm 105$\\
    PSSVF &$ -281\pm 117$ &$ 452\pm 34$\\
    PSVF  &$ -397\pm 71$ &$ 394\pm 71$\\
    \hline
  \end{tabular}
    \label{tab:avg_results}

\end{table}    

\begin{table}[!h]
  \caption{Final return with standard deviation (across 20 seeds) for hypermarameters optimizing the final return during training using deterministic policies.}
    \begin{tabular}{lccccc}
    \hline
    \toprule
    \textbf{Policy: []}  & MountainCar & Inverted & Reacher & Swimmer & Hopper\\
    & Continuous-v0 & Pendulum-v2 & -v2 & -v3 & -v3 \\
    \midrule
    ARS   &$ 73\pm 5$ &$ 657\pm 477$ &$ -8.6\pm 0.5$&$ 334\pm34 $&$ 1443\pm 713$\\
    PSSVF &$ 84\pm28 $ &$ 970\pm126 $ &$ -10.0\pm1.0 $&$ 350\pm8 $&$ 1560\pm 911$\\
    DDPG  &$ 0\pm1 $ &$ 777\pm 320$ &$ -7.3\pm0.4 $&$ 146\pm152 $&$ 704\pm234 $\\
    PSVF  &$ 76\pm36 $ &$ 906\pm289 $ &$ -16.5\pm1.6 $&$ 238\pm107 $&$ 1067\pm 340 $\\
    PAVF  &$ 68\pm42 $ &$ 950\pm223 $ &$ -17.2\pm 15.4 $&$ 298\pm40 $&$ 720\pm 281 $\\
    \toprule
    \textbf{Policy:[32]}   &  &  &  &  & \\
    ARS   &$ 54\pm 20$ &$ 936\pm 146$ &$ -9.2\pm 0.4$&$ 239\pm 117$&$ 1048\pm 68$\\
    PSSVF &$ 89\pm22 $ &$ 816\pm 234$ &$ -10.2\pm1.0 $&$ 294\pm41 $&$ 1204\pm 615 $\\
    DDPG  &$ 0\pm0 $ &$ 703\pm283 $ &$ -4.6\pm0.6 $&$ 179\pm 150$&$ 1290\pm 348 $\\
    PSVF  &$ 84\pm31 $ &$ 493\pm462 $ &$ -11.3\pm0.8 $&$ 290\pm70 $&$ 1003\pm 572 $\\
    PAVF  &$ 92\pm 7$ &$ 854\pm 295$ &$ -10.1\pm 0.9$&$ 307\pm 34 $&$ 967\pm411 $\\
    \toprule
    \textbf{Policy: [64,64]}   &  &  &  &  & \\
    \midrule
    ARS   &$ 11\pm 30$ &$ 976\pm 83$ &$ -9.4\pm0.4 $&$ 157\pm 54$&$ 1006\pm 47$\\
    PSSVF &$ 91\pm 16$ &$ 898\pm 227$ &$ -10.7\pm0.6 $&$ 224\pm 99 $&$ 1412\pm691 $\\
    DDPG  &$ 0\pm0 $ &$ 943\pm 73$ &$-4.4 \pm0.4 $&$ 196\pm 151$&$ 1437\pm 752$\\
    PSVF  &$ 93\pm1 $ &$ 1000\pm0 $ &$ -10.6\pm1.0 $&$ 257\pm26 $&$ 1247\pm 344 $\\
    PAVF  &$ 93\pm2 $ &$ 827\pm 267$ &$ -10.6\pm0.4 $&$ 232\pm42 $&$ 1005\pm 155 $\\
    \hline
  \end{tabular}
  \begin{tabular}{lcc}
    \hline
    \toprule
    \textbf{Policy: []}  & Acrobot-v1 & CartPole-v1 \\
    \midrule
    ARS   &$ -126\pm 26$ &$ 499\pm 2$\\
    PSSVF &$ -97\pm 6 $ &$ 482\pm 53$\\
    PSVF  &$ -100\pm 18 $ &$ 500\pm 0$\\
    \toprule
    \textbf{Policy:[32]}    &  &  \\
    \midrule
    ARS   &$ -215\pm 97$ &$ 471\pm 110$\\
    PSSVF &$ -116\pm 33$ &$ 500\pm0 $\\
    PSVF  &$ -244\pm 151$ &$ 488\pm 36$\\
    \toprule
    \textbf{Policy: [64,64]}   &  &  \\
    \midrule
    ARS   &$ -182\pm 45$ &$ 492\pm 18$\\
    PSSVF &$ -233\pm 139 $ &$ 500\pm 0$\\
    PSVF  &$ -406\pm 51$ &$ 499\pm 2$\\
    \hline
  \end{tabular}
    \label{tab:final_results}

\end{table}

Figures~\ref{fig:learning_curve_app1} and ~\ref{fig:learning_curve_app2} report all the learning curves from the main paper and for a small non linear policy with 32 hidden neurons.
\begin{figure}[h]
\begin{center}
Deterministic policies
\includegraphics[width=0.9\linewidth]{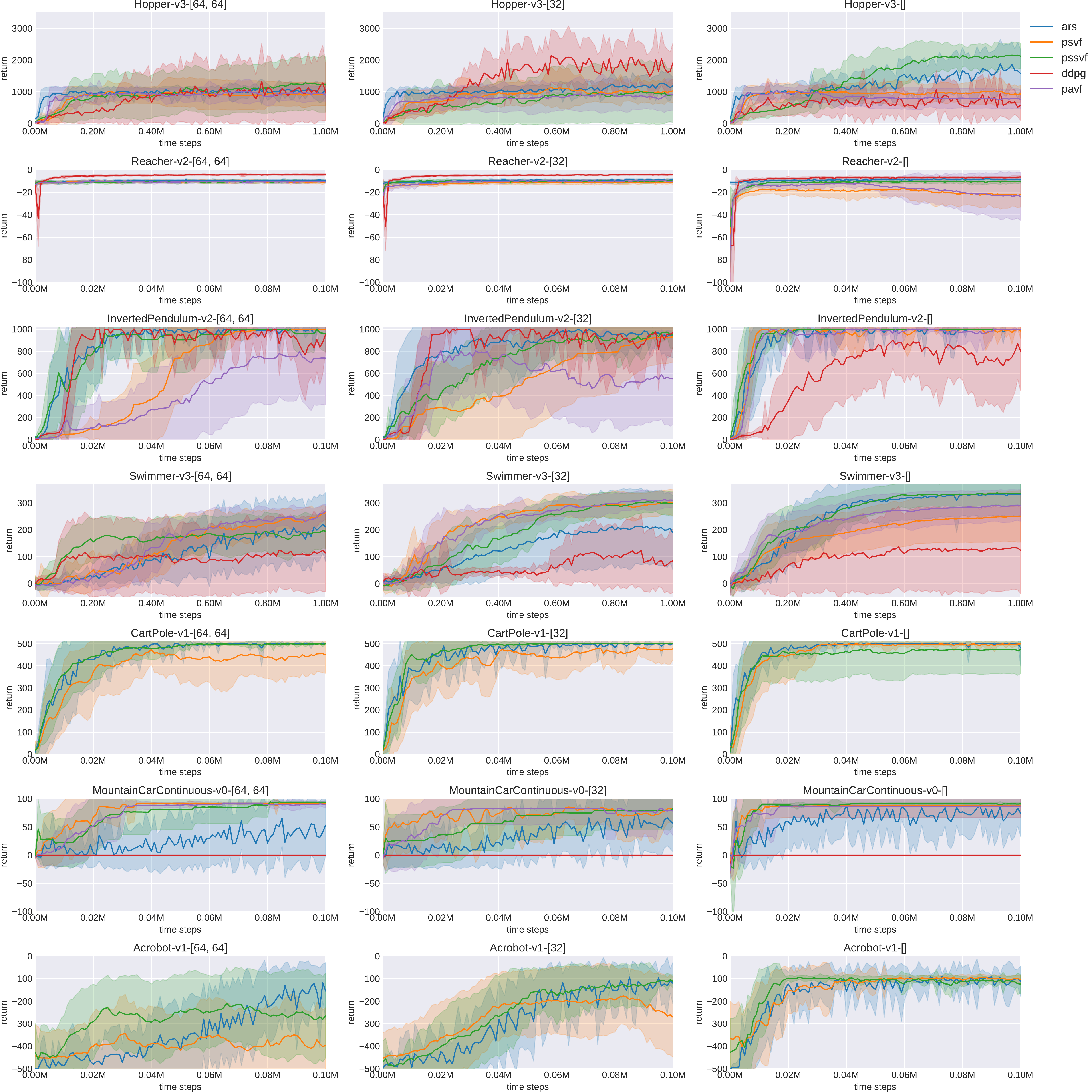}
\end{center}

\caption{Learning curves representing the average return as a function of the number of time steps in the environment (across 20 runs) with different environments and deterministic policy architectures.
We use the \textbf{best hyperparameters found while maximizing the average reward} for each task. For each subplot, the square brackets represent the number of neurons per policy layer. [] represents a linear policy.}
\label{fig:learning_curve_app1}
\end{figure}

\begin{figure}[h]
\begin{center}
Deterministic policies
\includegraphics[width=0.9\linewidth]{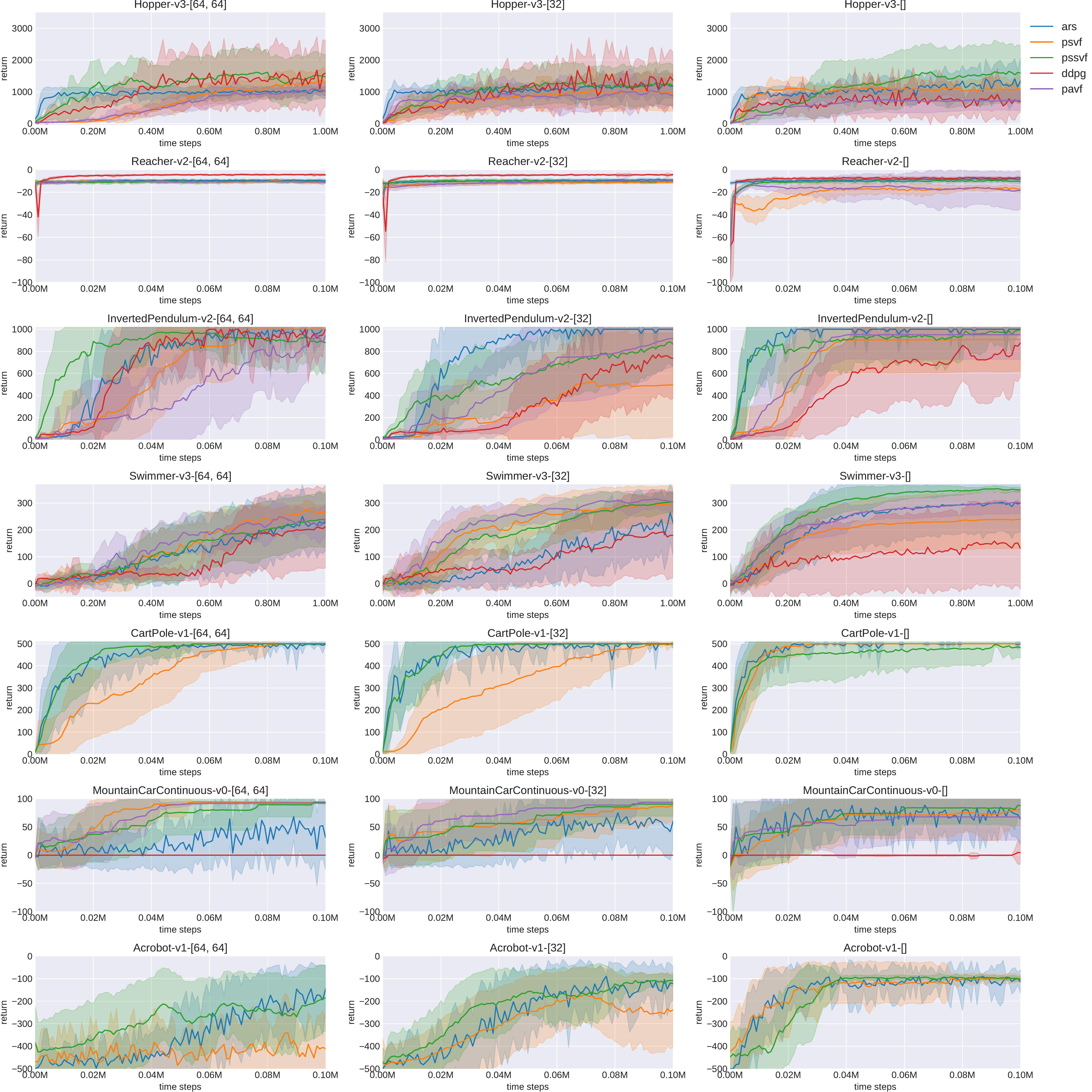}
\end{center}
\caption{Learning curves representing the average return as a function of the number of time steps in the environment (across 20 runs) with different environments and deterministic policy architectures.
We use the \textbf{best hyperparameters found while maximizing the final reward for each task}. For each subplot, the square brackets represent the number of neurons per policy layer. [] represents a linear policy.}
\label{fig:learning_curve_app2}

\end{figure}

\paragraph{Stochastic policies}
We include some results for stochastic policies when using PSSVF and PSVF. Figures~\ref{fig:learning_curve_lin_stoch} and~\ref{fig:learning_curve_deep_stoch} show a comparison with the baselines when using shallow and deep policies respectively. We observe results sometimes comparable, but often inferior with respect to deterministic policies. In particular, when using shallow policies, PBVFs are able to outperform the baselines in the MountainCar environment, while obtaining comparable performance in CartPole and InvertedPendulum. Like in previous experiments, PBVFs fail to learn a good policy in Reacher. When using deep policies, the results are slightly different: PBVFs outperform ARS and DDPG in Swimmer, but fail to learn InvertedPendulum. Although the use of stochastic policies can help smoothing the objective function and allows the agent exploring in action space, we believe that the lower variance provided by deterministic policies can facilitate learning PBVFs. 

\begin{figure}[h]
\begin{center}
Stochastic shallow policies
\includegraphics[width=0.9\linewidth]{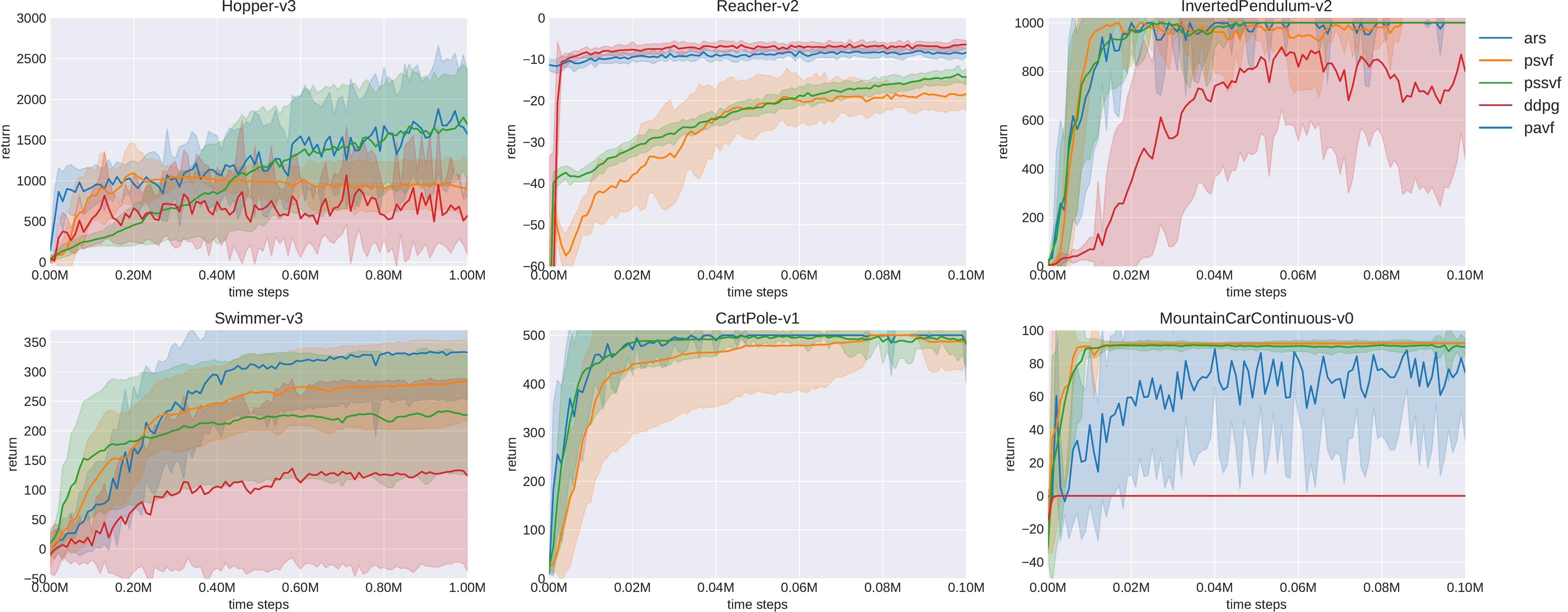}
\end{center}
\caption{Learning curves representing the average return as a function of the number of time steps in the environment (across 20 runs) with different environments using stochastic shallow policies.
We use the \textbf{best hyperparameters found while maximizing the average reward for each task.}}
\label{fig:learning_curve_lin_stoch}

\end{figure}

\begin{figure}[h]
\begin{center}
Stochastic deep policies
\includegraphics[width=0.9\linewidth]{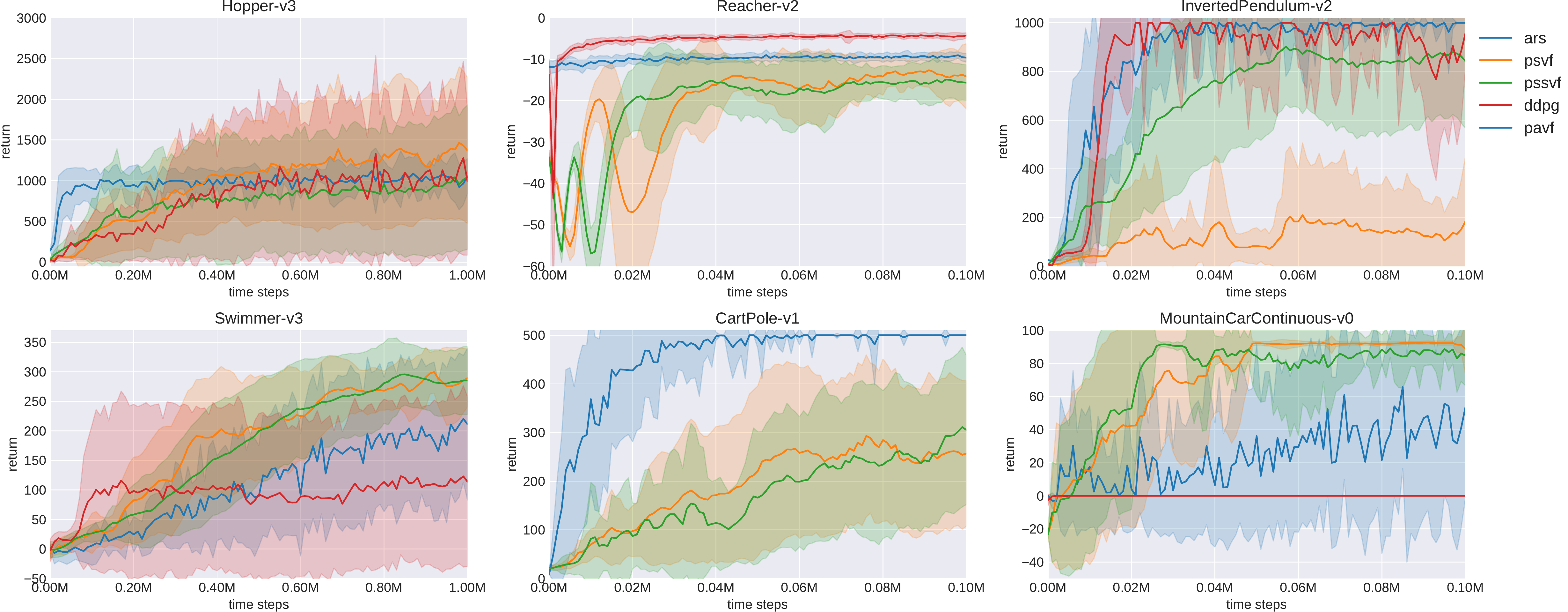}
\end{center}
\caption{Learning curves representing the average return as a function of the number of time steps in the environment (across 20 runs) with different environments using stochastic deep policies ([64,64]).
We use the \textbf{best hyperparameters found while maximizing the average reward for each task.}}
\label{fig:learning_curve_deep_stoch}

\end{figure}

\subsubsection{Sensitivity analysis}
\label{experiments_detail:sensitivity}
In the following, we report the sensitivity plots for all algorithms, for all deterministic policy architectures and environments. In particular, figure~\ref{fig:sensitivity_pvf},~\ref{fig:sensitivity_psvf},~\ref{fig:sensitivity_pavf},~\ref{fig:sensitivity_ddpg} and~\ref{fig:sensitivity_ars} show the performance of each algorithm given different hyperparameters tried during training. We observe that in general deep policies are more sensitive and, apart for DDPG, achieve often a better performance than smaller policies. The higher sensitivity displayed by ARS is in part caused by the higher number of hyperparameters we tried when tuning the algorithm. 
\begin{figure}[h]
\begin{center}
\includegraphics[width=0.9\linewidth]{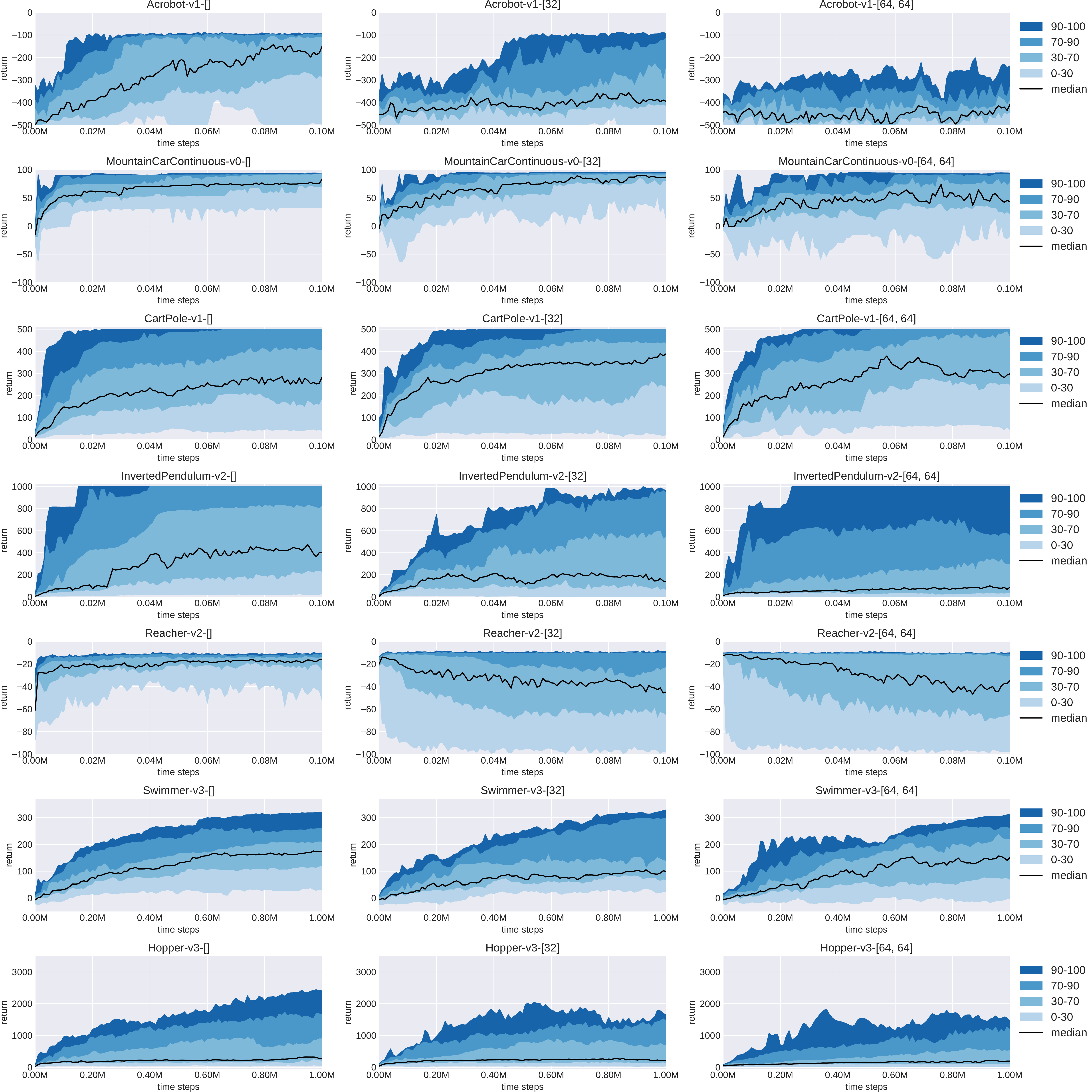}
\end{center}
\caption{\textbf{Sensitivity of PSSVFs using deterministic policies} to the choice of the hyperparameter. Performance is shown by percentile using all the learning curves obtained during hyperparameter tuning. The median performance is depicted as a dark line. For each subplot, the numbers in the square brackets represent the number of neurons per layer of the policy. [] represents a linear policy.}
\label{fig:sensitivity_pvf}

\end{figure}

\begin{figure}[h]
\begin{center}
\includegraphics[width=0.9\linewidth]{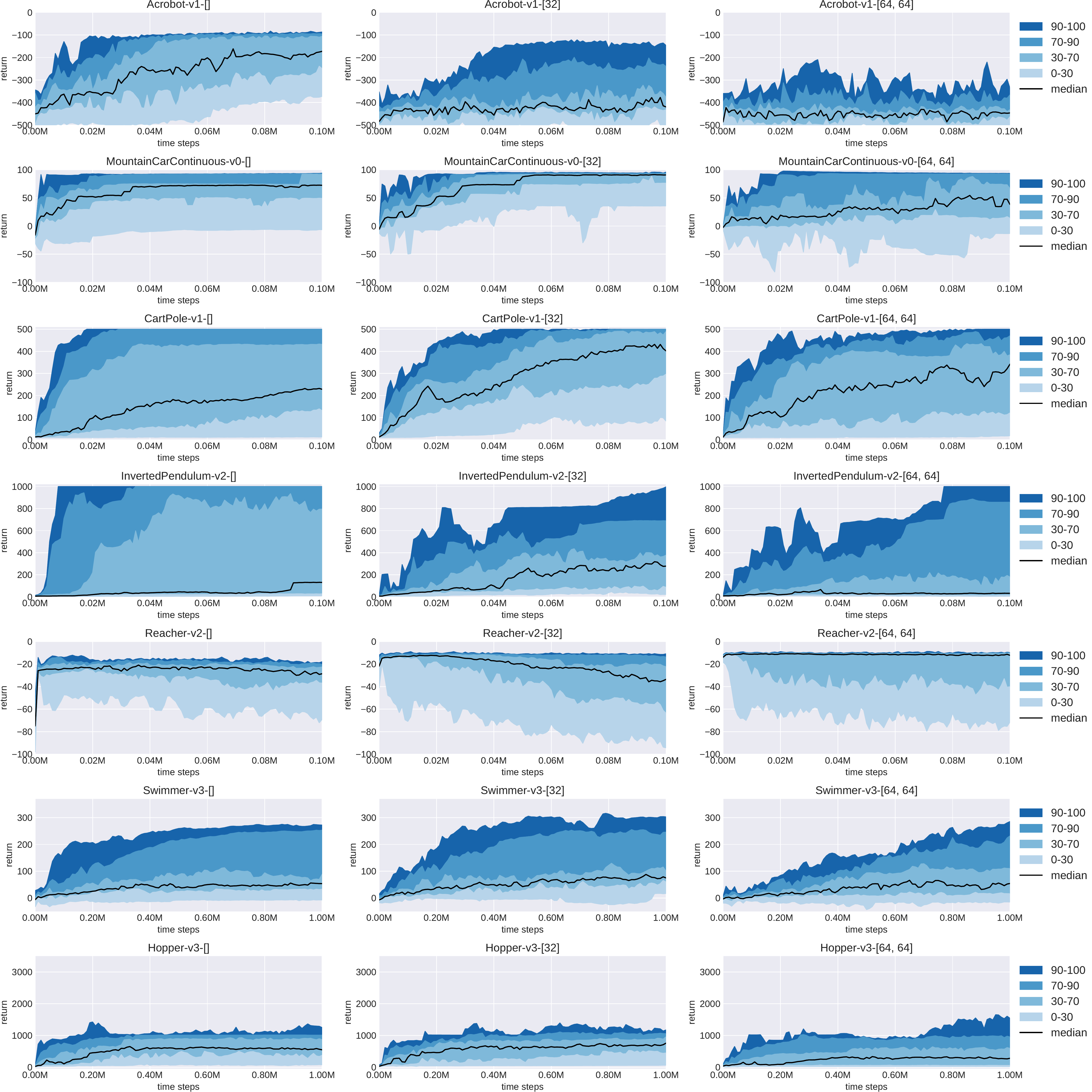}
\end{center}
\caption{\textbf{Sensitivity of PSVFs using deterministic policies} to the choice of the hyperparameter. Performance is shown by percentile using all the learning curves obtained during hyperparameter tuning. The median performance is depicted as a dark line. For each subplot, the numbers in the square brackets represent the number of neurons per layer of the policy. [] represents a linear policy.}
\label{fig:sensitivity_psvf}

\end{figure}

\begin{figure}[h]
\begin{center}
\includegraphics[width=0.9\linewidth]{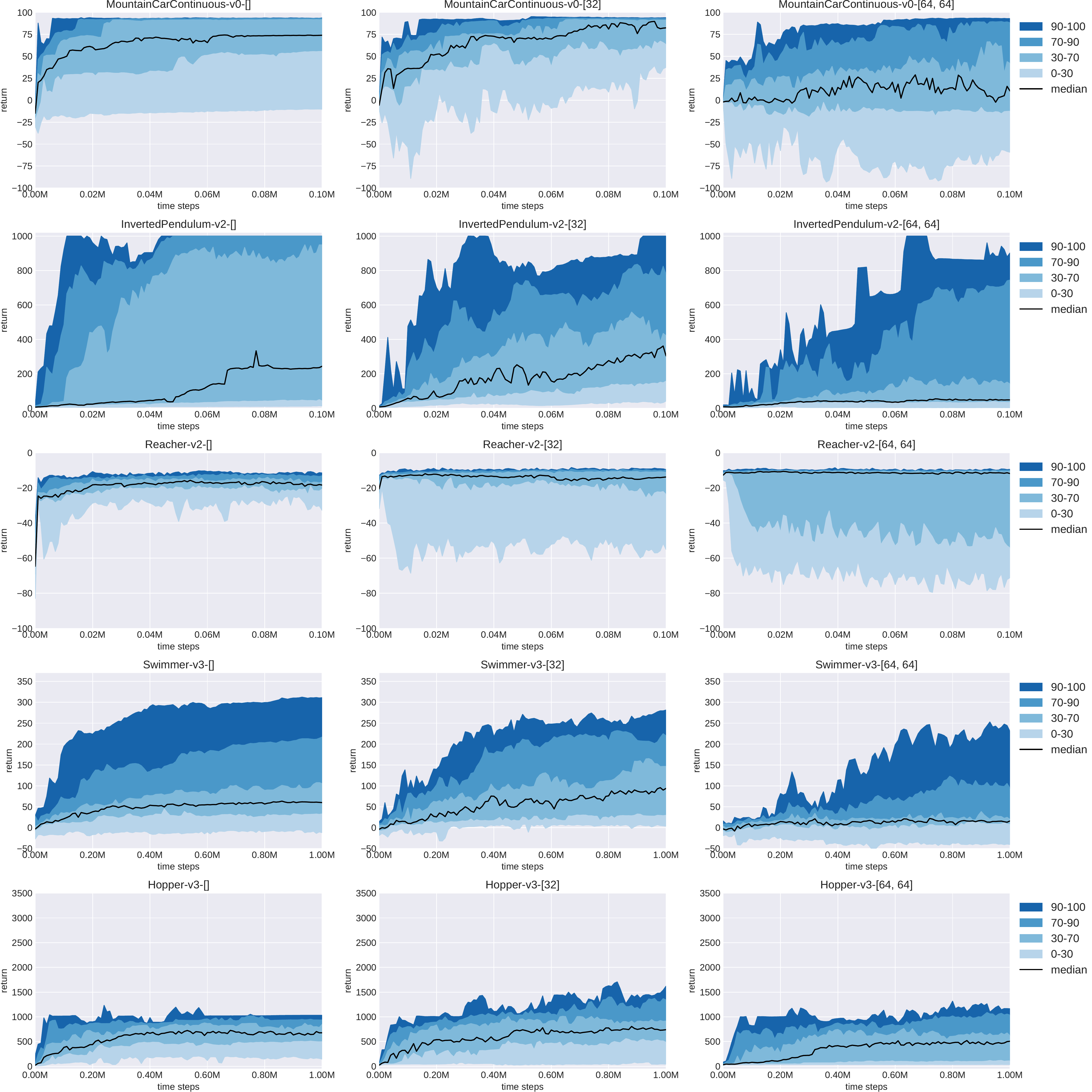}
\end{center}
\caption{\textbf{Sensitivity of PAVFs using deterministic policies} to the choice of the hyperparameter. Performance is shown by percentile using all the learning curves obtained during hyperparameter tuning. The median performance is depicted as a dark line. For each subplot, the numbers in the square brackets represent the number of neurons per layer of the policy.}
\label{fig:sensitivity_pavf}

\end{figure}

\begin{figure}[h]
\begin{center}
\includegraphics[width=0.9\linewidth]{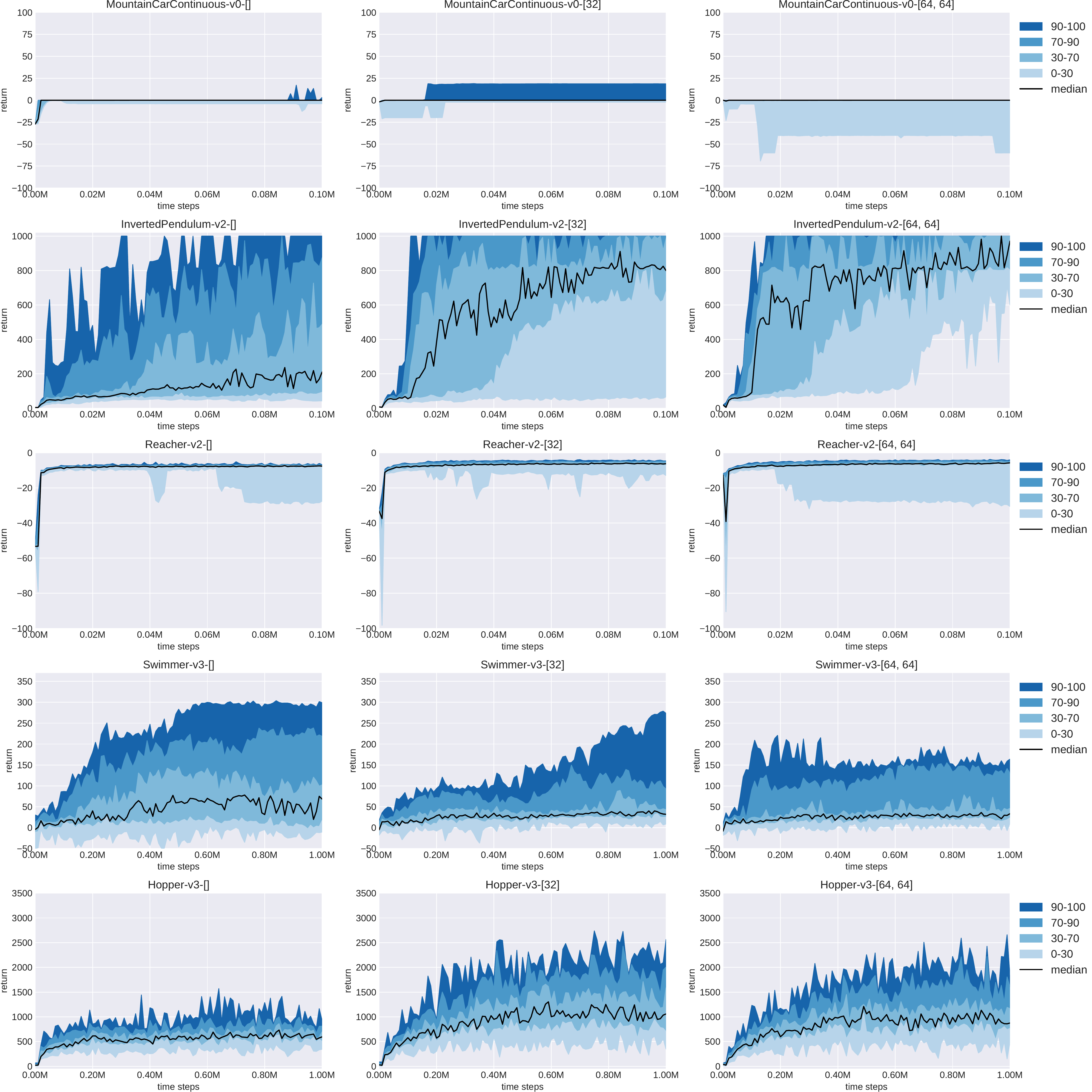}
\end{center}
\caption{\textbf{Sensitivity of DDPG} to the choice of the hyperparameter. Performance is shown by percentile using all the learning curves obtained during hyperparameter tuning. The median performance is depicted as a dark line. For each subplot, the numbers in the square brackets represent the number of neurons per layer of the policy.}
\label{fig:sensitivity_ddpg}

\end{figure}

\begin{figure}[h]
\begin{center}
\includegraphics[width=0.9\linewidth]{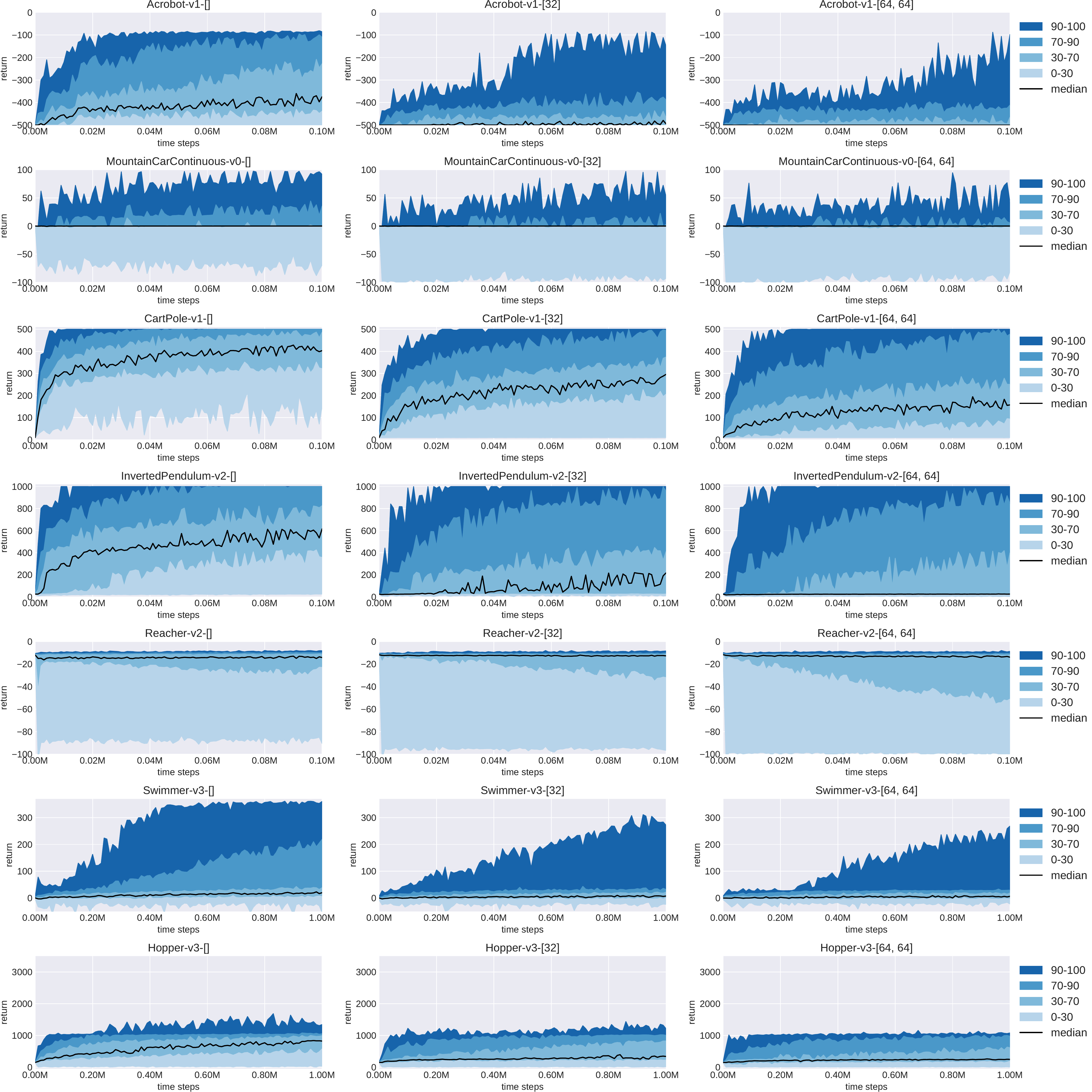}
\end{center}
\caption{\textbf{Sensitivity of ARS} to the choice of the hyperparameter. Performance is shown by percentile using all the learning curves obtained during hyperparameter tuning. The median performance is depicted as a dark line. For each subplot, the numbers in the square brackets represent the number of neurons per layer of the policy.}
\label{fig:sensitivity_ars}

\end{figure}

\subsubsection{Table of best hyperparameters}
We report for each algorithm, environment, and policy architecture the best hyperparameters found when optimizing for average return or final return in tables~\ref{tab:hyp_pvf},~\ref{tab:hyp_ars},~\ref{tab:hyp_psvf},~\ref{tab:hyp_stpavf},~\ref{tab:hyp_pavf} and~\ref{tab:hyp_ddpg}.

\begin{table}[!h]
  \caption{Table of best hyperparameters for PSSVFs using deterministic policies}
  \begin{tabular}{lrcccccc}
    \hline
    \textbf{Learning rate policy}  & Policy: &
      \multicolumn{2}{c}{[]} &
      \multicolumn{2}{c}{[32]} &
      \multicolumn{2}{c}{[64,64]} \\
    
    &Metric: &avg &last &avg &last &avg &last  \\
    \hline
    Acrobot-v1 & &1e-2 & 1e-3 &1e-4 & 1e-4 &1e-4 & 1e-4\\
    MountainCarContinuous-v0 & &1e-2 & 1e-3 &1e-4 & 1e-4 &1e-4 & 1e-4\\
    CartPole-v1 & &1e-3 & 1e-3 &1e-3 & 1e-3 &1e-4 & 1e-4\\
    Swimmer-v3 & &1e-3 & 1e-3 &1e-3 & 1e-3 &1e-2 & 1e-4\\
    InvertedPendulum-v2 & &1e-3 & 1e-3 &1e-3 & 1e-3 &1e-4 & 1e-4\\
    Reacher-v2 & &1e-4 & 1e-4 &1e-4 & 1e-4 &1e-4 & 1e-4\\
    Hopper-v3 & &1e-4 & 1e-4 &1e-4 & 1e-3 &1e-4 & 1e-4\\
    \hline
    \textbf{Learning rate critic} & & & & & & &\\
    \hline
    Acrobot-v1 & &1e-2 & 1e-3 &1e-2 & 1e-2 &1e-2 & 1e-2\\
    MountainCarContinuous-v0 & &1e-3 & 1e-2 &1e-3 & 1e-2 &1e-2 & 1e-2\\
    CartPole-v1 & &1e-2 & 1e-2 &1e-3 & 1e-3 &1e-2 & 1e-2\\
    Swimmer-v3 & &1e-3 & 1e-3 &1e-2 & 1e-2 &1e-3 & 1e-2\\
    InvertedPendulum-v2 & &1e-2 & 1e-2 & 1e-3 & 1e-2 &1e-3 & 1e-3\\
    Reacher-v2 & &1e-3 & 1e-3 &1e-3 & 1e-3 &1e-4 & 1e-4\\
    Hopper-v3 & &1e-3 & 1e-3 &1e-2 & 1e-2 &1e-2 & 1e-2\\
    \hline
    \textbf{Noise for exploration}  & & & & & & &\\
    \hline
    Acrobot-v1 & & 1.0 & 1.0 & 1e-1 & 1e-1 & 1e-1 & 1e-1\\
    MountainCarContinuous-v0 & & 1.0 & 1.0 & 1e-1 & 1e-1 & 1e-1 & 1e-1\\
    CartPole-v1 & & 1.0 & 1.0 & 1.0 & 1.0 & 1e-1 & 1e-1\\
    Swimmer-v3 & & 1.0 & 1.0 & 1.0 & 1.0 & 1.0 & 1e-1\\
    InvertedPendulum-v2 & & 1.0 & 1.0 & 1.0 & 1.0 & 1e-1 & 1e-1\\
    Reacher-v2 & & 1e-1 & 1e-1 & 1e-1 & 1e-1 & 1e-1 & 1e-1\\
    Hopper-v3 & & 1.0 & 1.0 & 1e-1 & 1.0 & 1e-1 & 1e-1\\
    \hline
  \end{tabular}
  \label{tab:hyp_pvf}
\end{table}  

\begin{table}[!h]
  \caption{Table of best hyperparameters for ARS}
  \begin{tabular}{lrcccccc}
    \hline
    \textbf{Learning rate policy}  & Policy: &
      \multicolumn{2}{c}{[]} &
      \multicolumn{2}{c}{[32]} &
      \multicolumn{2}{c}{[64,64]} \\
    
    &Metric: &avg &last &avg &last &avg &last  \\
    \hline
    Acrobot-v1 & &1e-2 & 1e-3 &1e-2 & 1e-2 &1e-2 & 1e-2\\
    MountainCarContinuous-v0 & &1e-2 & 1e-2 &1e-2 & 1e-2 &1e-2 & 1e-2\\
    CartPole-v1 & &1e-2 & 1e-2 &1e-2 & 1e-2 &1e-2 & 1e-2\\
    Swimmer-v3 & &1e-2 & 1e-2 &1e-2 & 1e-2 &1e-2 & 1e-2\\
    InvertedPendulum-v2 & &1e-2 & 1e-2 &1e-2 & 1e-2 &1e-2 & 1e-2\\
    Reacher-v2 & &1e-2 & 1e-2 &1e-3 & 1e-2 &1e-3 & 1e-3\\
    Hopper-v3 & &1e-2 & 1e-2 &1e-2 & 1e-2 &1e-2 & 1e-2\\
    \hline
    \textbf{Number of directions} & & & & & & &\\
    \textbf{and elite directions} & & & & & & &\\
    \hline
    Acrobot-v1 & &(4,4) &(4,4) &(1,1) &(1,1) &(1,1) &(1,1)\\
    MountainCarContinuous-v0 & &(1,1) &(1,1) &(1,1) &(16,4) &(1,1) &(1,1)\\
    CartPole-v1 & &(4,4) &(4,4) &(1,1) &(1,1) &(4,1) &(4,1)\\
    Swimmer-v3 & &(1,1) &(1,1) &(1,1) &(4,1) &(1,1) &(1,1)\\
    InvertedPendulum-v2 & &(4,4) &(4,4) &(1,1) &(4,4) &(4,1) &(16,1)\\
    Reacher-v2 & &(16,16) &(16,16) &(1,1) &(16,4) &(1,1) &(1,1)\\
    Hopper-v3 & &(4,1) &(4,1) &(1,1) &(1,1) &(1,1) &(1,1)\\
    \hline
    \textbf{Noise for exploration}  & & & & & & &\\
    \hline
    Acrobot-v1 & & 1e-2 & 1e-3 & 1e-1 & 1e-1 & 1e-1 & 1e-1\\
    MountainCarContinuous-v0 & & 1e-1 & 1e-1 & 1e-1 & 1e-1 & 1e-1 & 1e-1\\
    CartPole-v1 & & 1e-2 & 1e-2 & 1e-1 & 1e-1 & 1e-2 & 1e-2\\
    Swimmer-v3 & & 1e-1 & 1e-1 & 1e-2 & 1e-1 & 1e-1 & 1e-1\\
    InvertedPendulum-v2 & & 1e-2 & 1e-2 & 1e-2 & 1e-2 & 1e-2 & 1e-2\\
    Reacher-v2 & & 1e-2 & 1e-2 & 1e-2 & 1e-2 & 1e-2 & 1e-2\\
    Hopper-v3 & & 1e-1 & 1e-1 & 1e-1 & 1e-1 & 1e-1 & 1e-1\\
    \hline
  \end{tabular}
    \label{tab:hyp_ars}

\end{table}

\begin{table}[!h]
  \caption{Table of best hyperparameters for PSVFs using deterministic policies}
  \begin{tabular}{lrcccccc}
    \hline
    \textbf{Learning rate policy}  & Policy: &
      \multicolumn{2}{c}{[]} &
      \multicolumn{2}{c}{[32]} &
      \multicolumn{2}{c}{[64,64]} \\
    
    &Metric: &avg &last &avg &last &avg &last  \\
    \hline
    Acrobot-v1 & &1e-2 & 1e-2 &1e-4 & 1e-4 &1e-4 & 1e-2\\
    MountainCarContinuous-v0 & &1e-2 & 1e-3 &1e-2 & 1e-4 &1e-3 & 1e-4\\
    CartPole-v1 & &1e-2 & 1e-2 &1e-2 & 1e-4 &1e-3 & 1e-4\\
    Swimmer-v3 & &1e-3 & 1e-3 &1e-3 & 1e-3 &1e-3 & 1e-3\\
    InvertedPendulum-v2 & &1e-2 & 1e-3 &1e-4 & 1e-4 &1e-4 & 1e-4\\
    Reacher-v2 & &1e-3 & 1e-2 &1e-4 & 1e-4 &1e-4 & 1e-4\\
    Hopper-v3 & &1e-3 & 1e-3 &1e-4 & 1e-4 &1e-4 & 1e-3\\
    \hline
    \textbf{Learning rate critic}  & & & & & & &\\
    \hline
    Acrobot-v1 & &1e-3 & 1e-4 &1e-2 & 1e-2 &1e-3 & 1e-2\\
    MountainCarContinuous-v0 & &1e-4 & 1e-3 &1e-2 & 1e-4 &1e-3 & 1e-3\\
    CartPole-v1 & &1e-2 & 1e-2 &1e-2 & 1e-3 &1e-2 & 1e-4\\
    Swimmer-v3 & &1e-4 & 1e-4 &1e-4 & 1e-4 &1e-4 & 1e-4\\
    InvertedPendulum-v2 & &1e-3 & 1e-2 &1e-3 & 1e-4 &1e-4 & 1e-3\\
    Reacher-v2 & &1e-2 & 1e-2 &1e-3 & 1e-3 &1e-4 & 1e-4\\
    Hopper-v3 & &1e-2 & 1e-2 &1e-4 & 1e-4 &1e-2 & 1e-4\\
    \hline
    \textbf{Noise for exploration}  & & & & & & &\\
    \hline
    Acrobot-v1 & & 1.0 & 1.0 & 1e-1 & 1e-1 & 1e-1 & 1e-1\\
    MountainCarContinuous-v0 & & 1.0 & 1e-1 & 1e-1 & 1.0 & 1e-1 & 1e-1\\
    CartPole-v1 & & 1.0 & 1.0 & 1.0 & 1e-1 & 1e-1 & 1e-1\\
    Swimmer-v3 & & 1.0 & 1.0 & 1.0 & 1.0 & 1.0 & 1.0\\
    InvertedPendulum-v2 & & 1.0 & 1.0 & 1e-1 & 1e-1 & 1e-1 & 1e-1\\
    Reacher-v2 & & 1.0 & 1.0 & 1.0 & 1.0 & 1e-1 & 1e-1\\
    Hopper-v3 & & 1.0 & 1.0 & 1e-1 & 1e-1 & 1e-1 & 1.0\\
    \hline
  \end{tabular}
    \label{tab:hyp_psvf}

\end{table}  

\begin{table}[!h]
  \caption{Table of best hyperparameters for PSSVFs and PSVFs using stochastic policies}
  \begin{tabular}{lrcccc}
    \hline
    \toprule & Algo: & \multicolumn{2}{c}{PSSVF} & \multicolumn{2}{c}{PSVF}\\
    \midrule
    \textbf{Learning rate policy}  & Policy: &
     [] & [64,64] & [] & [64,64] \\
    &Metric: &avg  &avg  &avg  &avg  \\
    \hline
    Acrobot-v1 & & 1e-2 & 1e-2 & 1e-2 & 1e-3 \\
    MountainCarContinuous-v0 & & 1e-2 & 1e-3 & 1e-2 & 1e-3 \\
    CartPole-v1 & & 1e-3 & 1e-4 & 1e-2 & 1e-3 \\
    Swimmer-v3 & & 1e-2 & 1e-4 & 1e-3 & 1e-4 \\
    InvertedPendulum-v2 & & 1e-3 & 1e-4 & 1e-2 & 1e-3 \\
    Reacher-v2 & & 1e-4 & 1e-3 & 1e-2 & 1e-2 \\
    Hopper-v3 & & 1e-4 & 1e-4 & 1e-3 & 1e-4 \\
    \hline
    \textbf{Learning rate critic} & & & & &\\
    \hline
    Acrobot-v1 & & 1e-2 & 1e-4 & 1e-4 & 1e-2 \\
    MountainCarContinuous-v0 & & 1e-2 & 1e-2 & 1e-3 & 1e-3 \\
    CartPole-v1 & & 1e-2 & 1e-3 & 1e-2 & 1e-2 \\
    Swimmer-v3 & & 1e-2 & 1e-3 & 1e-3 & 1e-4 \\
    InvertedPendulum-v2 & & 1e-3 & 1e-3 & 1e-3 & 1e-2 \\
    Reacher-v2 & & 1e-3 & 1e-3 & 1e-3 & 1e-3 \\
    Hopper-v3 & & 1e-3 & 1e-2 & 1e-2 & 1e-4 \\
    \hline
    \textbf{Noise for exploration} & & & & &\\
    \hline
    Acrobot-v1 & & 1.0 & 1.0 & 1.0 & 1.0 \\
    MountainCarContinuous-v0  & & 1.0 & 1e-1 & 1.0 & 1e-1 \\
    CartPole-v1 & & 1.0 & 1.0 & 1.0 & 1e-1 \\
    Swimmer-v3 & & 1.0 & 1e-1 & 1.0 & 1e-1 \\
    InvertedPendulum-v2 & & 1.0 & 1.0 & 1.0 & 1e-1 \\
    Reacher-v2 & & 1e-1 & 0.0 & 1.0 & 0.0 \\
    Hopper-v3 & & 1.0 & 1e-1 & 1.0 & 1e-1 \\
    \hline
  \end{tabular}
    \label{tab:hyp_stpavf}

\end{table}

\begin{table}[!h]
  \caption{Table of best hyperparameters for PAVFs using deterministic policies}
  \begin{tabular}{lrcccccc}
    \hline
    \textbf{Learning rate policy}  & Policy: &
      \multicolumn{2}{c}{[]} &
      \multicolumn{2}{c}{[32]} &
      \multicolumn{2}{c}{[64,64]} \\
    
    &Metric: &avg &last &avg &last &avg &last  \\
    \hline
    MountainCarContinuous-v0 & &1e-2 & 1e-3 &1e-3 & 1e-4 &1e-4 & 1e-4\\
    Swimmer-v3 & &1e-3 & 1e-3 &1e-3 & 1e-3 &1e-3 & 1e-3\\
    InvertedPendulum-v2 & &1e-2 & 1e-3 &1e-3 & 1e-4 &1e-4 & 1e-4\\
    Reacher-v2 & &1e-3 & 1e-3 &1e-4 & 1e-4 &1e-4 & 1e-4\\
    Hopper-v3 & &1e-3 & 1e-4 &1e-4 & 1e-4 &1e-4 & 1e-3\\
    \hline
    \textbf{Learning rate critic} & & & & & & &\\
    \hline
    MountainCarContinuous-v0 & &1e-4 & 1e-4 &1e-4 & 1e-3 &1e-4 & 1e-3\\
    Swimmer-v3 & &1e-4 & 1e-4 &1e-4 & 1e-4 &1e-4 & 1e-4\\
    InvertedPendulum-v2 & &1e-3 & 1e-2 &1e-2 & 1e-4 &1e-2 & 1e-3\\
    Reacher-v2 & &1e-3 & 1e-3 &1e-3 & 1e-2 &1e-3 & 1e-3\\
    Hopper-v3 & &1e-4 & 1e-3 &1e-3 & 1e-2 &1e-4 & 1e-3\\
    \hline
    \textbf{Noise for exploration} & & & & & & &\\
    \hline
    MountainCarContinuous-v0 & & 1.0 & 1e-1 & 1e-1 & 1e-1 & 1e-1 & 1e-1\\
    Swimmer-v3 & & 1.0 & 1.0 & 1.0 & 1.0 & 1.0 & 1.0\\
    InvertedPendulum-v2 & & 1.0 & 1.0 & 1e-1 & 1e-1 & 1e-1 & 1e-1\\
    Reacher-v2 & & 1e-1 & 1e-1 & 1e-1 & 1.0 & 1.0 & 1.0\\
    Hopper-v3 & & 1.0 & 1.0 & 1e-1 & 1e-1 & 1e-1 & 1.0\\
    \hline
  \end{tabular}
    \label{tab:hyp_pavf}

\end{table}      

\begin{table}[!h]
  \caption{Table of best hyperparameters for DDPG}
  \begin{tabular}{lrcccccc}
    \hline
    \textbf{Learning rate policy}  & Policy: &
      \multicolumn{2}{c}{[]} &
      \multicolumn{2}{c}{[32]} &
      \multicolumn{2}{c}{[64,64]} \\
    
    &Metric: &avg &last &avg &last &avg &last  \\
    \hline
    MountainCarContinuous-v0 & &1e-2 & 1e-2 &1e-2 & 1e-4 &1e-3 & 1e-3\\
    Swimmer-v3 & &1e-3 & 1e-3 &1e-2 & 1e-2 &1e-2 & 1e-2\\
    InvertedPendulum-v2 & &1e-4 & 1e-4 &1e-3 & 1e-3 &1e-3 & 1e-4\\
    Reacher-v2 & &1e-4 & 1e-3 &1e-2 & 1e-2 &1e-3 & 1e-3\\
    Hopper-v3 & &1e-2 & 1e-2 &1e-2 & 1e-4 &1e-2 & 1e-2\\
    \hline
    \textbf{Learning rate critic}  & & & & & & &\\
    \hline
    MountainCarContinuous-v0 & &1e-4 & 1e-4 &1e-4 & 1e-3 &1e-3 & 1e-3\\
    Swimmer-v3 & &1e-3 & 1e-3 &1e-3 & 1e-3 &1e-2 & 1e-3\\
    InvertedPendulum-v2 & &1e-3 & 1e-3 &1e-3 & 1e-4 &1e-3 & 1e-3\\
    Reacher-v2 & &1e-3 & 1e-3 &1e-3 & 1e-3 &1e-3 & 1e-3\\
    Hopper-v3 & &1e-3 & 1e-3 &1e-4 & 1e-4 &1e-4 & 1e-4\\
    \hline
    \textbf{Noise for exploration}  & & & & & & &\\
    \hline
    MountainCarContinuous-v0 & & 1e-2 & 1e-2 & 1e-2 & 1e-1 & 1e-1 & 1e-1\\
    Swimmer-v3 & & 1e-1 & 1e-1 & 1e-2 & 1e-2 & 1e-2 & 1e-1\\
    InvertedPendulum-v2 & & 1e-1 & 1e-1 & 1e-2 & 1e-2 & 1e-2 & 1e-2\\
    Reacher-v2 & & 1e-1 & 1e-2 & 1e-1 & 1e-1 & 1e-1 & 1e-1\\
    Hopper-v3 & & 1e-1 & 1e-1 & 1e-1 & 1e-2 & 1e-1 & 1e-2\\
    \hline
  \end{tabular}
    \label{tab:hyp_ddpg}

\end{table}

\end{document}